\documentclass[twoside,11pt]{article}

\usepackage{cmap}
\usepackage[T1]{fontenc}
\usepackage{lmodern}
\usepackage[utf8]{inputenc}
\usepackage{verbatim}
\usepackage{booktabs}       
\usepackage{nicefrac}       
\usepackage{microtype}      
\usepackage{graphicx} 
\usepackage[colorlinks]{hyperref}
\usepackage{url}
\usepackage{subcaption}
\usepackage{enumerate}
\usepackage{geometry} % to change the page dimensions
\usepackage[toc,page]{appendix}
\usepackage{mathrsfs}
\usepackage{natbib}

\bibpunct{(}{)}{;}{a}{,}{,}

\usepackage[linesnumbered, ruled, noend]{algorithm2e}
\usepackage{tikz}
\usepackage{float}
\usepackage{bm}
\usepackage{afterpage}
\usepackage{multirow}

\usepackage[font=small,labelfont=bf,labelsep=period]{caption}

\usepackage{amsmath, amsbsy, amsgen, amscd, amsthm, amsfonts, amssymb} 

\providecommand{\mathbold}[1]{\bm{#1}}

\newtheoremstyle{myThm}   % Name
     {\topsep}                          % Space above 
     {\topsep}                          % Space below
     {\itshape}                         % Body font
     {}                                      % Indent amount
     {\sffamily\bfseries}           % Theorem head font
     {.}                                     % Punctuation after theorem head
     {.5em}                              % Space after theorem head
     {}                                      % Theorem head specifications. Empty means normal

\newtheoremstyle{myRem}   % Name
     {\topsep}                          % Space above 
     {\topsep}                          % Space below
     {}                         % Body font
     {}                                      % Indent amount
     {\sffamily\bfseries}           % Theorem head font
     {.}                               % Punctuation after theorem head
     {.5em}                              % Space after theorem head
     {}                                      % Theorem head specifications. Empty means normal
     
\newtheoremstyle{myDef}   % Name
     {\topsep}                          % Space above 
     {\topsep}                          % Space below
     {\itshape}                         % Body font
     {}                                      % Indent amount
     {\sffamily\bfseries}           % Theorem head font
     {.}                                 % Punctuation after theorem head
     {.5em}                              % Space after theorem head
     {}                                      % Theorem head specifications. Empty means normal     

\theoremstyle{myThm}
\newtheorem{theorem}{Theorem}
\newtheorem{lemma}{Lemma}
\newtheorem{proposition}{Proposition}
\newtheorem{corollary}{Corollary}

\newtheorem{assumption}{Assumption}

\theoremstyle{myRem}
\newtheorem{remark}{Remark}

\theoremstyle{myDef}
\newtheorem{definition}{Definition}

\usepackage[toc,page]{appendix}
\usepackage{makeidx}
\makeindex
\usepackage{float}
\usepackage{mathrsfs}
\usepackage{framed}
\usepackage{float}
\usepackage[font=small,labelfont=bf,labelsep=period]{caption}

\usepackage{graphicx} 
\usepackage{epstopdf}

\DeclareMathOperator*{\argmax}{\text{argmax}}

\newcommand{\binomial}{\text{Binomial}}
\newcommand{\cover}{\mathcal{C}}

\renewcommand{\epsilon}{\varepsilon}

\newcommand{\kldiv}{\text{KL}}

\newcommand{\naturals}{\mathbb{N}}

\newcommand{\reals}{\mathbb{R}}
\newcommand{\risk}{R}

\newcommand{\sfield}{\mathcal{G}}
\newcommand{\simplex}{\Delta}

\newcommand{\support}{\text{supp}}

\renewcommand{\xspace}{\mathcal{X}}
\newcommand{\yspace}{\mathcal{Y}}

\newcommand{\expect}{\mathbb{E}}
\newcommand{\prob}{\mathbb{P}}

\newcommand{\ind}{\mathbold{1}}

\newcommand{\truepositive}{\text{TP}}
\newcommand{\truenegative}{\text{TN}}
\newcommand{\falsepositive}{\text{FP}}
\newcommand{\falsenegative}{\text{FN}}
\newcommand{\confusionmatrix}{\text{CM}}

\newcommand{\tpe}{(\text{tpe})}
\newcommand{\tne}{(\text{tne})}
\newcommand{\fpe}{(\text{fpe})}
\newcommand{\fne}{(\text{fne})}

\newcommand{\eventA}{\mathcal{A}}
\newcommand{\eventB}{\mathcal{B}}
\newcommand{\eventC}{\mathcal{C}}
\newcommand{\eventD}{\mathcal{D}}

\newcommand{\emptne}{(\hat{\text{tne}})}
\newcommand{\empfne}{(\hat{\text{fne}})}

\newcommand{\precision}{\text{prec}}
\newcommand{\recall}{\text{rec}}
\newcommand{\fone}{\text{F1}}
\newcommand{\mcc}{\text{MCC}}
\newcommand{\uniform}{\text{unif}}
\newcommand{\excessrisk}{\mathcal{E}}
\newcommand{\uniformerror}{U}
\newcommand{\boundaryproximal}{\partial'}

\newcommand{\closedball}{\overline{B}}
\newcommand{\openball}{B}

\renewcommand{\P}{\mathcal{P}} % Class of probability distributions
\newcommand{\R}{\reals} % Real numbers
\newcommand{\X}{\mathcal{X}} % Sample space

\renewcommand{\hat}{\widehat} % Make hats wide by default
\renewcommand{\tilde}{\widetilde} % Make tildes wide by default

\newcommand{\vect}[1]{\bm{#1}}

\title{Multiclass Classification via Class-Weighted Nearest Neighbors}

\author{Justin Khim\footnote{Machine Learning Department, Carnegie Mellon University, Pittsburgh, PA 15213}
\and Ziyu Xu$^*$ \and Shashank Singh$^*$\footnote{Google Pittsburgh}}

\begin{document}

\maketitle

%---------------------------------------------%
%---------------------------------------------%

\begin{abstract}
We study statistical properties of the \(k\)-nearest neighbors algorithm for multiclass classification, with a focus on settings where the number of classes may be large and/or classes may be highly imbalanced. In particular, we consider a variant of the $k$-nearest neighbor classifier with  non-uniform class-weightings, for which we derive upper and minimax lower bounds on accuracy, class-weighted risk, and uniform error.
Additionally, we show that uniform error bounds lead to bounds on the difference between empirical confusion matrix quantities and their population counterparts across a set of weights.
As a result, we may adjust the class weights to optimize classification metrics such as F1 score or Matthew's Correlation Coefficient that are commonly used in practice, particularly in settings with imbalanced classes.
We additionally provide a simple example to instantiate our bounds and numerical experiments.
\end{abstract}
%---------------------------------------------%
%---------------------------------------------%

%---------------------------------------------%
%---------------------------------------------%
\section{Introduction}
\label{sec:Intro}

Classification is a fundamental problem in statistics and machine learning that arises in many scientific and engineering problems.
Scientific applications include identifying plant and animal species from body measurements, determining cancer types based on gene expression, and satellite image processing \citep{fisher1936use, fisher1938statistical, khan2001classification, lee2004cloud};
in modern engineering contexts, credit card fraud detection, handwritten digit recognition, word sense disambiguation, and object detection in images are all examples of classification tasks.

These applications have brought two new challenges: multiclass classification with a potentially large number of classes and imbalanced data.
For example, in online retailing, websites have hundreds of thousands or millions of products, and they may like to categorize these products within a pre-existing taxonomy based on product descriptions \citep{lin2018overview}.
While the number of classes alone makes the problem difficult, an added difficulty with text data is that it is usually highly imbalanced, meaning that a few classes may constitute a large fraction of the data while many classes have only a few examples.
In fact, \cite{feldman2019does} notes that if the data follows the classical Zipf distribution for text data \citep{zipf1936psycho}, i.e., the class probabilities satisfy a power-law distribution, then up to 35\% of seen examples may appear only once in the training data.
Additionally, natural image data also seems to have the problems of many classes and imbalanced data \citep{salakhutdinov2011learning, zhu2014capturing}.

Focusing on the problem of imbalanced data, researchers have found that a few heuristics help ``do better,'' and the most principled and studied of these is weighting.
There are a number of forms of weighting; we consider the most basic in which we incur a loss of weight \(q_{c}\) for misclassifying an example of class \(c\) and refer to this method as class-weighting.
Class-weighting provides a principled approach to applications such as credit card fraud detection or online retailing, in which it can be fairly easy to assign a cost to an example of a given class, e.g., perhaps it costs a credit card company hundreds of dollars on average to pay a fraudulent charge.
As a result, weighting has been studied in a number of settings, including as a wrapper around a black-box classifier \citep{domingos1999metacost}, in support vector machines (SVMs) \citep{lin2002support, scott2012calibrated}, and in neural networks \citep{zhou2006training}.
Additionally, weighting has been observed to be equivalent to adjusting the threshold that determines the decision boundary, and this has also been used to estimate class probabilities from hard classifiers \citep{wang2008probability, wu2010robust, wang2019multiclass}.

Of course, while ``doing better'' on imbalanced data usually corresponds to somehow improving performance on small classes, this is a vague notion.
In practice, success is often evaluated using a different metric than prediction accuracy, and examples of popular metrics include precision, recall, \(F_{\beta}\)-measure, and Matthew's Correlation Coefficient \citep{van1974foundation, van1979information}.
There are a couple of lines of work in this area, and each usually focuses on a particular class of metrics.
The first line considers plug-in binary classification, and optimizing many of these metrics amounts to finding the proper threshold for the decision rule \citep{koyejo2014consistent, lewis1995evaluating, menon2013statistical, narasimhan2014statistical}.
The other line of work concerns optimization for pre-existing algorithms such as SVMs and neural networks \citep{dembczynski2013optimizing, fathony2019genericMetrics, joachims2005support}.
Statistically, the best result from this line of work is consistency in that the algorithm under consideration asymptotically optimizes the metric of choice under additional assumptions.

A popular algorithm that has received little theoretical attention for multiclass classification and imbalanced classification is \(k\)-nearest neighbors (kNN).
Theoretically, kNN is well-understood in binary classification with respect to accuracy and excess risk \citep{biau2015lectures, chaudhuri2014rates}, and because it provides class probability estimates, kNN can readily be combined with class-weighting.
On the other hand, most research on using kNN for imbalanced classification problems focuses on algorithmic modifications.
Examples include prototype selection and weighting, in which a set of possibly-weighted representative points are chosen based on the training set to use with the kNN classifier \citep{liu2011class, lopez2014addressing, vluymans2016eprennid},
and gravitational methods, in which the distance function is modified to resemble the gravitational force \citep{cano2013weighted, zhu2015gravitational}.
Further variations are surveyed in \cite{fernandez2018learning}.

In this paper, we consider class-weighted nearest neighbors for multiclass classification.
First, we extend theoretical results on the accuracy, risk, weighted risk, and uniform error to the multiclass setting.
These include upper bounds for kNN as well as matching minimax lower bounds.
Second, using our results for uniform error, we obtain bounds on the difference between the empirical confusion matrix and the population confusion matrix uniformly over a set of weights \(q\).
These lead to quantitative upper bounds on the difference between empirical and population values for a given metric and allow us to optimize commonly-used performance metrics such as \(F_{\beta}\) score adaptively with respect to the choice of weight \(q\), since \(q\) is the multiclass analog of the decision threshold.
Our upper bounds depend on the data-generating distribution, and we show that this is ultimately unavoidable via corresponding lower bounds.

The remainder of this paper is organized as follows.
In Section~\ref{sec:Setup}, we establish our notation.
In Section~\ref{sec:MainResults}, we state our main results for the nearest neighbor classifier.
In Section~\ref{sec:ApplicationGeneralClassification}, we present the convergence of the empirical confusion matrix to the true confusion matrix, which implies the convergence of many general classification metrics.
In Section~\ref{sec:Example}, we illustrate our bounds numerically on a simple example distribution.
In Section~\ref{sec:NumericalResults}, we consider simple algorithms for optimizing the F1 score based on our results.
Finally, we conclude with a discussion in Section~\ref{sec:Discussion}.
Due to space considerations, all proofs are deferred to the appendices.
Since the uniform convergence over weightings is, to the best of our knowledge, novel even for binary classification, we also consider binary classification in the appendix, and we note that the results are stronger than for multiclass classification.

%---------------------------------------------%
%---------------------------------------------%
\subsection{Related Work}

The kNN classifier, first published by \citet{fix1951discriminatory}, is one of the oldest and most well-studied nonparametric classifiers. Early theoretical results include those of \citet{cover1967nearest}, who showed that the misclassification risk of the kNN classifier with $k = 1$ is at most twice that of the Bayes-optimal classifier, and \citet{stone1977consistent}, who showed that the kNN classifier is Bayes-consistent if $k \to \infty$ at an appropriate rate. For overviews of the theory of kNN methods in binary classification (as well as in regression and density estimation), see \citet{devroye1996probabilistic, gyorfi2002distribution, biau2015lectures}. Presently, we discuss some of the most recent related work in kNN.

\cite{samworth2012optimal} studies schemes for reweighting the kNN classifier in order to improve rates for highly smooth regression function; to do this, he uses the order of neighbors' distances, but not their labels as in this paper.
\cite{chaudhuri2014rates} proves general distribution-dependent upper and lower bounds for the excess risk with respect to accuracy of kNN.
Additionally, they introduce a general smoothness condition and specialize their results to a few specific settings.
\cite{doring2018rate} considers the excess risk with respect to accuracy and the \(L_2\) risk of kNN under a modified Lipschitz condition, which is a special case of the smoothness condition of \cite{chaudhuri2014rates}, while avoiding the assumption that the density is lower bounded away from \(0\). 
\cite{gadat2016classification} considers classification accuracy for general distributions without a strong density assumption and provides rates of convergence of excess risk where \(k\) is allowed to vary across the covariate space.
\cite{cannings2019local} also considers a semi-supervised classification setting where \(k\) is allowed to vary across the covariate space.
Here, the unlabeled samples are used to estimate the density, and the authors obtain an asymptotic form of the excess risk.

Unlike accuracy bounds, our bounds on uniform risk are more closely related to risk bounds for kNN regression, of which the results of \cite{biau2010rates} are representative.
\cite{biau2010rates} gives convergence rates for kNN regression in $L_2$ risk, weighted by the covariate distribution, in terms of covering numbers of the sample space and the variance of the noise. While closely related to our bounds on uniform ($L_\infty$) risk, their results differ in a few key ways. First, minimax rates under $L_\infty$ are necessarily worse than under $L_2$ risk by a logarithmic factor (as implied by our lower bounds). Second, instead of assuming categorical labels, they assume noise with finite variance; as a result, naively applying their bound to our multiclass setting would give a logarithmic dependency on the number $C$ of classes, whereas we derive rates independent of $C$.
Perhaps most importantly, that fact that their risk is weighted by the covariate distribution allows them to avoid our assumption that the covariate density is lower bounded away from \(0\). However, the lower boundedness assumption is unavoidable under $L_\infty$ risk and is ultimately necessary for our confusion matrix bounds.

Convergence rates for classification typically require an assumption about how well classes are separated. In binary nonparametric classification, this is commonly specified by the Tsybakov margin condition (also called the Tsybakov or Mammen-Tsybakov noise condition), first proposed by \citet{mammen1999smooth}.
This condition allows for fast rates of convergence for the excess risk; in particular, the rate may be faster than the parametric \(O(n^{-1/2})\) rate.
\cite{audibert2007fast} further demonstrate that plug-in classifiers can achieve very fast rates of convergence under strict conditions. Our results for classification accuracy rely on an appropriate generalization of the Tsybakov noise condition to the weighted multiclass case.

Relative to the large body of statistical theory on kNN classification, little attention has been given to the multiclass case. To the best of our knowledge, the only results are the recent ones of \cite{puchkin2020adaptive}, who consider a method for aggregating over many values of \(k\) in multiclass kNN in order to adapt to the unknown smoothness of the regression function. As in \citet{samworth2012optimal}, the underlying kNN regression estimates are weighted by distance and not by class label. To derive fast convergence rates, \citet{puchkin2020adaptive} introduce a natural multiclass analogue of the margin condition of \citep{mammen1999smooth}. We note that their unweighted setting differs from our class-weighted setting in terms of the margin assumption, kNN classifier, and loss function used.

In statistical learning theory, basic multiclass results in terms of accuracy can be found in standard texts \citep{mohri2012}.
Interestingly, cost-weighting has been studied in the multiclass case before with class imbalance as a motivation \citep{scott2012calibrated} but in the context of calibrated losses, i.e., the guarantee that minimizing a surrogate loss for the zero-one classification loss does in fact lead to a Bayes-optimal estimator \citep{liu2007fisher, tewari2007consistency}.

In addition to weighting, there are three other methods that are commonly used for imbalanced classification: margin adjustment, data augmentation, and Neyman-Pearson classification.
Prior work considers adjusting the margins of SVMs to appropriately handle class-weighting \citep{lin2002support}, and \cite{scott2012calibrated} allows for class-based modifications to the margin as well.
More recent work adjusts the margins for deep neural network classifiers \citep{cao2019learning}.
Although theoretical interest in data augmentation is growing as a result of its success in deep learning \citep{chen2019invariance}, the techniques used for imbalanced classification, particularly SMOTE, are poorly understood \citep{chawla2002smote}. 
A number of variants have followed, including the use of generative adversarial networks (GANs) to produce additional data \citep{mariani2018bagan}.
Neyman-Pearson classification attempts to minimize the misclassification error on one class subject to a constraint on the maximum misclassification error on a second class.
Unlike many methods in imbalanced classification, Neyman-Pearson classification is fairly well-understood theoretically \citep{rigollet2011neyman, tong2013plug, tong2016survey}.
The results build on work in both empirical risk minimization and plug-in methods for nonparametric classification.

Finally, while the focus of our paper is statistical first and foremost, we do point out related computational work, particularly regarding a large number of classes and kNN.
In the text-processing community, the problem of a large number of classes is known as extreme classification.
Most of the work in this area focuses on efficient computation when algorithms must be sublinear in the number of classes \citep{joulin2016bag, yen2018loss}.

There is also extensive work on computing the exact \(k\)-nearest neighbors estimate more efficiently, but, particularly in high dimensions, fast approximations to \(k\)-nearest neighbors are also considered \citep{andoni2006near, andoni2018approximate, indyk2018approximate, dong2019scalable}.
The goal is to design a classifier that can be quickly evaluated on a large amount of data, especially in higher dimensions.
While such approximate algorithms are certainly of interest in the case of modern large-scale datasets, in this paper, we consider simpler exact algorithms and focus on the statistical problem.
Additionally, there is some work that attempts to bridge the gap between computational and statistical efficiency in classification.
Work in this area is motivated by \cite{hart1968condensed}, who first proposed compressing a data sample for use in classification.
\cite{kontorovich2017nearest} considers a compression scheme that leads to a \(1\)-nearest neighbor classifier and proves that this is Bayes consistent in finite-dimensional and certain infinite-dimensional settings.
\cite{gottlieb2018near} studies the computational hardness of approximate sampling compression, gives a compression scheme, and provides a PAC-learning guarantee.
Finally, \cite{efremenko2020fast} provides a compression algorithm based on locality sensitive hashing and quantifies the convergence rate of the excess risk.

%---------------------------------------------%
%---------------------------------------------%
\section{Setup}
\label{sec:Setup}

We consider a sample of \(n\) points \((X_{1}, Y_{1}), \ldots, (X_{n}, Y_{n})\) drawn from some distribution \(P_{X, Y}\) on \(\xspace \times \yspace\) with marginals $P_X$ and $P_Y$.
For our purposes, \((\xspace, \rho)\) is a separable metric space, and we assume \(\yspace = \{1, \ldots, C\} = [C]\).
Let 
\[
\simplex^{C - 1} 
= \left\{a \in \reals^{C} : \sum_{c = 1}^{C} a_{c} = 1, a_{c} \geq 0, \; c = 1, \ldots, C\right\}
\]
be the \((C - 1)\)-dimensional simplex.
The regression function \(\eta: \xspace \to \simplex^{C - 1}\) is defined as
\[
\eta_{c}(x) = \prob\left(Y = c| X = x\right);
\]
that is, given a point $X_i$ we assume the label $Y_i$ has a categorical conditional distribution with mean $\eta(X_i)$.
Additionally, we define \(\eta\) on a subset \(A \subseteq \xspace\) by
\(
\eta_{c}(A)
=
\prob\left(Y = c| X \in A\right).
\)

%---------------------------------------------%
%---------------------------------------------%
\subsection{The Class-Weighted Nearest Neighbor Classifier}

In this section, we define the kNN regressor and classifier.
Given a point \(x\) in \(\xspace\), we define the reordered points \(X_{\sigma_{1}(x)}, \ldots, X_{\sigma_{n}(x)}\) such that
\[
\rho\left(X_{\sigma_{1}(x)}, \; x\right)
\leq 
\ldots
\leq 
\rho\left(X_{\sigma_{n}(x)}, \; x\right).
\]
Tie-breaking procedures are well-known; so we assume without loss of generality that there are no ties.
Let \(S \subseteq \xspace\) be a measurable set.
We define the nearest neighbors regression estimator on \(S\) by
\[
\hat{\eta}_{c}(S) 
= 
\frac{\sum_{i = 1}^{n} \ind\left\{X_{i} \in S, \; Y_{i} = c\right\}}{\sum_{i = 1}^{n} \ind\left\{X_{i} \in S\right\}}.
\]
We are most interested in using this with balls in \(\xspace\).
Thus, we define the open ball of radius $r$ centered at $x$ to be $\openball(x, r) = \{x' \in \X : \rho(x, x') < r\}$ and
the closed ball of radius \(r\) centered at \(x\) to be \(\closedball(x, r) = \{x' \in \X : \rho(x, x') \leq r\}\) .

Now, we turn to defining the kNN regression estimate.
Fix an integer \(k > 0\).
Then, the \(k\)-nearest neighbor regressor at \(x\) for class \(c\) is defined to be
\begin{align}
& \begin{aligned}
\hat{\eta}_{c}(x)
=
\hat{\eta}_{c}\left(\closedball\left(x, \rho\left(X_{\sigma_{k}(x)}, \; x\right)\right)\right).
\label{eqn:KNNRegressor}
\end{aligned}
\end{align}
In words, the kNN classifier \(f_{n, k}(x)\) is the plurality class of the \(k\) nearest points to \(x\).

Finally, we define the \(q\)-weighted kNN classifier.
Let \(q\) be a vector in the positive orthant \(\reals^{C}_{+}\).
Then, the \(q\)-weighted kNN classifier is
\begin{align}
& \begin{aligned}
f_{q, n, k}(x)
=
\argmax_{c \in [C]} \;
q_{c} \hat{\eta}_{c}\left(x\right).
\label{eqn:qWeightedKNN}
\end{aligned}
\end{align}
Intuitively, \(q_{c}\) is the weight put on classifying an example of class \(Y = c\), and
the usual unweighted case corresponds to selecting \(q = (1, \ldots, 1)\).
In the unweighted case, we drop the \(q\) from the classifier subscript and write \(f_{n, k}(x)\) for the standard kNN estimate at \(x\).
Finally, in many of our results, we reference the maximum element of \(q\), and so we write \(q_{\max} = \max_{j = 1, \ldots, C} q_j\).

%---------------------------------------------%
%---------------------------------------------%
\subsection{Error Measures}

In the following sections, we define various error measures and additional terminology.
Accuracy and uniform error are the first two, and these are the quantities for which we derive bounds for kNN.
Subsequently, we consider the confusion matrix.

%---------------------------------------------%
%---------------------------------------------%
\subsubsection{Accuracy and Risk}

In the usual setting, we are interested in two quantities when evaluating the nearest neighbor classifier: the probability of a suboptimal prediction and the risk.
First, we define the \(q\)-Bayes-optimal classifier \(f^{*}_{q}\) by
\[
f^{*}_{q}(x)
= \argmax_{c \in [C]} \; q_{c} \eta_{c}(x).
\]
Again, if \(q = (1, \ldots, 1)\), we drop the subscript and write \(f^*\).
For a new test sample \((X, Y)\), denote let \(\prob_{X}\) be the probability measure with respect to \(X\).
Then, the first quantity we are interested in is the accuracy, denoted by
\[
\prob_{X}\left(f_{q, n, k}(X) \neq f^{*}_{q}(X)\right),
\]
which is the probability of a suboptimal choice on new data.
Note that this is still a random quantity depending on the sample of \(n\) data points.
When we need to take the probability or expectation with respect to the sample, we occasionally emphasize this by writing \(\prob_{n}\) and \(\expect_{n}\).

Second, we consider the risk.
The risk of a classifier \(f\) is
\begin{align*}
& \begin{aligned}
\risk(f) 
&=
\expect_{n}
\prob_{X, Y}\left(f(X) \neq Y\right)
=
\expect_{n}
\expect_{X, Y}
\ind\left\{
f(X) \neq Y
\right\}.
\end{aligned}
\end{align*}
The Bayes risk is \(\risk^{*} = \risk(f^{*}(X))\), and we analyze the excess risk
\[
\excessrisk(f) 
:=
\risk(f) - \risk^{*}
=
\expect_{n} \expect_{X, Y} \ind\left\{
f(X) \neq f^{*}(X)
\right\}.
\]

Finally, we consider the \(q\)-weighted risk.
Define the class-conditioned risk of class \(c\) to be
\[
\risk_{c}(f)
=
\expect_{n}
\prob_{X, Y}\left(
f(X) \neq Y | Y = c
\right).
\]
Note that this is related to the population precision of \(f\) on class \(c\).
Define the vector of marginal probabilities \(p\) by
\(
p_{c}
=
\prob_{X, Y}\left(Y = c\right).
\)
Then, the \(q\)-weighted risk is
\[
\risk_{q}(f)
=
\sum_{c = 1}^{C} q_{c} p_{c} \risk_{c}(f).
\]
Note that the \(q\)-weighted risk does reduce to the usual risk \(\risk(f)\) when \(q = (1, \ldots, 1)\), but we have considered the risk and weighted risks separately because they require different assumptions to analyze.
Additionally, the \(q\)-Bayes classifier \(f_{q}^{*}\) is the minimizer of \(\risk_{q}\), and so we define the \(q\)-Bayes risk to be \(\risk_{q}^{*}\) and the excess \(q\)-risk to be \(\excessrisk_{q}(f) = \risk_{q}(f) - \risk_{q}^{*}\).

%---------------------------------------------%
%---------------------------------------------%
\subsubsection{Uniform Error}

In this section, we consider the uniform error.
We define the uniform error to be 
\begin{equation}
\uniformerror(\hat{\eta})
=
\max_{c \in [C]} \|\hat{\eta}_{c} - \eta_{c}\|_{\xspace, \infty},
\label{eqn:UniformError}
\end{equation}
where, for a function $f : \X \to \R$, $\|f\|_{\xspace, \infty} := \sup_{x \in \X} |f(x)|$ denotes the $\sup$-norm of $f$.

For our results on uniform error, we impose additional requirements on our space and regression function, and so we discuss additional notation.
One of the key assumptions is that \((X, \rho)\) is a totally bounded metric space, and
under the assumption that \((\X, \rho)\) is totally bounded, we may introduce covering numbers.
For any $r > 0$, let $N(r)$ denote the $r$-covering number of $(\X,\rho)$; that is, $N(r)$ is the smallest positive integer such that, for some $x_1,...,x_{N(r)} \in X$,
\[
\X \subseteq \bigcup_{i = 1}^{N(r)} \openball(x_i, r).
\]
Finally, for positive integers $n$, let
\[S(n)
:= \sup_{x_1,...,x_n \in \X} \left| \left\{ \{x_1,...,x_n\} \cap \openball(x, r) : x \in \X, r \geq 0 \right\} \right|
\]
denote the shattering coefficient of the class of open balls in $(\X, \rho)$.

%---------------------------------------------%
%---------------------------------------------%
\subsubsection{General Classification Metrics}

In this section, we consider general classification metrics.
A number of commonly-used metrics are derived from the confusion matrix of a classifier, i.e., the true positive, true negative, false positive, and false negative rates. 
Presently, we define the confusion matrix entries.

Usually, practitioners consider the empirical versions of the confusion matrix entries, and we denote these by
\(\widehat{\truepositive}\), \(\widehat{\truenegative}\), \(\widehat{\falsepositive}\), and \(\widehat{\falsenegative}\) respectively. For simplicity, let 
\begin{align*}
M(q, c, x) 
&= \max_{j \neq c} \; q_{j} \eta_{j}(X)
&
\hat{M}(q, c, x)
&= \max_{j \neq c} \; q_{j} \hat{\eta}_{j}(X).
\end{align*}
The confusion matrix entries are then defined by
\begin{align}
&\begin{aligned}
\widehat{\truenegative}_{c}(q) 
&= \frac{1}{n} \sum_{i = 1}^{n} \ind\left\{Y_{i} \neq c\right\} \ind\left\{q_{c} \hat{\eta}_{c}(X_{i}) < \hat{M}(q, c, X_{i})\right\}  
\\ % &\qquad 
\widehat{\falsenegative}_{c}(q) 
&= \frac{1}{n} \sum_{i = 1}^{n} \ind\left\{Y_{i} = c\right\} \ind\left\{q_{c} \hat{\eta}_{c}(X_{i}) < \hat{M}(q, c, X_{i})\right\}\\
\widehat{\falsepositive}_{c}(q) 
&= \frac{1}{n} \sum_{i = 1}^{n} \ind\left\{Y_{i} \neq c\right\} \ind\left\{q_{c} \hat{\eta}_{c}(X_{i}) \geq \hat{M}(q, c, X_{i})\right\}
\\ % &\qquad 
\widehat{\truepositive}_{c}(q) 
&= \frac{1}{n} \sum_{i = 1}^{n} \ind\left\{Y_{i} = c\right\} \ind\left\{q_{c} \hat{\eta}_{c}(X_{i}) \geq \hat{M}(q, c, X_{i})\right\}.
\end{aligned}
\end{align}
In contrast to the usual empirical use of the confusion matrix, we normalize each entry in order to discuss convergence.
The population confusion matrix is a result of evaluating the confusion matrix quantities on the true distribution, which gives
\begin{align}
&\begin{aligned}
\truenegative_{c}(q) 
&= \int_{\xspace} (1 - \eta_{c}(X)) 
\ind\left\{q_{c} \eta_{c}(X) < M(q, c, X)\right\} dP_{X}  
\\
\falsenegative_{c}(q) 
&= \int_{\xspace} \eta_{c}(X) 
\ind\left\{q_{c} \eta_{c}(X) < M(q, c, X)\right\} dP_{X} \\
\falsepositive_{c}(q) 
&= \int_{\xspace} (1 - \eta_{c}(X)) 
\ind\left\{q_{c} \eta_{c}(X) \geq M(q, c, X)\right\} dP_{X} \\ 
\truepositive_{c}(q) 
&= 
\int_{\xspace} \eta_{c}(X) 
\ind\left\{q_{c} \eta_{c}(X) \geq M(q, c, X)\right\} dP_{X}.
\end{aligned}
\end{align}

As we must discretize the space of weights, we define the weights that are covered by a given discretization.
\begin{definition}
A weight \(q\) is class \(c\)-covered by \((q', q'')\) if
\begin{align*}
q_{c} &< q'_{c} &\qquad q_{j} > q'_{j} \text{ for all } j \neq c, \\
q_{c} &> q''_{c} & \qquad q_{j} < q''_{j} \text{ for all } j \neq c.
\end{align*}
Let \(Q \subset \reals^{C}\) be a finite set of weights.
We say that \(Q\) class \(c\)-covers \(q\) if it contains a class \(c\)-cover \((q', q'')\), and we say that \(Q\) is a multiclass cover if it is a \(c\) cover for all \(c\) in \([C]\).
\label{def:CoveredWeight}
\end{definition}

%---------------------------------------------%
%---------------------------------------------%
%---------------------------------------------%
%---------------------------------------------%
\section{Main Results}
\label{sec:MainResults}

In this section, we present our main results for accuracy, risk, and uniform error of nearest neighbor classifiers.

%---------------------------------------------%
%---------------------------------------------%
\subsection{Accuracy and Risk}

First, we consider bounds on the accuracy, risk, and weighted risk.
The results and proofs divide the space \(\xspace\) into two parts: a ``good'' region where the optimal classification is made with high probability and a ``bad'' region near the true decision boundary.
The results are largely multiclass extensions from \cite{chaudhuri2014rates}, although we note that the rate differs in \(n\) for the multiclass risk.

Let \(\Delta\) in \((0, 1]\) be a constant.
We define the effective boundary \(\partial_{p, q, \Delta}\) by
\[
\partial_{p, q, \Delta}
  = \biggr\{
x \in \xspace
\; \biggr| \text{ for some $r \leq r_{p}(x)$, }\; q_{f^{*}_{q}(x)} \eta_{f^{*}_{q}(x)}(B(x, r)) <  \max_{c \neq f^{*}_{q}(x)} q_{c} \eta_{c}(B(x, r)) + \Delta
\biggr\}
\]
where \(r_{p}(x) = \inf\{r: P_{X}(\closedball(x, r)) \geq p\}.\)
That is, $\partial_{p, q, \Delta}$ is the set of points $x$ such that the difference in probability between the optimal class $f^{*}_{q}(x)$ at $x$ and some other class, on some small ball $B(x, r)$, is less than $\Delta$.
We can now state the theorem.

\begin{theorem}
Let \(\delta\) be in \((0, 1)\), and pick a positive integer \(k \leq n\).
Define the terms
\begin{align*}
p
=
\frac{k}{n} \cdot \frac{1}{1 - \sqrt{(4 / k) \log (2 / \delta)}}
\quad \text{ and } \quad
\Delta
=
\min\left(1, \; \sqrt{\frac{2 q_{\max}^{2}}{k}\left(\log C  + 2\log \frac{2}{\delta}\right)}
\right).
\end{align*}
Then, with probability at least \(1 - \delta\) with respect to the training data, for a new sample \((X, Y)\) we have
\begin{align*}
& \begin{aligned}
\prob\left(f_{q, n, k}(X) \neq f^{*}_{q}(X)\right)
&\leq
\delta
+
P_{X}\left(\partial_{p, q, \Delta}\right).
\end{aligned}
\end{align*}
\label{theorem:Theorem5Analogue}
\end{theorem}

Theorem~\ref{theorem:Theorem5Analogue} tells us that, with high probability over the training data, sub-optimal classifications of new test points are most likely to occur within the effective boundary $\partial_{p,q,\Delta}$.
Compared to binary classification, here the boundary terms \(\Delta\) must be larger by a factor of \(\log C\) within the root.
Of course, the effect it has on the error depends on the underlying distribution.
This theorem can be specialized to more traditional cases of interest where we impose more restrictions.
In the following sections, we consider smooth regression functions and a margin condition.

%---------------------------------------------%
%---------------------------------------------%
\subsubsection{Smooth Measures}

In this section, we specialize our accuracy result to smooth regression functions.
The first step is to define a notion of smoothness.
In the binary classification case of \cite{chaudhuri2014rates}, the authors define a regression function \(\eta\) to be \((\alpha, L)\)-smooth if
\[
|\eta(x) - \eta(\openball(x, r))| \leq L P_{X}(\openball(x, r))^{\alpha}.
\]
The regression function \(\eta\) is observed to be \((\alpha, L)\)-smooth if it is \(d \alpha\)-H\"{o}lder.

In the multiclass case, we define smoothness entry-wise.
Given scalars \(\alpha\) in \((0, 1 / d)\) and \(L > 0\), we say that \(\eta\) is \((\alpha, L)\)-smooth if
\begin{equation}
|\eta_{c}(x) - \eta_{c}(\closedball(x, r))|
\leq
LP_{X}\left(\closedball(x, r)\right)^{\alpha}
\label{eqn:SmoothnessMulticlass}
\end{equation}
for each \(c\) in \([C]\).
As in the univariate case, this is implied if the \(c\)th coordinate of the regression function is \(d\alpha\)-H\"{o}lder.

\begin{corollary}
Let \(\eta\) be an \((\alpha, L)\)-smooth function.
Then with probability at least \(1 - \delta\), we have the upper bound
\begin{align*}
&\begin{aligned}
\prob(&f_{q, n, k}(X) \neq f^{*}_{q}(X))
\leq
\delta +
\prob\biggr(
q_{f^{*}_{q}(X)}\eta_{f^{*}_{q}(X)}(X)
\leq
 \\ &\qquad
q_{c} \eta_{c}(X) + \sqrt{\frac{2q_{\max}^{2}}{k}\left(\log C + 2 \log \frac{2}{\delta}\right)}
+
2q_{\max}L\left(\frac{2k}{n}\right)^{\alpha}
\text{ for some }c \neq f^{*}_{q}(X)
\biggr).
\end{aligned}
\end{align*}
\label{corollary:SmoothTheorem}
\end{corollary}

\begin{remark}
\label{remark:Smooth}
Suppose that \(k \geq 16 \log 2 / \delta\) so that \(p \leq 2k / n\).
The optimal choice of \(k\) is
\[
k
=
M_{\alpha, L}
q_{\max}^{\frac{2}{2\alpha + 1}}
\left(2 \log C + 4 \log \frac{2}{\delta}\right)^{\frac{1}{2\alpha + 1}}
n^{\frac{2\alpha}{2\alpha+ 1}},
\]
where \(M_{\alpha, L}\) is some constant depending on \(\alpha\) and \(L\).
This leads to the bound
\begin{align*}
&\begin{aligned}
\prob&(f_{q, n, k}(X) \neq f^{*}_{q}(X))
\leq
\delta
\\
&
+
\prob\left(
\eta_{f_{q}^{*}(X)}(X)
\leq
\eta_{c}(X)
+
M'_{\alpha, L}
q_{\max}^{\frac{2\alpha}{2\alpha + 1}}
\left(2 \log C + 4 \log \frac{2}{\delta}\right)^{\frac{\alpha}{2\alpha + 1}}
n^{-\frac{\alpha}{2\alpha+ 1}}
\text{ for some } c \neq f^{*}_{q}(X)
\right).
\end{aligned}
\end{align*}
Note that both \(k\) and the error probability depend sub-logarithmically on the number of classes \(C\).
\end{remark}

%---------------------------------------------%
%---------------------------------------------%
\subsubsection{Margin Conditions}

One of the key assumptions for obtaining fast rates in plug-in classification is the Tsybakov margin condition \citep{audibert2007fast}.
Here, we consider two variants: a weighted margin condition and a conditional margin assumption.
These allow us to obtain concrete rates of convergence for the excess risk of the kNN classifier.

Define \(\Delta_{q}(x) = q_{f^{*}_{q}(x)} \eta_{f^{*}_{q}(x)}(x) -  \max_{c \neq f^{*}_{q}(x)} q_{c} \eta_{c}(x) \).
The weighted margin condition is
\begin{equation}
\prob\left(
\Delta_{q}(X) \leq t
\right)
\leq
M t^{\beta}.
\label{eqn:TsybakovMargin}
\end{equation}
If we have \(q = (1, \ldots, 1)\), then we recover the usual Tsybakov margin condition.
Under this condition, we have the following result.

\begin{corollary}
Let \(\eta\) be an \((\alpha, L)\)-smooth function, and suppose that \(P_{X, Y}\) satisfies the \((\beta, M)\)-margin condition.
Pick \(k\) according to Remark~\ref{remark:Smooth}.
Then, we have
\begin{align*}
&\begin{aligned}
\prob(f_{q, n, k}(X) \neq f^{*}_{q}(X))
&\leq
\delta
+
M_{\alpha, \beta, L}
q_{\max}^{\frac{2\alpha \beta}{2\alpha + 1}}
\left(2 \log C + 4 \log \frac{2}{\delta}\right)^{\frac{\alpha \beta}{2\alpha + 1}}
n^{-\frac{\alpha\beta}{2\alpha + 1}},
\end{aligned}
\end{align*}
where \(M_{\alpha, \beta, L} = M \cdot (M'_{\alpha, L})^{\beta}\) is a constant.
Additionally, if we have the uniform weights \(q = (1, \ldots, 1)\) and  pick \(k
=
M'_{\alpha, \beta, L} n^{2\alpha / (2\alpha + 1)},\)
then
we can bound the excess risk by
\[
\excessrisk(f_{n, k})
%\expect_{X} \expect_{n} \left[\risk_{n, k}(X) - \risk^{*}(X)\right]
\leq
M''_{\alpha, \beta, L} \cdot C n^{-\frac{\alpha\beta}{2\alpha + 1}}.
\]
\label{corollary:MarginTheorem}
\end{corollary}

%---------------------------------------------%
%---------------------------------------------%
We can also consider weighted risk.
Since the weighted risk is a linear combination of conditional risks, this requires a conditional margin assumption.

\begin{assumption}
A distribution \(P_{X, Y}\) satisfies the \((\beta, M)\)-conditional margin assumption if
\[
\prob\left(
\Delta_{q}(X) \leq t | Y = c
\right)
\leq
M t^{\beta}
\]
for some constant \(M > 0\) and every \(c\) in \([C]\).
\label{assumption:ConditionalMargin}
\end{assumption}

\begin{proposition}
Assume that \(\eta\) is an \((\alpha, L)\)-smooth function and that the distribution \(P_{X, Y}\) satisfies the \((\beta, M)\)-conditional margin assumption.
Then, we have
\begin{align*}
& \begin{aligned}
\excessrisk_{q}(f_{q, n, k})
=
O\left(
C n^{-\frac{\alpha \beta}{2\alpha + 1}}
\sum_{c = 1}^{C} q_{c} p_{c}
\right).
\end{aligned}
\end{align*}
\label{proposition:WeightedRisk}
\end{proposition}

%---------------------------------------------%
%---------------------------------------------%
\subsubsection{Lower Bounds}
\label{subsubsec:AccuracyLowerBound}

Finally, as a converse to Theorem~\ref{theorem:Theorem5Analogue}, we provide a lower bound for accuracy.
The strategy for the lower bound is somewhat different in that we need to reduce to the binary case; hence, the set of \(x\) in \(\xspace\) providing the lower bound does not naturally mirror the corresponding set \(\partial_{p, q, \Delta}\) in the upper bound as closely as in the binary classification case.

We define our set to be \(\boundaryproximal_{q, n, k} \subseteq \xspace\), and since the definition is lengthy, we first state a few auxiliary definitions.
Let \(t_{c}(x) = q_{f^{*}_{q}(x)} / (q_{f^{*}_{q}(x)} + q_{c})\), and define \(p'(x, r) = \eta_{c}(\closedball') / (\eta_{c}(\closedball') + \eta_{f^{*}_{q}(x)}(\closedball'))\) where \(\closedball' = \closedball(x, \rho(x, r))\).
Thus, we define
\begin{align}
&\begin{aligned}
\boundaryproximal_{q, n, k}
&=
\bigg\{
x \in \xspace \bigg|
\text{ there exists } c \neq f^{*}_{q}(x) \text{ such that for all } r \in \left[r_{k/n}(x), r_{(k + \sqrt{k} + 1) / n}(x)\right], \\ 
& \qquad \text{(A) } 0 < p'(x, r) < 1/2, \text{ (B) } t_{c} < 1 - p'(x, r), \text{ and (C) } 0 < t_{c} - 1/\sqrt{k} \leq p'(x, r)
\bigg\}
\end{aligned}
\end{align}
Conditions (A) and (B) are specific assumptions for binomial distributions and, in the unweighted case, can be ignored.
Condition (C) is the requirement that \(x\) is near the decision boundary.
Now, we can state the theorem.
\begin{theorem}
There exists a constant \(\gamma > 0\) such that 
\(
\gamma P_{X}(\boundaryproximal_{q, n, k})
\leq 
\expect_{n} \prob_{X}\left(
f_{q, n, k}(x) \neq f^{*}_{q}(x) 
\right).
\)
	\label{theorem:AccuracyLowerBound}
\end{theorem}
%---------------------------------------------%
%---------------------------------------------%
\subsection{Uniform Error}

We now give bounds on the uniform error \(\uniformerror(\hat{\eta})\) defined in \eqref{eqn:UniformError}.
We start with additional assumptions.

\begin{assumption}
We make the following three assumptions.
\begin{itemize}
\item[(a)] The pair $(\X, \rho)$ is a totally bounded metric space.
%f_{q, n, k} doesn't seem like the right symbol here.
\item[(b)] For any positive integers \(k < n\), let \(f_{q, n, k}\) denote the marginal $P_X$ is lower bounded in the sense that, for some constants $p_*,r^*,d > 0$, for any point $x$ in $\X$ and radius $r$ in $(0,r^*]$, we have the inequality $P_X(B_r(x)) \geq p_* r^d$.
\item[(c)] Each $\eta_c$ is $\alpha$-H\"older continuous with constant $L$.
\end{itemize}
\label{assumption:UniformError}
\end{assumption}
We now provide our bound on the uniform error, proved in Appendix~\ref{app:UniformConvergenceProof}.

\begin{theorem}
    Under Assumption~\ref{assumption:UniformError}, whenever $k / n \leq p_*(r^*)^d / 2$,
    for any
    $\delta > 0$, with probability at least $1 - N\left( \left( 2k / (p_* n) \right)^{1/d} \right) e^{-k/4} - \delta$,
    we have the uniform error bound
    \begin{equation}
    U(\hat \eta_k) \leq 2^\alpha L\left( \frac{2k}{p_* n} \right)^{\alpha/d} + \frac{1}{\sqrt{k}} + \sqrt{\frac{1}{2k} \log \frac{S(n)}{\delta}}.
    \label{eq:uniform_error_bound}
    \end{equation}
    \label{thm:multiclass_unif_convergence}
\end{theorem}

Of the three terms in \eqref{eq:uniform_error_bound}, the first term, of order $(k/n)^{\alpha/d}$, comes from implicit smoothing bias of the kNN classifier, while the remaining two terms, effectively of order $\sqrt{k^{-1} \log (S(n) /\delta) }$, come from variance due to label noise.
Notably, in contrast to the risk bounds in the previous section, this bound is entirely independent of the number $C$ of classes; indeed the result holds even in the case of infinitely many classes. While this may initially be surprising (as the estimand $\hat \eta$ is function taking $C$-dimensional values), the constraint that $\hat \eta$ lies in the $(C - 1)$ dimensional probability simplex is sufficiently restrictive to prevent a growth in error with $C$; that is, adding new classes that occur only with very low probability does not increase the uniform error.
As shown in the next two examples, $k$ can be selected to balance these two terms.

\begin{corollary}[Euclidean, Absolutely Continuous Case]
Let \(V_{d}\) denote the volume of the unit \(\ell_{2}\) ball in \(\reals^{d}\).
Suppose $(\X,\rho) = ([0,1]^d,\|\cdot\|_2)$ is the unit cube in $\R^d$, equipped with the Euclidean metric, and $P_X$ has a density that is lower bounded by $2^dp_* / V_d > 0$, so that, on any ball $B_r(x)$ of radius at most $r \leq 1$, $P_X(B_r(x)) \geq cr^d$. Then, $N(r) \leq (2/r)^d$ and $S(n) \leq 2n^{d + 1} + 2$. Hence, by Theorem~\ref{thm:multiclass_unif_convergence}, with probability at least $1 - p_*n e^{-k/4} / k - \delta$,
\[U(\hat \eta_k) \leq 2^{\alpha} L \left( \frac{2k}{p_* n} \right)^{\alpha/d} + \frac{1}{\sqrt{k}} + \sqrt{\frac{1}{2k} \log \left( \frac{4n^{d + 1}}{\delta} \right)}.\]
This bound is minimized when
$k \asymp n^{\frac{2\alpha}{2\alpha+d}} (\log n)^{\frac{d}{2\alpha+d}}$, giving
\[U(\hat \eta_k)
  = O \left( \left( \frac{\log n}{n} \right)^{\frac{\alpha}{2\alpha+d}} \right).\]
  \label{corrollary:euclidean_AC_uniform_rate}
\end{corollary}

\begin{corollary}[Implicit Manifold Case]
Suppose $Z$ is a $[0,1]^d$-valued random variable
with a density lower bounded by $2^dp_* / V_d > 0$, and suppose that, for some $1$-Lipschitz map $T : [0,1]^d \to \R^D$, $X = T(Z)$. Then, one can check that on any ball $B_r(x)$ of radius at most $r \leq 1$, we have the three inequalities $P_X(B_r(x)) \geq cr^d$, $N(r) \leq (2/r)^d$, and $S(n) \leq 2n^{D + 1} + 2$. Hence, by Theorem~\ref{thm:multiclass_unif_convergence}, with probability at least $1 - p_*n e^{-k/4} / k - \delta$,
\[
U(\hat \eta_k) \leq 2^{\alpha} L \left( \frac{2k}{p_* n} \right)^{\alpha/d} + \frac{1}{\sqrt{k}} + \sqrt{\frac{1}{2k} \log \left( \frac{4n^{D + 1}}{\delta} \right)}.
\]
This bound is minimized when
$k \asymp n^{\frac{2\alpha}{2\alpha+d}} (\log n)^{\frac{d}{2\alpha+d}}$, giving, for fixed $D$,
\[U(\hat \eta_k)
  = O \left( \left( \frac{\log n}{n} \right)^{\frac{\alpha}{2\alpha+d}} \right).\]
\end{corollary}
The latter example demonstrates that, when data implicitly lie on a low-dimensional manifold within a high-dimensional feature space, convergence rates depend only on the intrinsic dimension $d$, which may be much smaller than the number $D$ of features.

%---------------------------------------------%
%---------------------------------------------%
\subsubsection{Lower Bounds}

We close this section with the following lower bound on the minimax uniform error, which shows that the rate provided in Theorem~\ref{thm:multiclass_unif_convergence} is minimax optimal over H\"older regression functions:

\begin{theorem}
    Suppose $\X = [0,1]^d$ is the $d$-dimensional unit cube and the marginal distribution of $X$ is uniform on $\X$. Let $\Sigma^\alpha(L)$ denote the family of $\alpha$-H\"older continuous regression function with Lipschitz constant $L$, and, for any particular regression function $\eta$, let $P_\eta$ denote the joint distribution of $(X,Y)$ under that regression function. Then, for any $\alpha, L > 0$, there exist constants $n_0$ and $c > 0$ (depending only on $\alpha$, $L$, and $d$) such that, for all $n \geq n_0$
    \[\inf_{\hat \eta} \sup_{\eta \in \Sigma^\alpha(L)} \prob_{\{(X_i,Y_i)\}_{i = 1}^n \sim P_\eta^n} \left( \left\|\eta - \hat \eta \right\|_\infty \geq c \left( \frac{\log n}{n} \right)^{\frac{\alpha}{2\alpha + d}}\right) \geq 1/8,\]
    where the infimum is taken over all estimators $\hat \eta$.
    \label{theorem:UniformErrorLowerBound}
\end{theorem}

%---------------------------------------------%
%---------------------------------------------%

\section{Application to General Classification Metrics}
\label{sec:ApplicationGeneralClassification}

In this section, we consider general classification metrics to which Theorem~\ref{thm:multiclass_unif_convergence} may be applied.
We start by examining upper and lower bounds for estimating confusion matrix quantities; then we turn to precision, recall, and F1 score, which are derived from the confusion matrix.

%---------------------------------------------%
%---------------------------------------------%
\subsection{Confusion Matrix Entries}
First, we consider the uniform convergence of empirical confusion matrix entries to the population confusion matrix.
The results of this section do not depend on the choice of regression estimator; so we start with an assumption on the uniform error of the regression estimator.

\begin{assumption}
We assume the regression function estimate converges uniformly.
More precisely, for every \(\epsilon > 0\), we have 
\(
\uniformerror(\hat{\eta}) 
\leq 
\epsilon
\)
with probability at least \(1 - \delta /2\).
\label{assumption:UniformEstimate}
\end{assumption}

We showed that this assumption holds for nearest neighbors under additional conditions in the previous section.
We now turn to definitions used in our main theorem of this section.
Recall that we define \(q_{\max} = \max_{j = 1, \ldots, C} q_j\).
Define the quantities
\begin{align*}
r(q, c) 
& %= \frac{q_{c} + q_{\max}}{q_{c}} 
=
1 + \frac{q_{\max}}{q_{c}},
\label{eqn:MainResults:rqc}
&
t(q, c, x)
&=
%\frac{1}{q_{c}} M(q, c, x) =
\frac{1}{q_{c}} \max_{j \neq c} q_{j} \eta_{j}(x).
%\label{eqn:MainResults:tqcx}
\end{align*}
For brevity, we write \(r'\) and \(t'\) to denote \(r(q', \cdot)\) and \(t(q', \cdot, \cdot)\), and we handle double-primes similarly.
Now, we state the theorem.

\begin{theorem}
Suppose that Assumption~\ref{assumption:UniformEstimate} holds.
Consider a set of weights \(Q\), and let \(Q_{\cover}\) be a multiclass cover for \(Q\) of size \(N\).
Let \(q\) be an element of \(Q\) that is multiclass covered by \((q', q'')\), where \(q'\) and \(q''\) are in \(Q_{\cover}\).
Define
\begin{align*}
\tne(q', q'', c)
&= \fpe(q', q'', c) = 
\prob\left(
Y \neq c, \;
t'(c, X) - \epsilon r'(c) \leq \eta_{c}(X) \leq t''(c, X) + \epsilon r''(c)
\right) \\
\fne(q', q'', c)
&= \tpe(q', q'', c) = 
\prob\left(
Y = c, \;
t'(c, X) - \epsilon r'(c) \leq \eta_{c}(X) \leq t''(c, X) + \epsilon r''(c)
\right)
\\
%\fpe(q', q'', c)
%&= 
%\prob\left(
%Y \neq c, \;
%t'(c, X) - \epsilon r'(c) \leq \eta_{c}(X) \leq t''(c, X) + \epsilon r''(c)
%\right)
%\\
%\tpe(q', q'', c)
%&=
%\prob\left(
%Y = c, \;
%t'(c, X) - \epsilon r'(c) \leq \eta_{c}(X) \leq t''(c, X) + \epsilon %r''(c)
%\right)
\end{align*}
Then with probability at least \(1 - \delta\), we have
\begin{align*}
|\hat{\truenegative}_{c}(q) - \truenegative_{c}(q)|
&\leq 
3 \tne(q', q'', c)
+
3 \sqrt{\frac{\log \frac{24N}{\delta} + 2 \log C}{2n}} 
=:
E_{\truenegative, c}(q)
\\
|\hat{\falsenegative}_{c}(q) - \falsenegative_{c}(q)|
&\leq 
3 \fne(q', q'', c)
+
3 \sqrt{\frac{\log \frac{24N}{\delta} + 2 \log C}{2n}} 
=:
E_{\falsenegative, c}(q)
\\ 
|\hat{\falsepositive}_{c}(q) - \falsepositive_{c}(q)|
&\leq 
3\fpe(q', q'', c) 
+
3 \sqrt{\frac{\log \frac{24N}{\delta} + 2 \log C}{2n}} 
=:
E_{\falsepositive, c}(q)
\\
|\hat{\truepositive}_{c}(q) - \truepositive_{c}(q)|
&\leq 
3 \tpe(q', q'', c)
+
3 \sqrt{\frac{\log \frac{24N}{\delta} + 2 \log C}{2n}}
=:
E_{\truepositive, c}(q)
\end{align*}
for all \(c\) in [C] and \(q\) in \(Q\) simultaneously.
\label{theorem:MulticlassConfusionMatrixBound}
\end{theorem}

Simultaneous uniform convergence of the confusion matrix entries is a powerful tool for proving the uniform convergence of general classification metrics.
We consider an example in Section~\ref{sec:Example}.
However, we do note that this bound cannot be calculated from the data itself, unlike in empirical risk minimization, since the bounds depend on the underlying data distribution.

%---------------------------------------------%
%---------------------------------------------%
\subsection{Lower Bounds}
\label{subsec:ConfusionMatrix:LowerBounds}

In this section, we consider a lower bound for estimating confusion matrix entries.
We make two simplifications. 
First we just consider estimating the true negatives, since similar results may be derived analogously for the true positives, false positives, and false negatives.
Second, we consider the binary case.
In terms of notation, we consider a threshold \(t\), which in binary classification is equivalent to \(t = q_{1} / (q_{1} + q_{2})\).
For a more detailed treatment of the binary case, see Appendix~\ref{sec:AppConfusionMatrixQuantitiesBinary}.

\begin{proposition}
Consider the class of probability distributions \(\mathfrak{P}(\alpha, L)\)  consisting of joint distributions of \((X, Y)\) on the space \(\xspace = [0, 1]^{d} \cup \{x'\}\) for some \(x'\) not in \([0, 1]^{d}\)  such that the regression function belongs to the class \(\Sigma^{\alpha}(L)\) of \(\alpha\)-H\"{o}lder functions with constant \(L\).
For any estimator \(\hat{\truenegative}\), there exists a distribution \(P_{X, Y}\) belonging to \(\mathfrak{P}(\alpha, L)\) for an appropriately-chosen \(L\) such that with probability at least \(\delta < 1/4\), there is a threshold \(t\) such that
\[
\frac{p}{4} 
\leq 
|\hat{\truenegative}(t) - \truenegative(t)|
\]
where \(p\) is defined by
\(
p
=
\prob\left(t - \epsilon \leq \eta(X) \leq t + \epsilon
\right)
\)
and \(\epsilon\) can be chosen to satisfy the inequality
\(
\epsilon
\geq 
B
\left(
\frac{\log 1 / (4 \delta)}{p n}
\right)^{\alpha / (2\alpha + d)}
\)
for some constant \(B\).
\label{prop:LowerBoundAlpha}
\end{proposition}

At this point, we take a moment to clarify how to view Proposition~\ref{prop:LowerBoundAlpha} as a natural lower bound to Theorem~\ref{theorem:MulticlassConfusionMatrixBound}.
First, the lower bound \(p\) is the natural analogue to \(\tne(q', q'', c)\).
Here, since we have to specify the distribution for the lower bound, we get the actual threshold \(t\) in the expression for \(p\) instead of \(t'(c, X)\) and \(t''(c, X)\) as in the expression for \(\tne(q', q'', c)\) from a covering argument.
Second, for the class of \(\alpha\)-H\"{o}lder regression functions, we obtain the proper rate in \(\epsilon\).
Using Theorem~\ref{thm:multiclass_unif_convergence} in conjunction with Theorem~\ref{theorem:MulticlassConfusionMatrixBound}, the \(\epsilon\) in \(\tne(q', q'', c)\) is of order \(O(n^{-\alpha/(2\alpha + d)})\) up to logarithmic factors.
Our lower bound contains such an \(\epsilon\) of order \(\Omega(n^{-\alpha/(2\alpha + d)})\), and so the bounds match in \(\epsilon\).
Finally, the expansion of \(\xspace\) via the union to include the discrete point \(x'\) is solely to control the size of \(p\) and may be omitted by simply considering \(p = 1\).

%---------------------------------------------%
%---------------------------------------------%

%---------------------------------------------%
%---------------------------------------------%
\subsection{Precision, Recall, and F1 Score}
\label{subsec:Examples:PrecisionEtc}

Now, we turn to the uniform convergence of class-wise precision, class-wise recall, and macro-averaged F1 score with respect to weights.
Similar results can be obtained for other general classification metrics such as macro-averaged and micro-averaged precision and recall, usually simultaneously with these results in the sense that there is no further degradation of the probability with which these events hold.

First, we define class-wise precision, recall, and F1 scores by
\begin{align*}
\precision_{c}(q)
&= 
\frac{\truepositive_{c}(q)}{\truepositive_{c}(q) + \falsenegative_{c}(q)}
& %\\
\recall_{c}(q)
&=
\frac{\truepositive_{c}(q)}{\truepositive_{c}(q) + \falsenegative_{c}(q)} % \\
& \fone_{c}(q)
&=
2 \frac{\precision_{c}(q) \cdot \recall_{c}(q)}{\precision_{c}(q) + \recall_{c}(q)}.
\end{align*}
Next, the macro-averaged F1 score is
\begin{align}
\fone(q)
&=
\frac{1}{C} \sum_{j = 1}^{C} \fone_{c}(q).
\label{equation:macroF1}
\end{align}
We note that the empirical versions of these quantities may be obtained by simply replacing the confusion matrix entries by their empirical counterparts, and we denote them by \(\hat{\precision}_{c}\), \(\hat{\recall}_{c}\), \(\hat{\fone}_{c}\), and \(\hat{\fone}\).
Now, we present our corollaries for these quantities.

\begin{corollary}
Define \(E_{\precision, c}(q)\) and \(E_{\recall, c}(q)\) by
\begin{align*}
E_{\precision, c}(q)
&=
3 
\frac{E_{\truepositive, c}(q) + E_{\falsepositive, c}(q)}{\truepositive_{c}(q) + \falsepositive_{c}(t) 
- E_{\truepositive, c}(q) - E_{\falsepositive, c}(q)} \\ 
E_{\recall, c}(q)
&=
3 
\frac{E_{\truepositive, c}(q) + E_{\falsenegative, c}(q)}{\truepositive_{c}(q) + \falsenegative{c}(t) 
- E_{\truepositive, c}(q) - E_{\falsenegative, c}(q)}.
\end{align*}
Suppose Assumption~\ref{assumption:UniformEstimate} holds.
Then we have
\begin{align*}
|\hat{\precision}_{c}(q) - \precision_{c}(q)|
&\leq 
E_{\precision, c}(q)   
&\text{ and }&&
|\hat{\recall}_{c}(q) - \recall_{c}(q)|
&\leq 
E_{\recall, c}(q)
\end{align*}
for all \(q\) for which \(E_{\precision, c}(q) > 0\) and
for all \(q\) for which \(E_{\recall, c}(q) > 0\) with probability at least \(1 - \delta\).
\label{corollary:Examples:PrecisionRecall}
\end{corollary}

%Now, we consider F1 score.

\begin{corollary}
Suppose that Assumption~\ref{assumption:UniformEstimate} holds.
Define \(E_{\fone, c}(q)\) by
\[
E_{\fone, c}(q)
=
9
\frac{E_{\precision, c}(q) + E_{\recall, c}(q)}{\precision_{c}(q) + \recall_{q}(q) - E_{\precision, c}(q) - E_{\recall, c}(q)}.
\]
Then with probability at least \(1 - \delta\), we have
\(
|\hat{\fone}_{c}(q) - \fone_{c}(q)|
\leq 
E_{\fone, c}(q)
\)
for each \(c\) in \([C]\), and consequently we obtain
\[
|\hat{\fone}(q) - \fone(q)|
\leq 
\frac{1}{C} \sum_{j = 1}^{C} 
E_{\fone, c}(q).
\]
\label{corollary:Examples:F1Score}
\end{corollary}

%---------------------------------------------%
%---------------------------------------------%
\section{An Example}
\label{sec:Example}

In this section, we consider an example with three classes on the unit interval to make our results on general classification metrics clearer.
Specifically, we specify a concrete distribution, illustrate the effects of weighting on optimal class selection, and compute the theoretical errors for estimating the confusion matrix implied by our bounds.
Finally, we investigate the error empirically through a simulation study. Python code for reproducing the results in this section is available at \url{https://github.com/neilzxu/weighted_knn_classification}.

First, we specify the distribution.
We let \(X\) be drawn uniformly from \([0, 1]\), and we set 
\begin{align*}
    \eta_{1}(x) &= \exp(-2x) \cos^{2}(4\pi x) \\
    \eta_{2}(x) &= (1 - x)(1 - \eta_{1}(x)) \\ 
    \eta_{3}(x) &= x (1 - \eta_{1}(x)).
\end{align*}
Note that the vector of marginal probabilities is approximately \(p = (0.22, 0.35, 0.43)\).

Next, we examine the effect of weighting.
Figure~\ref{figure:Example:RegressionWeighting} gives the unweighted and weighted regression function for the weighting \(q = (0.5, 0.3, 0.2)\), which emphasizes the first class most prominently.
\begin{figure}[t]
\centering 
\minipage{0.5\textwidth}
\includegraphics[width=3in]{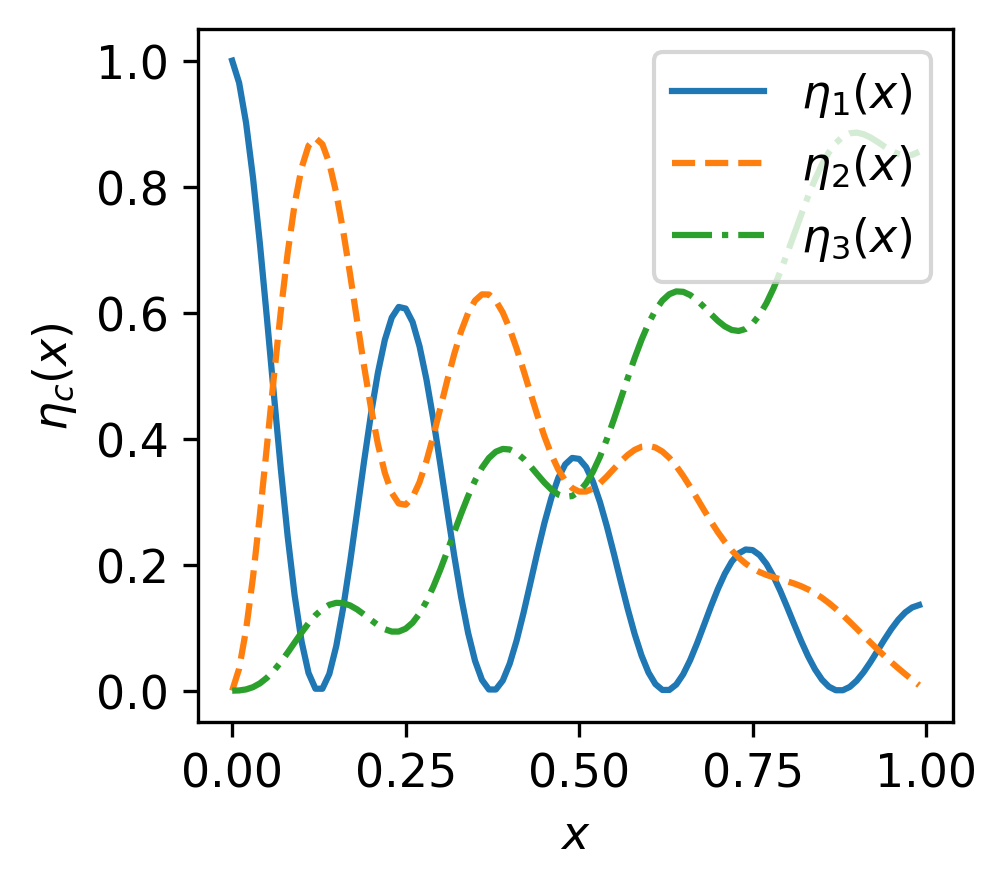}
\endminipage 
\minipage{0.5\textwidth}
%\vspace{-60pt}
\includegraphics[width=3in]{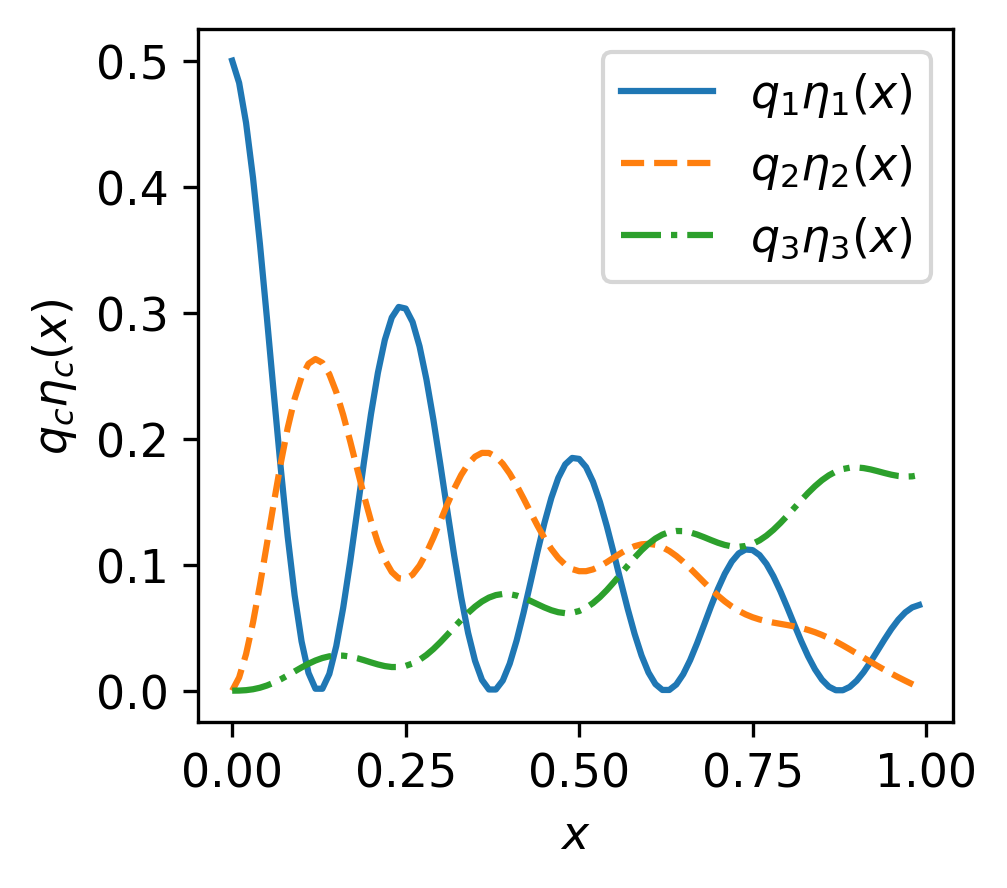}
\endminipage\hfill
%\vspace{-30pt}
\caption{The values of the unweighted and weighted regression function on the unit interval. 
         Note that the weighting increases the areas in which classing the data as class \(1\) is optimal.}
\label{figure:Example:RegressionWeighting}
\end{figure}

Finally, we examine the error region that defines \(\tne\) and \(\fne\) for this regression function and weighting \(q = (0.5, 0.3, 0.2)\).
To determine the error, we have to define a set of permissible weightings \(Q\), a cover \(Q_{\cover}\), and an error \(\epsilon\).
First, we select \(Q = \{q: \sum_{c = 1}^{3} q_{c} = 1, q_c \geq 0.1 \text{ for all } c\}\).
Second, we define \(Q_{\cover} = Q \cap \{0.01 \times (a,  b, c): a, b, c \in \naturals\}\).
Note that since \(q\) is in \(Q_{\cover}\) we could pick \(q = q' = q''\), but for illustration purposes or under a slight perturbation of \(q\) of order \(0.001\), we select \(q' = (0.52, 0.29, 0.19)\) and \(q'' = (0.48, 0.31, 0.21)\) to cover \(q\).
Finally, we set \(\epsilon = 0.1.\)

With these parameters, we estimate the true negative error and false negative error.
To estimate these, we compute a fine grid \(x_{1}, \ldots, x_{N}\) of the unit interval.
Then, we set
\begin{align*}
    \emptne(q', q'', 1) 
    &= \frac{1}{N} \sum_{i = 1}^{N} (1 - \eta_{1}(x_{i})) \ind\left\{t'(1, x_{i}) - \epsilon r'(1)
                                                                     \leq \eta_{1}(x_{i})
                                                                     \leq t''(1, x_{i}) + \epsilon r''(1) \right\} \\
    \empfne(q', q'', 1) 
    &= \frac{1}{N} \sum_{i = 1}^{N} \eta_{1}(x_{i}) \ind\left\{t'(1, x_{i}) - \epsilon r'(1)
                                                               \leq \eta_{1}(x_{i})
                                                               \leq t''(1, x_{i}) + \epsilon r''(1) \right\}
\end{align*}
With a grid of size \(N = 1,000\), we obtain \(\emptne(q', q'', 1) = 0.20\) and \(\empfne(q', q'', 1) = 0.06\).
While the error is large for these parameters, it is consequently easy to visualize in Figure~\ref{figure:Example:ErrorRegion}.
\begin{figure}[t]
\centering 
%\minipage{0.5\textwidth}
%\includegraphics[width=\linewidth]{figs/gamma_2d_03t.png}
%\endminipage 
%\minipage{0.5\textwidth}
%\vspace{-60pt}
\includegraphics[width=3in]{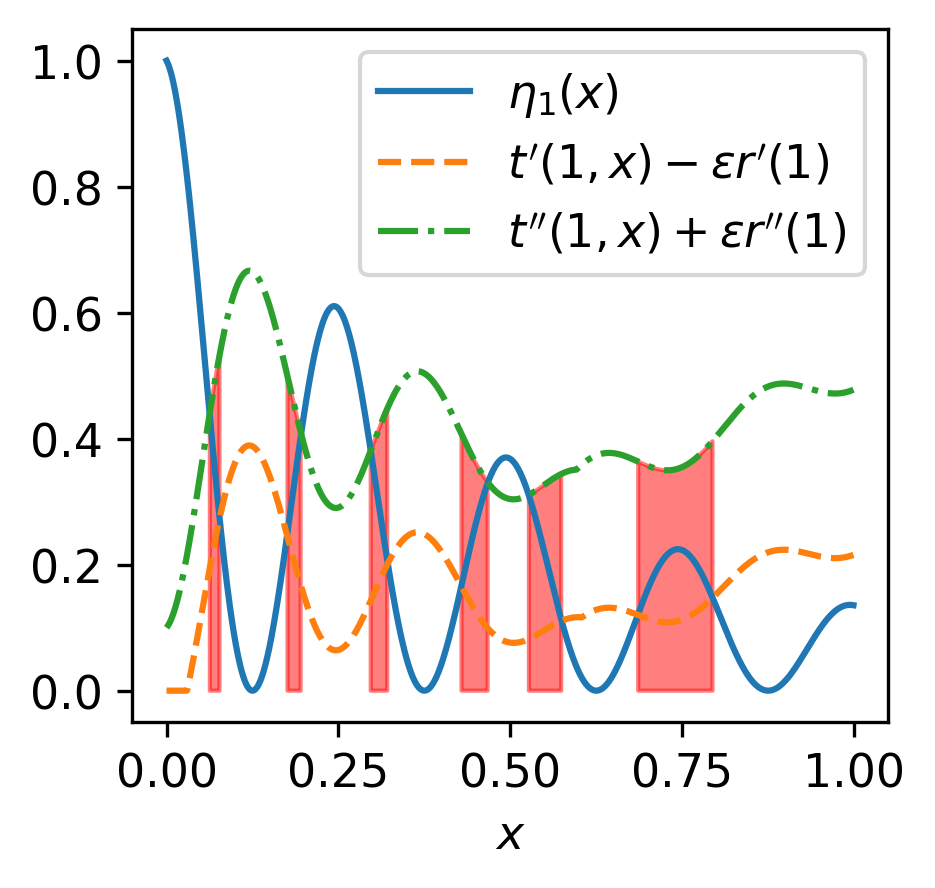}
%\endminipage\hfill
%\vspace{-30pt}
\caption{The error regions for the regression function when \(\epsilon = 0.1\).}
\label{figure:Example:ErrorRegion}
\end{figure}
Specifically, we plot the curves \(\eta_{1}(x)\), \(t'(1, x_{i}) - \epsilon r'(1)\), and \(t''(1, x_{i}) + \epsilon r''(1)\).
Additionally, we shade the region down to the \(x\)-axis where \(\eta_{1}(x)\) is between the other two curves, and the measure of the shaded region on the \(x\)-axis is equal to \(\emptne(q', q'', 1) + \empfne(q', q'', 1)\).
Informally, when \(\eta_{1}(x)\) is not between the other two curves, then there is enough signal to determine that class 1 should or should not be chosen with high probability; when \(\eta_{1}(x)\) is between the other two curves, then making an error is more likely.

Our final task in analyzing this example is to draw samples, compute a kNN estimator, and then determine the true errors in estimating the confusion matrix.
For this, we run \(1,000\) trials in which we pick \(n\) sample points and use weighted \(k\)-nearest neighbors with \(q = (0.5, 0.3, 0.2)\).
For each trial, we calculate the error of the confusion matrix entries, and then we average this over the trials.
The results can be seen in Table~\ref{table:EmpiricalConfusionMatrixErrors}.
For a large enough sample size, we see that the errors do indeed diminish, e.g., for \(n = 1,000\) and \(k = 49\), all confusion matrix errors are under \(0.05\).
\begin{table}[t]
    \caption{Average confusion matrix errors over \(1,000\) trials of \(n\) samples with \(k\)-nearest neighbors weighted by \(q = (0.5, 0.3, 0.2)\). For this particular distribution and weighting, true negatives are hardest to estimate.}
    \centering
    \begin{tabular}{llcccc}
    \(n\)     & \(k\) & \(|\hat{\truenegative}_{1} - \truenegative_{1}|\) & \(|\hat{\truepositive}_{1} - \truepositive_{1}|\) 
     & \(|\hat{\falsenegative}_{1} - \falsenegative_{1}|\) & \(|\hat{\falsepositive}_{1} - \falsepositive_{1}|\)\\ \bottomrule
    50 & 18 & \(0.16\) & \(0.08\) & \(0.05\) & \(0.12\) \\
    100 & 23 & \(0.11\) & \(0.06\) & \(0.04\) & \(0.08\) \\
    1,000 & 49 & \(0.03\) & \(0.01\) & \(0.01\) & \(0.02\)
    \end{tabular}
    \label{table:EmpiricalConfusionMatrixErrors}
\end{table}
Note that these numbers are not directly comparable to the theoretical error upper bound obtained earlier because we did not have an exact relationship between \(\epsilon\) and \(n\) due to unknown constants.
However, the numerical results do suggest that the confusion matrix entries may be accurately estimated when the sample size is sufficiently large.
This ultimately serves as justification for using empirical estimates of the target metric to adjust weights \(q\) when attempting to optimize a general classification metric such as \(\fone\) score. We discuss this further in the next section.

%---------------------------------------------%
%---------------------------------------------%
%---------------------------------------------%
%---------------------------------------------%
\section{Optimization}
\label{sec:NumericalResults}

In this section we introduce two basic algorithms for optimizing a general classification metric and perform experiments on real and synthetic data to examine the error with respect to the macro-averaged F1-score defined in equation~\eqref{equation:macroF1}.
In contrast to the example given in the previous section where \(q\) was fixed, here we choose \(q\) adaptively. Python code for reproducing the results in this section is available at \url{https://github.com/neilzxu/weighted_knn_classification}.

%---------------------------------------------%
%---------------------------------------------%
\subsection{Algorithms}

We start by introducing our algorithms.
Our first algorithm is a coordinate-wise greedy algorithm for choosing \(q\).
At each step, we construct candidates by increasing or decreasing a single coordinate of \(q\) and normalizing to obtain a new \(q'\). 
We do this for all coordinates, yielding \(2C\) candidates, and then we select new weighting \(q''\) that has the highest \(\hat{\fone}\)-score.
Our algorithm is given in Algorithm~\ref{algorithm:CoordinateGreedy}, and we use \(\vect{e}_{i}\) to denote the \(i\)th standard basis vector.

%---------------------------------------------%
%---------------------------------------------%
\begin{algorithm}[t]
	\caption{Greedy Coordinate Search}
	\SetKwInOut{Input}{Input}
	\SetKwInOut{Output}{Output}
	\Input{step size \(\gamma\), number of steps \(T\), initial weights \(q^{(0)}\), empirical F1 function \(\hat{\fone}\).}
	\For{\(t\) in \(1, \ldots, T\)}{
		\For{\(i\) in \(1, \ldots, C\)}{
			\For{sign in \(\{+1, -1\}\)}{
				\(q_{\text{step}} = \left(q^{(t - 1)} + \text{sign}* \gamma * \vect{e}_i\right)_{+}\)\\
				\(q_{\text{candidate}} = \frac{q_{\text{step}}}{\left\|q_{\text{step}}\right\|_1} \)\\
				\If{\(\hat{\fone}(q_{\text{candidate}}) \geq \hat{\fone}\left(q^{(t)}\right)\)\textsl{}}{
					\(q^{(t)} \gets q_{\text{candidate}}\)\\	
				}
			}
		}
	}
	\Output{The weights \(q^{(T)}\).}
\label{algorithm:CoordinateGreedy}
\end{algorithm}
%---------------------------------------------%
%---------------------------------------------%

In addition to the greedy algorithm, we also perform a grid search over the weights, and we defer details of this algorithm to the Appendix.
The benefit of grid search is that it is attains better performance than greedy algorithms in the absence of additional structure; the drawback is that in general computation scales exponentially in the number of classes \(C\).

%---------------------------------------------%
%---------------------------------------------%
\subsection{Data}

We use two types of data for our experiments: synthetic and real.
The synthetic data comes from the simple distribution described in Section~\ref{sec:Example}.
For the synthetic experiments, we use an initial weighting \(q^{(0)} = (0.3, 0.3, 0.4)\), which is close to the unweighted classifier, and use \(k = \min\left(n, 5n^{1/3}\right)\) nearest neighbors, which leads to asymptotically optimal convergence of the uniform error, on a sample of size \(n\).

For real data, we use the Covertype dataset \citep{blackard1999comparative} from the UCI dataset repository \citep{blake1998uci}. 
This dataset contains 54 cartographic features and 7 classes corresponding to different types of forest cover for a patch of land.
The dataset is split into train, dev, and test sets with 11,340, 3,780, and 565,892 samples respectively.

%---------------------------------------------%
%---------------------------------------------%
\subsection{Results}

In this section, we examine the performance of our algorithms on the two datasets. 

%---------------------------------------------%
%---------------------------------------------%
\subsubsection{Synthetic dataset}

\begin{figure}[t]
	\begin{subfigure}{0.48\textwidth}
		\includegraphics[width=\columnwidth]{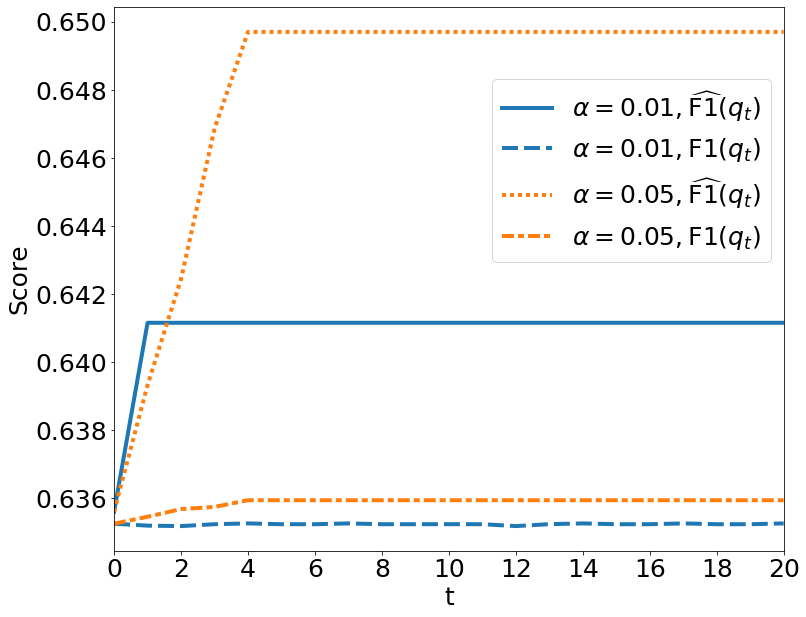}
		\vspace{-15pt}
		\caption{Empirical and population F1 by step.}
		\label{subfig:StepSizeF1}
	\end{subfigure}
	\begin{subfigure}{0.48\textwidth}
		\includegraphics[width=\columnwidth]{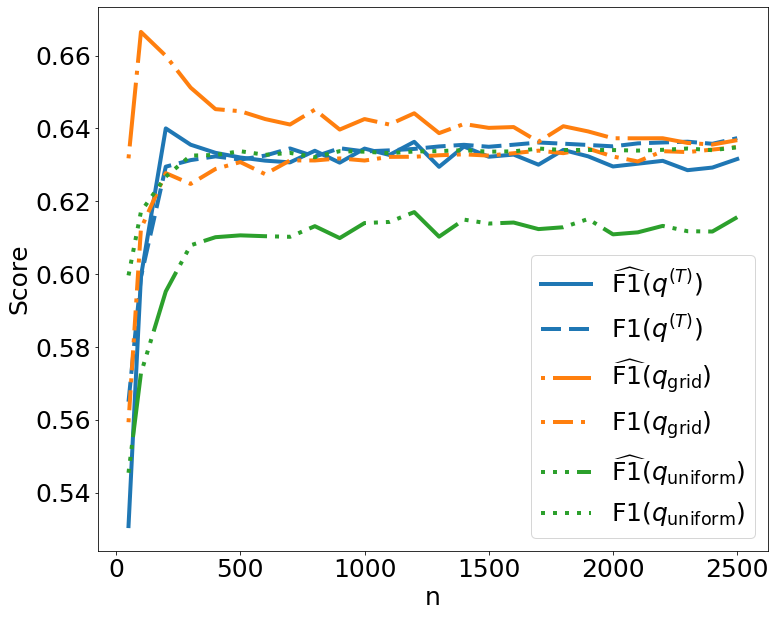}
		\vspace{-15pt}
		\caption{Empirical and population F1 by sample size.}
		\label{subfig:GreedyGridComp}
	\end{subfigure}	
	\caption{Empirical and population F1-score by step in the greedy algorithm and by sample size in the greedy algorithm, grid search, and unweighted classifier. For (a), the sample size is \(n = 1,000\). The quality of the solution depends very modestly on the step size. For (b), the greedy algorithm is run with \(T = 20\) and \(\gamma = 0.01\), and grid search is conducted with differences of \(0.01\). The greedy algorithm achieves similar performance to that of grid search, and the difference between empirical and population F1 decreases with \(n\). It also has marginally better population F1 than the baseline unweighted classifier for sample sizes greater than \(n=1,000\).}
	\label{figure:Results:Synthetic}
\end{figure}

We start by considering the synthetic data.
We perform two different experiments on the synthetic dataset. First, we examine the effect of different step sizes on the weights found by the greedy algorithm. Second, we compare the performance of each algorithm as the training set size increases. In addition to examining \(\hat{\fone}\), we can also compute the population \(\fone\) score to determine how well the resulting classifier generalizes. For our experiments we use \(N = 10,000\) to form a grid for numerical computation of the confusion matrix, and calculate the F1 score using the confusion matrix. 

To examine the difference in \(\hat{\fone}\) and \(\fone\) achieved by the greedy algorithm for different step sizes, we use a step count \(T = 20\), a dataset of size \(n = 1,000\), and step sizes of \(0.01\) and \(0.05\). Figure~\ref{subfig:StepSizeF1} shows that a greedy algorithm with the step size \(0.05\) finds a weighting that is better in both empirical and population F1. However, the algorithm takes few steps, since it gets stuck quickly with both step sizes.  

Next, we examine the relationship between \(\hat{\fone}\) and \(\fone\) of the greedy algorithm, grid algorithm, and unweighted classifier as \(n\) increases. For the greedy algorithm we choose parameters \(T = 20\) and \(\gamma = 0.01\). To calculate \(\hat{\fone}\) and \(\fone\), we average each score of each algorithm over 50 trials for each \(n\).  For grid search, we choose a grid over the weight space with a spacing of 0.01 between points. Figure~\ref{subfig:GreedyGridComp} shows that \(\hat{\fone}\) seems to converge towards \(\fone\) as \(n\) increases for all 3 classifiers, although the greedy and grid algorithms converge much faster than the unweighted classifier. The convergence is consistent with the convergence of individual entries of the empirical confusion matrix to their corresponding entries in the true confusion matrix. For all 3 classifiers, the population F1 remains relatively similar across all \(n\), with the greedy algorithm performing marginally better than the unweighted classifier, and the grid algorithm performing marginally worse. The similarity in performance between the unweighted, i.e., usual kNN, and the approaches that optimize the weights may be in part due to the relatively balanced nature of the problem, since the class probabilities are \(0.22\), \(0.35\), and \(0.43\) respectively, leading to the optimal choice of weights being relatively close to having equal weights across classes.

The greedy algorithm also has consistently higher population F1 than grid search. This may be a consequence of the initial weighting being close to final weighting for the greedy algorithm, allowing it to do more slightly more granular search than the grid algorithm. In any case, these experiments demonstrate that greedy search over the weights can perform as well as grid search empirically.
%---------------------------------------------%
%---------------------------------------------%
\subsubsection{Real dataset}

For the real dataset, we fit the underlying kNN model to the training set, and then fit weights on the dev set. We finally evaluate the \(\fone\) performance of our optimized weights on the test set. Here, we consider \(\hat{\fone}\) as the \(\fone\) score of classifier on the training set and population \(\fone\) as the \(\fone\) score on the test set. For the kNN classifier, we chose \(k=160\). We set the initial weighting for our greedy algorithm to be the balanced weighting, and search for \(T = 25\) steps with a step size of \(\alpha = 0.02\). For the grid algorithm, we use a grid with a spacing of 0.083 between points. In addition to considering these algorithms, we also consider the unweighted baseline kNN classifier, as well as a simple logistic regression classifier trained with stochastic gradient descent.

\renewcommand{\arraystretch}{1.5}
\begin{table}[t]
    \caption{Training and testing \(\fone\) scores on the Covertype dataset for our greedy classifier, grid search classifier, unweighted kNN, and a linear classifier (logistic regression). The linear classifier performs best, reflecting the separability of the dataset. Greedy search outperforms both grid search and baseline unweighted classifier on test \(\fone\). The learned weights for the greedy and grid search have a similar distribution, yet the \(\hat{\fone}\) score of grid search is much higher, suggesting overfitting.}
    \centering
    \begin{tabular}{l||ll|lllllll}
    Algorithm     & \(\hat{\fone}\) & \(\fone\)  
    & \(q_{1}\) & \(q_{2}\) & \(q_{3}\) & \(q_{4}\) & \(q_{5}\) & \(q_{6}\) & \(q_{7}\) \\ \bottomrule
    Greedy     & 0.134 & 0.144 & 0.045 & 0.195 & 0.151 & 0.317 & 0.069 & 0.152 & 0.070  \\
    Grid       & 0.170 & 0.072 & 0.0 & 0.083 & 0.167 & 0.583 & 0.0 & 0.167 & 0.0   \\
    Unweighted & 0.054  & 0.072 & 0.143 & 0.143 & 0.143 & 0.143 & 0.143 & 0.143 & 0.143  \\
    Linear  & 0.467 &  0.242 & N/A   & N/A   & N/A   & N/A   & N/A   & N/A   & N/A
    \end{tabular}
    \label{table:RealResults}
\end{table}

We present our results, consisting of the train \(\fone\), test \(\fone\), and final class weights, in Table~\ref{table:RealResults}.
The linear classifier vastly outperforms the kNN models, showing the relatively linearly separable nature of this dataset. Reflecting their performances on the synthetic experiment, the greedy algorithm has the highest test \(\fone\) score when compared to the grid algorithm and the unweighted classifier, both of which have the same F1 score. The higher \(\hat{\fone}\) of the grid algorithm compared to the greedy algorithm suggests it overfit on the training dataset. We do note that the learned weights of the two algorithms are similar, however. In particular, both the greedy and grid algorithm place much of their weight on class 4. 

%---------------------------------------------%
%---------------------------------------------%
%---------------------------------------------%
%---------------------------------------------%
\section{Discussion}
\label{sec:Discussion}

While we make progress in the theory of modern classification problems using nearest neighbors, there are still a number of future directions, both statistical and computational.
First, there are still many questions on how to optimize a given metric such as F1 score.
For instance, it is unclear that our algorithms find the optimal weighting, or whether such weighted plug-in approaches are even optimal as in the binary classification case.
Second, we would also like a method that finds an optimal weighting as the number of classes grows large.
In this regime, we cannot effectively use grid search because it requires exponential computation in the number of classes.
Finally, the ultimate goal is to produce a computationally efficient classifier with minimax optimal statistical risk with respect to a given metric, again such as F1 score.
To the best of our knowledge, such a classifier and the rate of convergence of the risk is not known for any non-trivial problem, parametric or nonparametric.

\section*{Acknowledgments}

JK was partially supported by Accenture, Rakuten, and Lockheed Martin.
We would like to thank L\'aszl\'o Gy\"orfi and Aryeh Kontorovich for their helpful comments and suggestions.

%---------------------------------------------%
%---------------------------------------------%
%\bibliographystyle{abbrvnat}
\bibliography{extreme_refs}

\begin{thebibliography}{70}
\providecommand{\natexlab}[1]{#1}
\providecommand{\url}[1]{\texttt{#1}}
\expandafter\ifx\csname urlstyle\endcsname\relax
  \providecommand{\doi}[1]{doi: #1}\else
  \providecommand{\doi}{doi: \begingroup \urlstyle{rm}\Url}\fi

\bibitem[Andoni and Indyk(2006)]{andoni2006near}
A.~Andoni and P.~Indyk.
\newblock Near-optimal hashing algorithms for approximate nearest neighbor in
  high dimensions.
\newblock In \emph{2006 47th Annual IEEE Symposium on Foundations of Computer
  Science (FOCS'06)}, pages 459--468. IEEE, 2006.

\bibitem[Andoni et~al.(2018)Andoni, Indyk, and
  Razenshteyn]{andoni2018approximate}
A.~Andoni, P.~Indyk, and I.~Razenshteyn.
\newblock Approximate nearest neighbor search in high dimensions.
\newblock \emph{arXiv preprint arXiv:1806.09823}, 2018.

\bibitem[Audibert and Tsybakov(2007)]{audibert2007fast}
J.-Y. Audibert and A.~B. Tsybakov.
\newblock Fast learning rates for plug-in classifiers.
\newblock \emph{The Annals of Statistics}, 35\penalty0 (2):\penalty0 608--633,
  2007.

\bibitem[Biau and Devroye(2015)]{biau2015lectures}
G.~Biau and L.~Devroye.
\newblock \emph{Lectures on the Nearest Neighbor Method}.
\newblock Springer, 2015.

\bibitem[Biau et~al.(2010)Biau, C{\'e}rou, and Guyader]{biau2010rates}
G.~Biau, F.~C{\'e}rou, and A.~Guyader.
\newblock Rates of convergence of the functional $ k $-nearest neighbor
  estimate.
\newblock \emph{IEEE Transactions on Information Theory}, 56\penalty0
  (4):\penalty0 2034--2040, 2010.

\bibitem[Blackard and Dean(1999)]{blackard1999comparative}
J.~A. Blackard and D.~J. Dean.
\newblock Comparative accuracies of artificial neural networks and discriminant
  analysis in predicting forest cover types from cartographic variables.
\newblock \emph{Computers and electronics in agriculture}, 24\penalty0
  (3):\penalty0 131--151, 1999.

\bibitem[Blake and Merz(1998)]{blake1998uci}
C.~Blake and C.~Merz.
\newblock Uci repository of machine learning datasets.
\newblock \emph{University of California, Irvine, Dept. of Information and
  Computer Sciences}, 1998.

\bibitem[Cannings et~al.(2019)Cannings, Berrett, and
  Samworth]{cannings2019local}
T.~I. Cannings, T.~B. Berrett, and R.~J. Samworth.
\newblock Local nearest neighbour classification with applications to
  semi-supervised learning.
\newblock \emph{arXiv preprint arXiv:1704.00642 v3}, 2019.

\bibitem[Cano et~al.(2013)Cano, Zafra, and Ventura]{cano2013weighted}
A.~Cano, A.~Zafra, and S.~Ventura.
\newblock Weighted data gravitation classification for standard and imbalanced
  data.
\newblock \emph{IEEE Transactions on Cybernetics}, 43\penalty0 (6):\penalty0
  1672--1687, 2013.

\bibitem[Cao et~al.(2019)Cao, Wei, Gaidon, Arechiga, and Ma]{cao2019learning}
K.~Cao, C.~Wei, A.~Gaidon, N.~Arechiga, and T.~Ma.
\newblock Learning imbalanced datasets with label-distribution-aware margin
  loss.
\newblock \emph{arXiv preprint arXiv:1906.07413}, 2019.

\bibitem[Chaudhuri and Dasgupta(2014)]{chaudhuri2014rates}
K.~Chaudhuri and S.~Dasgupta.
\newblock Rates of convergence for nearest neighbor classification.
\newblock In \emph{Advances in Neural Information Processing Systems}, pages
  3437--3445, 2014.

\bibitem[Chawla et~al.(2002)Chawla, Bowyer, Hall, and
  Kegelmeyer]{chawla2002smote}
N.~V. Chawla, K.~W. Bowyer, L.~O. Hall, and W.~P. Kegelmeyer.
\newblock {SMOTE}: synthetic minority over-sampling technique.
\newblock \emph{Journal of Artificial Intelligence Research}, 16:\penalty0
  321--357, 2002.

\bibitem[Chen et~al.(2019)Chen, Dobriban, and Lee]{chen2019invariance}
S.~Chen, E.~Dobriban, and J.~H. Lee.
\newblock Invariance reduces variance: Understanding data augmentation in deep
  learning and beyond.
\newblock \emph{arXiv preprint arXiv:1907.10905}, 2019.

\bibitem[Cover and Hart(1967)]{cover1967nearest}
T.~Cover and P.~Hart.
\newblock Nearest neighbor pattern classification.
\newblock \emph{IEEE Transactions on Information Theory}, 13\penalty0
  (1):\penalty0 21--27, 1967.

\bibitem[Dembczynski et~al.(2013)Dembczynski, Jachnik, Kotlowski, Waegeman, and
  Huellermeier]{dembczynski2013optimizing}
K.~Dembczynski, A.~Jachnik, W.~Kotlowski, W.~Waegeman, and E.~Huellermeier.
\newblock Optimizing the f-measure in multi-label classification: Plug-in rule
  approach versus structured loss minimization.
\newblock In \emph{Proceedings of the 30th International Conference on Machine
  Learning}. PMLR, 2013.

\bibitem[Devroye et~al.(1996)Devroye, Gy{\"o}rfi, and
  Lugosi]{devroye1996probabilistic}
L.~Devroye, L.~Gy{\"o}rfi, and G.~Lugosi.
\newblock \emph{A probabilistic theory of pattern recognition}.
\newblock Springer Science \& Business Media, 1996.

\bibitem[Domingos(1999)]{domingos1999metacost}
P.~Domingos.
\newblock Metacost: A general method for making classifiers cost-sensitive.
\newblock In \emph{KDD}, volume~99, pages 155--164, 1999.

\bibitem[Dong et~al.(2019)Dong, Indyk, Razenshteyn, and
  Wagner]{dong2019scalable}
Y.~Dong, P.~Indyk, I.~Razenshteyn, and T.~Wagner.
\newblock Scalable nearest neighbor search for optimal transport.
\newblock \emph{arXiv preprint arXiv:1910.04126}, 2019.

\bibitem[D{\"o}ring et~al.(2018)D{\"o}ring, Gy{\"o}rfi, and
  Walk]{doring2018rate}
M.~D{\"o}ring, L.~Gy{\"o}rfi, and H.~Walk.
\newblock Rate of convergence of k-nearest-neighbor classification rule.
\newblock \emph{The Journal of Machine Learning Research}, 18\penalty0
  (227):\penalty0 1--16, 2018.

\bibitem[Efremenko et~al.(2020)Efremenko, Kontorovich, and
  Noivirt]{efremenko2020fast}
K.~Efremenko, A.~Kontorovich, and M.~Noivirt.
\newblock Fast and bayes-consistent nearest neighbors.
\newblock In \emph{Proceedings of the 23rd International Concerence on
  Artificial Intelligence and Statistics}. PMLR, 2020.

\bibitem[Fathony and Kolter(2019)]{fathony2019genericMetrics}
R.~Fathony and J.~Z. Kolter.
\newblock {AP}-perf: Incorporating generic performance metrics in
  differentiable learning.
\newblock \emph{arXiv preprint arXiv:1912.00965}, 2019.

\bibitem[Feldman(2019)]{feldman2019does}
V.~Feldman.
\newblock Does learning require memorization? a short tale about a long tail.
\newblock \emph{arXiv preprint arXiv:1906.05271}, 2019.

\bibitem[Fern{\'a}ndez et~al.(2018)Fern{\'a}ndez, Garc{\'\i}a, Galar, Prati,
  Krawczyk, and Herrera]{fernandez2018learning}
A.~Fern{\'a}ndez, S.~Garc{\'\i}a, M.~Galar, R.~C. Prati, B.~Krawczyk, and
  F.~Herrera.
\newblock \emph{Learning from imbalanced data sets}.
\newblock Springer, 2018.

\bibitem[Fisher(1936)]{fisher1936use}
R.~A. Fisher.
\newblock The use of multiple measurements in taxonomic problems.
\newblock \emph{Annals of Eugenics}, 7\penalty0 (2):\penalty0 179--188, 1936.

\bibitem[Fisher(1938)]{fisher1938statistical}
R.~A. Fisher.
\newblock The statistical utilization of multiple measurements.
\newblock \emph{Annals of Eugenics}, 8\penalty0 (4):\penalty0 376--386, 1938.

\bibitem[Fix and Hodges~Jr(1951)]{fix1951discriminatory}
E.~Fix and J.~L. Hodges~Jr.
\newblock Discriminatory analysis-nonparametric discrimination: consistency
  properties.
\newblock Technical report, USAF School of Aviation Medicine, Randolph Field,
  Texas, 1951.

\bibitem[Gadat et~al.(2016)Gadat, Klein, and Marteau]{gadat2016classification}
S.~Gadat, T.~Klein, and C.~Marteau.
\newblock Classification in general finite dimensional spaces with the
  k-nearest neighbor rule.
\newblock \emph{The Annals of Statistics}, 44\penalty0 (3):\penalty0 982--1009,
  2016.

\bibitem[{Gottlieb} et~al.(2018){Gottlieb}, {Kontorovich}, and
  {Nisnevitch}]{gottlieb2018near}
L.~{Gottlieb}, A.~{Kontorovich}, and P.~{Nisnevitch}.
\newblock Near-optimal sample compression for nearest neighbors.
\newblock \emph{IEEE Transactions on Information Theory}, 64\penalty0
  (6):\penalty0 4120--4128, 2018.

\bibitem[Gy{\"o}rfi et~al.(2002)Gy{\"o}rfi, Kohler, Krzyzak, and
  Walk]{gyorfi2002distribution}
L.~Gy{\"o}rfi, M.~Kohler, A.~Krzyzak, and H.~Walk.
\newblock \emph{A distribution-free theory of nonparametric regression}.
\newblock Springer Science \& Business Media, 2002.

\bibitem[Hart(1968)]{hart1968condensed}
P.~Hart.
\newblock The condensed nearest neighbor rule.
\newblock \emph{IEEE Transactions on Information Theory}, 14\penalty0
  (3):\penalty0 515--516, 1968.

\bibitem[Indyk and Wagner(2018)]{indyk2018approximate}
P.~Indyk and T.~Wagner.
\newblock Approximate nearest neighbors in limited space.
\newblock \emph{arXiv preprint arXiv:1807.00112}, 2018.

\bibitem[Joachims(2005)]{joachims2005support}
T.~Joachims.
\newblock A support vector method for multivariate performance measures.
\newblock In \emph{Proceedings of the 22nd International Conference on Machine
  Learning}, pages 377--384. ACM, 2005.

\bibitem[Joulin et~al.(2016)Joulin, Grave, Bojanowski, and
  Mikolov]{joulin2016bag}
A.~Joulin, E.~Grave, P.~Bojanowski, and T.~Mikolov.
\newblock Bag of tricks for efficient text classification.
\newblock \emph{arXiv preprint arXiv:1607.01759}, 2016.

\bibitem[Khan et~al.(2001)Khan, Wei, Ringner, Saal, Ladanyi, Westermann,
  Berthold, Schwab, Antonescu, Peterson, et~al.]{khan2001classification}
J.~Khan, J.~S. Wei, M.~Ringner, L.~H. Saal, M.~Ladanyi, F.~Westermann,
  F.~Berthold, M.~Schwab, C.~R. Antonescu, C.~Peterson, et~al.
\newblock Classification and diagnostic prediction of cancers using gene
  expression profiling and artificial neural networks.
\newblock \emph{Nature medicine}, 7\penalty0 (6):\penalty0 673, 2001.

\bibitem[Kontorovich et~al.(2017)Kontorovich, Sabato, and
  Weiss]{kontorovich2017nearest}
A.~Kontorovich, S.~Sabato, and R.~Weiss.
\newblock Nearest-neighbor sample compression: Efficiency, consistency,
  infinite dimensions.
\newblock In \emph{Advances in Neural Information Processing Systems}, pages
  1573--1583, 2017.

\bibitem[Koyejo et~al.(2014)Koyejo, Natarajan, Ravikumar, and
  Dhillon]{koyejo2014consistent}
O.~O. Koyejo, N.~Natarajan, P.~K. Ravikumar, and I.~S. Dhillon.
\newblock Consistent {{Binary Classification}} with {{Generalized Performance
  Metrics}}.
\newblock In \emph{Advances in {{Neural Information Processing Systems}} 27},
  pages 2744--2752. {Curran Associates, Inc.}, 2014.

\bibitem[Lee et~al.(2004)Lee, Wahba, and Ackerman]{lee2004cloud}
Y.~Lee, G.~Wahba, and S.~A. Ackerman.
\newblock Cloud classification of satellite radiance data by multicategory
  support vector machines.
\newblock \emph{Journal of Atmospheric and Oceanic Technology}, 21\penalty0
  (2):\penalty0 159--169, 2004.

\bibitem[Lewis(1995)]{lewis1995evaluating}
D.~D. Lewis.
\newblock Evaluating and optimizing autonomous text classification systems.
\newblock In \emph{SIGIR}, volume~95, pages 246--254. Citeseer, 1995.

\bibitem[Lin et~al.(2002)Lin, Lee, and Wahba]{lin2002support}
Y.~Lin, Y.~Lee, and G.~Wahba.
\newblock Support vector machines for classification in nonstandard situations.
\newblock \emph{Machine learning}, 46\penalty0 (1-3):\penalty0 191--202, 2002.

\bibitem[Lin et~al.(2018)Lin, Das, and Datta]{lin2018overview}
Y.-C. Lin, P.~Das, and A.~Datta.
\newblock {Overview of the SIGIR 2018 eCom Rakuten Data Challenge}.
\newblock In \emph{eCOM@ SIGIR}, 2018.

\bibitem[Liu and Chawla(2011)]{liu2011class}
W.~Liu and S.~Chawla.
\newblock Class confidence weighted knn algorithms for imbalanced data sets.
\newblock In \emph{Pacific-Asia Conference on Knowledge Discovery and Data
  Mining}, pages 345--356. Springer, 2011.

\bibitem[Liu(2007)]{liu2007fisher}
Y.~Liu.
\newblock Fisher consistency of multicategory support vector machines.
\newblock In \emph{Artificial Intelligence and Statistics}, pages 291--298,
  2007.

\bibitem[L{\'o}pez et~al.(2014)L{\'o}pez, Triguero, Carmona, Garc{\'\i}a, and
  Herrera]{lopez2014addressing}
V.~L{\'o}pez, I.~Triguero, C.~J. Carmona, S.~Garc{\'\i}a, and F.~Herrera.
\newblock Addressing imbalanced classification with instance generation
  techniques: {IPADE-ID}.
\newblock \emph{Neurocomputing}, 126:\penalty0 15--28, 2014.

\bibitem[Mammen and Tsybakov(1999)]{mammen1999smooth}
E.~Mammen and A.~B. Tsybakov.
\newblock Smooth discrimination analysis.
\newblock \emph{The Annals of Statistics}, 27\penalty0 (6):\penalty0
  1808--1829, 1999.

\bibitem[Mariani et~al.(2018)Mariani, Scheidegger, Istrate, Bekas, and
  Malossi]{mariani2018bagan}
G.~Mariani, F.~Scheidegger, R.~Istrate, C.~Bekas, and C.~Malossi.
\newblock Bagan: Data augmentation with balancing {GAN}.
\newblock \emph{arXiv preprint arXiv:1803.09655}, 2018.

\bibitem[Menon et~al.(2013)Menon, Narasimhan, Agarwal, and
  Chawla]{menon2013statistical}
A.~Menon, H.~Narasimhan, S.~Agarwal, and S.~Chawla.
\newblock On the statistical consistency of algorithms for binary
  classification under class imbalance.
\newblock In \emph{International Conference on Machine Learning}, pages
  603--611, 2013.

\bibitem[Mohri et~al.(2012)Mohri, Rostamizadeh, and Talwalkar]{mohri2012}
M.~Mohri, A.~Rostamizadeh, and A.~Talwalkar.
\newblock \emph{Foundations of Machine Learning}.
\newblock MIT Press, 2012.

\bibitem[Narasimhan et~al.(2014)Narasimhan, Vaish, and
  Agarwal]{narasimhan2014statistical}
H.~Narasimhan, R.~Vaish, and S.~Agarwal.
\newblock On the statistical consistency of plug-in classifiers for
  non-decomposable performance measures.
\newblock In \emph{Advances in Neural Information Processing Systems}, pages
  1493--1501, 2014.

\bibitem[Puchkin and Spokoiny(2020)]{puchkin2020adaptive}
N.~Puchkin and V.~Spokoiny.
\newblock An adaptive multiclass nearest neighbor classifier.
\newblock \emph{ESAIM: Probability and Statistics}, 24:\penalty0 69--99, 2020.

\bibitem[Rigollet and Tong(2011)]{rigollet2011neyman}
P.~Rigollet and X.~Tong.
\newblock {N}eyman-{P}earson classification, convexity and stochastic
  constraints.
\newblock \emph{Journal of Machine Learning Research}, 12\penalty0
  (Oct):\penalty0 2831--2855, 2011.

\bibitem[Salakhutdinov et~al.(2011)Salakhutdinov, Torralba, and
  Tenenbaum]{salakhutdinov2011learning}
R.~Salakhutdinov, A.~Torralba, and J.~Tenenbaum.
\newblock Learning to share visual appearance for multiclass object detection.
\newblock In \emph{CVPR 2011}, pages 1481--1488. IEEE, 2011.

\bibitem[Samworth(2012)]{samworth2012optimal}
R.~J. Samworth.
\newblock Optimal weighted nearest neighbour classifiers.
\newblock \emph{Annals of Statistics}, 40\penalty0 (5):\penalty0 2733--2763,
  2012.

\bibitem[Scott(2012)]{scott2012calibrated}
C.~Scott.
\newblock Calibrated asymmetric surrogate losses.
\newblock \emph{Electronic Journal of Statistics}, 6:\penalty0 958--992, 2012.

\bibitem[Slud(1977)]{slud1977distribution}
E.~V. Slud.
\newblock Distribution inequalities for the binomial law.
\newblock \emph{The Annals of Probability}, pages 404--412, 1977.

\bibitem[Stone(1977)]{stone1977consistent}
C.~J. Stone.
\newblock Consistent nonparametric regression.
\newblock \emph{The annals of statistics}, pages 595--620, 1977.

\bibitem[Tewari and Bartlett(2007)]{tewari2007consistency}
A.~Tewari and P.~L. Bartlett.
\newblock On the consistency of multiclass classification methods.
\newblock \emph{Journal of Machine Learning Research}, 8\penalty0
  (May):\penalty0 1007--1025, 2007.

\bibitem[Tong(2013)]{tong2013plug}
X.~Tong.
\newblock A plug-in approach to {N}eyman-{P}earson classification.
\newblock \emph{The Journal of Machine Learning Research}, 14\penalty0
  (1):\penalty0 3011--3040, 2013.

\bibitem[Tong et~al.(2016)Tong, Feng, and Zhao]{tong2016survey}
X.~Tong, Y.~Feng, and A.~Zhao.
\newblock A survey on {N}eyman-{P}earson classification and suggestions for
  future research.
\newblock \emph{Wiley Interdisciplinary Reviews: Computational Statistics},
  8\penalty0 (2):\penalty0 64--81, 2016.

\bibitem[Tsybakov(2009)]{tsybakov2009introduction}
A.~B. Tsybakov.
\newblock \emph{Introduction to Nonparametric Estimation. Revised and extended
  from the 2004 French original. Translated by Vladimir Zaiats}.
\newblock Springer Series in Statistics. Springer, New York, 2009.

\bibitem[Van~Rijsbergen(1974)]{van1974foundation}
C.~J. Van~Rijsbergen.
\newblock Foundation of evaluation.
\newblock \emph{Journal of Documentation}, 30\penalty0 (4):\penalty0 365--373,
  1974.

\bibitem[Van~Rijsbergen(1979)]{van1979information}
C.~J. Van~Rijsbergen.
\newblock \emph{Information Retrieval}.
\newblock Butterworth-Heinemann, London, 2nd edition, 1979.

\bibitem[Vluymans et~al.(2016)Vluymans, Triguero, Cornelis, and
  Saeys]{vluymans2016eprennid}
S.~Vluymans, I.~Triguero, C.~Cornelis, and Y.~Saeys.
\newblock Eprennid: An evolutionary prototype reduction based ensemble for
  nearest neighbor classification of imbalanced data.
\newblock \emph{Neurocomputing}, 216:\penalty0 596--610, 2016.

\bibitem[Wang et~al.(2008)Wang, Shen, and Liu]{wang2008probability}
J.~Wang, X.~Shen, and Y.~Liu.
\newblock Probability estimation for large-margin classifiers.
\newblock \emph{Biometrika}, 95\penalty0 (1):\penalty0 149--167, 2008.

\bibitem[Wang et~al.(2019)Wang, Helen~Zhang, and Wu]{wang2019multiclass}
X.~Wang, H.~Helen~Zhang, and Y.~Wu.
\newblock Multiclass probability estimation with support vector machines.
\newblock \emph{Journal of Computational and Graphical Statistics}, pages
  1--18, 2019.

\bibitem[Wu et~al.(2010)Wu, Zhang, and Liu]{wu2010robust}
Y.~Wu, H.~H. Zhang, and Y.~Liu.
\newblock Robust model-free multiclass probability estimation.
\newblock \emph{Journal of the American Statistical Association}, 105\penalty0
  (489):\penalty0 424--436, 2010.

\bibitem[Yen et~al.(2018)Yen, Kale, Yu, Holtmann-Rice, Kumar, and
  Ravikumar]{yen2018loss}
I.~E.-H. Yen, S.~Kale, F.~Yu, D.~Holtmann-Rice, S.~Kumar, and P.~Ravikumar.
\newblock Loss decomposition for fast learning in large output spaces.
\newblock In \emph{International Conference on Machine Learning}, pages
  5626--5635, 2018.

\bibitem[Zhou and Liu(2006)]{zhou2006training}
Z.-H. Zhou and X.-Y. Liu.
\newblock Training cost-sensitive neural networks with methods addressing the
  class imbalance problem.
\newblock \emph{IEEE Transactions on Knowledge and Data Engineering},
  18\penalty0 (1):\penalty0 63--77, 2006.

\bibitem[Zhu et~al.(2014)Zhu, Anguelov, and Ramanan]{zhu2014capturing}
X.~Zhu, D.~Anguelov, and D.~Ramanan.
\newblock Capturing long-tail distributions of object subcategories.
\newblock In \emph{Proceedings of the IEEE Conference on Computer Vision and
  Pattern Recognition}, pages 915--922, 2014.

\bibitem[Zhu et~al.(2015)Zhu, Wang, and Gao]{zhu2015gravitational}
Y.~Zhu, Z.~Wang, and D.~Gao.
\newblock Gravitational fixed radius nearest neighbor for imbalanced problem.
\newblock \emph{Knowledge-Based Systems}, 90:\penalty0 224--238, 2015.

\bibitem[Zipf(1936)]{zipf1936psycho}
G.~K. Zipf.
\newblock \emph{The Psycho-Biology of Language: an Introduction to Dynamic
  Philology}.
\newblock George Routledge \& Sons, Ltd., 1936.

\end{thebibliography}

\newpage
\appendix

%---------------------------------------------%
%---------------------------------------------%
%---------------------------------------------%
%---------------------------------------------%
\section{Proofs for Accuracy and Risk}

In this section, we consider proofs for accuracy and risk.
Our main results are Theorem~\ref{theorem:Theorem5Analogue}, Corollary~\ref{corollary:SmoothTheorem}, Corollary~\ref{corollary:MarginTheorem}, Proposition~\ref{proposition:WeightedRisk}, and Theorem~\ref{theorem:AccuracyLowerBound}.

%---------------------------------------------%
%---------------------------------------------%
\subsection{Proof of Theorem~\ref{theorem:Theorem5Analogue}}

The goal of this section is to prove Theorem~\ref{theorem:Theorem5Analogue}.
Before the main proof,  we establish a few lemmas.

\begin{lemma}
Let \(x\) be an arbitrary point.
Define the events
\begin{align*}
& \begin{aligned}
\eventA
&= \left\{
(q_{f^{*}_{q}(x)} \hat{\eta}_{f^{*}_{q}(x)}(B') - q_{c} \hat{\eta}_{c}(B'))
-
(q_{f^{*}_{q}(x)} \eta_{f^{*}_{q}(x)}(B') - q_{c} \eta_{c}(B'))
<
-\Delta
%\right. \\&\qquad  \left.
 \text{ for some } c \neq f^{*}_{q}(x)
\right\} \\
\eventB
&=
\left\{
r_{p}(x)
<
\rho\left(x, X_{\sigma_{k + 1}(x)}\right)
\right\}\\
\eventC
&=
\left\{
x \in \partial_{p, \Delta}
\right\}.
\end{aligned}
\end{align*}
Then, we have the inequality
\begin{align*}
& \begin{aligned}
\ind\left\{
f_{q, n, k}(x) \neq f^{*}_{q}(x)
\right\}
&\leq
\ind(\eventA)
% \\&\qquad
+
\ind(\eventB)
+
\ind(\eventC).
\end{aligned}
\end{align*}
\label{lemma:Lemma7Analogue:Weighted}
\end{lemma}

\begin{proof}
Suppose that \(x\) is not in \(\partial_{p, \Delta}\).
Then, if event \(\eventB\) does not occur, there are \(k + 1\) sample points within a radius \(r_{p}(x)\) of \(x\).
Thus, in order for \(f_{q, n, k}(x) \neq f^{*}_{q}(x)\) to occur,
we must have
\[
q_{f^{*}_{q}(x)}\hat{\eta}_{f^{*}_{q}(x)}(x)
<
q_{c} \hat{\eta}_{c}
\]
for some \(c \neq f^{*}_{q}(x)\).
Since \(x\) is not in \(\partial_{p, \Delta}\), this means that
\[
(q_{f^{*}_{q}(x)}\hat{\eta}_{f^{*}_{q}(x)}(x)
-
q_{c} \hat{\eta}_{c})
-
(
q_{f^{*}_{q}(x)} \eta_{f^{*}_{q}(x)}(x)
<
q_{c} \eta_{c}
)
<
-\Delta.
\]
Thus, event \(\eventA\) must occur.
\end{proof}

%---------------------------------------------%
%---------------------------------------------%

Now, we present a union bound lemma.
Let
\[
\eventA_{a, b}
=
\left\{
(q_{a} \hat{\eta}_{a}(B') - q_{b} \hat{\eta}_{b}(B'))
-
(q_{a} \eta_{a}(B') - q_{b} \eta_{b}(B'))
<
-\Delta
\right\}
\]
for \(a \neq b\).
For simplicity, let \(\eventA_{a, a} = \emptyset\).

\begin{lemma}
We have the following inequality:
\begin{align*}
& \begin{aligned}
\ind(\eventA)
&\leq
\sum_{c = 1}^{C}
\ind\left(\eventA_{f^{*}_{q}(x), c}\right).
\end{aligned}
\end{align*}
\label{lemma:UnionBound:Weighted}
\end{lemma}

The proof is immediate, since if the left hand side is \(1\), then one of the summands on the right must also be \(1\).

\begin{lemma}
We have the probability bounds
\begin{align*}
& \begin{aligned}
\prob\left(
\eventA_{a, b}
\right)
&\leq
\exp\left(
-\frac{2k \Delta^{2}}{(q_{a} + q_{b})^{2}}
\right)
\leq
\exp\left(
-\frac{k \Delta^{2}}{2 q_{\max}^{2}}
\right).
\end{aligned}
\end{align*}
\label{lemma:Lemma9Analogue:Weighted}
\end{lemma}

\begin{proof}
The proof consists of applying Hoeffding's inequality to a sampling scheme that makes it clear that the event \(\eventA_{a, b}\) only depends on \(k\) random variables.
The sampling procedure is as follows:
\begin{itemize}
\item[(1)]
first sample \(X_{1}, Y_{1}\) from the marginal distribution of the \((k + 1)\)-nearest neighbor of \(x\);

\item[(2)]
sample \(k\) pairs \((X_{2} Y_{2}), \ldots, (X_{k + 1}, Y_{k + 1})\) independently of each other from the distribution of the \(k\)-nearest neighbors of \(x\) conditional on the \((k + 1)\)st-nearest neighbor;

\item[(3)]
sample the remaining \(n - k - 1\) points indpendently of each other from the distribution of the remaining points, conditional on the \((k + 1)\)st-nearest neighbor;

\item[(4)]
shuffle the points.
\end{itemize}

Now, the resulting sample has the same joint distribution as the standard iid sampling scheme.
Additionally, the event \(\eventA_{i, j}\) only depends on the points sampled in step (2).
Thus, we can apply Hoeffding's inequality to the random variables
\[
Z_{i}
=
q_{a} \ind\left\{Y_{i} = a\right\}
-
q_{b}\ind\left\{Y_{i} = b\right\}
\]
for \(i = 2, \ldots, k + 1\) to obtain the first inequality of the lemma.
The second follows immediately, completing the proof.
\end{proof}

\begin{lemma}
Let \(x\) be an arbitrary point of \(\xspace\).
Let \(p\) and \(\gamma\) be constants in \([0, 1]\).
Let \(k \leq (1 - \gamma)np\).
\begin{align*}
& \begin{aligned}
\prob\left(
\eventB
\right)
&\leq
\exp\left(-\frac{np\gamma^{2}}{2}\right)
\leq
\exp\left(-\frac{k \gamma^{2}}{2}
\right).
\end{aligned}
\end{align*}
\label{lemma:Lemma8Analogue}
\end{lemma}

\begin{proof}
The event \(\eventB\) satisfies
\begin{align*}
& \begin{aligned}
\eventB
&\subseteq
\left\{
\sum_{j = 1}^{n} \ind\{r_{p} \leq \rho(x, X_{j})\}
\leq
k
\right\}
%\\ &
\subseteq
\left\{
\sum_{j = 1}^{n} \ind\left\{
\rho(x, X_{j}) \leq r_{p}(x)
\right\}
\leq
(1 - \gamma) np
\right\}.
\end{aligned}
\end{align*}
Since \(p \leq \expect \ind\left\{\rho(x, X_{j}) \leq r_{p}(x)\right\}\),
we can apply the multiplicative Chernoff bound to obtain
\begin{align*}
& \begin{aligned}
\prob(\eventB)
&\leq
\exp\left(- \frac{np \gamma^{2}}{2}\right)
\leq
\exp\left(-\frac{k \gamma^{2}}{2}\right).
\end{aligned}
\end{align*}
This completes the proof.
\end{proof}

%---------------------------------------------%
%---------------------------------------------%
Now, we can prove our main theorem.

%---------------------------------------------%
%---------------------------------------------%
\begin{proof}[Proof of Theorem~\ref{theorem:Theorem5Analogue}]
Applying Lemma~\ref{lemma:Lemma7Analogue:Weighted} and taking expectations with respect to \(X\), we have
\begin{align}
& \begin{aligned}
\prob_{X}\left(f_{q, n, k}(X) \neq f^{*}_{q}(X)\right)
&\leq
\prob_{X}(\eventA) + \prob_{X}(\eventB) + \prob_{X}(\eventC).
\label{eqn:Theorem5AnalogueDecomposition}
\end{aligned}
\end{align}
Since \(\prob_{X}(\eventC) = P_{X}(\partial_{p, \Delta})\), it simply remains to show that
\begin{align*}
& \begin{aligned}
R
&:=
\prob_{X}(\eventA) + \prob_{X}(\eventB)
\leq
\delta
\end{aligned}
\end{align*}
with probability at least \(1 - \delta\).

By Lemma~\ref{lemma:Lemma8Analogue}, we have
\begin{align*}
& \begin{aligned}
R
&\leq
\expect_{X}\left[ \prob(\eventA | X = x)\right]
+
\exp\left(-\frac{k \gamma^{2}}{2}\right).
\end{aligned}
\end{align*}
Applying Lemma~\ref{lemma:UnionBound:Weighted} and Lemma~\ref{lemma:Lemma9Analogue:Weighted}, we have
\begin{align}
& \begin{aligned}
R
&\leq
\expect \sum_{c \neq f^{*}_{q}(x)} \prob(\eventA_{f^{*}_{q}(x),  c} | X = x)
+
\exp\left(-\frac{k\gamma^{2}}{2}\right)
%\\ &
\leq
C \exp\left(- \frac{k \Delta^{2}}{2 q_{\max}^{2}}\right)
+
\exp\left(-\frac{k \gamma^{2}}{2}\right).
\label{eqn:Theorem5UnionBound}
\end{aligned}
\end{align}

The final step of the proof is to plug in values for \(\Delta\) and \(\gamma\) and to use a simple concentration argument.
First, we have
\begin{align*}
& \begin{aligned}
C \exp\left(-\frac{k \Delta^{2}}{2 q_{\max}^{2}}\right)
&\leq
C \exp \left(\log \frac{1}{C} + 2 \log \frac{\delta}{2}\right)
\leq
\frac{\delta^{2}}{4}.
\end{aligned}
\end{align*}
Second, we have
\begin{align*}
& \begin{aligned}
\exp\left(-\frac{k \gamma^{2}}{2}\right)
&\leq
\exp\left(-\frac{k}{2} \cdot \frac{4}{k} \log \frac{2}{\delta}\right)
=
\frac{\delta^{2}}{4}.
\end{aligned}
\end{align*}
Thus, we have \(R \leq \delta^{2}\).
Applying Markov's inequality, we then have
\begin{align*}
& \begin{aligned}
\prob_{n}\left(R \geq \delta\right)
&\leq
\frac{\expect_{n} R}{\delta}
\leq
\delta.
\end{aligned}
\end{align*}
Combining this with equation~\eqref{eqn:Theorem5AnalogueDecomposition} completes the proof.
\end{proof}

%---------------------------------------------%
%---------------------------------------------%
\subsection{Proof of Corollary~\ref{corollary:SmoothTheorem}}

In this section, we prove Corollary~\ref{corollary:SmoothTheorem}.
We start with a lemma.

\begin{lemma}
If \(\eta\) is \((\alpha, L)\)-smooth, then for \(p \geq 0\) \(\Delta \geq 0\),we have
\[
\partial_{p, q, \Delta} \cap \support(P_{X})
\subseteq
\biggr\{
x \in \xspace
\biggr|
q_{f^{*}_{q}(x)} \eta_{f^{*}_{q}(x)}(x)
\leq
\max_{c \neq f^{*}_{q}(x)}q_{c} \eta_{c}(x) + \Delta + 2 q_{\max} L p^{\alpha}
\biggr\}.
\]
\label{lemma:Lemma18Analogue}
\end{lemma}

\begin{proof}
Pick \(p \geq 0\), \(r \leq r_{p}(x)\).
Let \(x\) be in \(\support(P_{X})\).
Then by \((\alpha, L)\) smoothness, we have
\begin{align*}
&\begin{aligned}
q_{f^{*}_{q}(x)} \eta_{f^{*}_{q}(x)}(x)
&\leq
q_{f^{*}_{q}(x)} \eta_{f^{*}_{q}(x)} \left(\closedball(x, r)\right) + q_{f^{*}_{q}(x)} L P_{X}\left(\closedball(x, r)\right)^{\alpha} \\
&\leq
q_{f^{*}_{q}(x)} \eta_{f^{*}_{q}(x)}\left(\closedball(x, r) \right) + q_{\max} L p^{\alpha}.
\end{aligned}
\end{align*}
Similarly, we have
\begin{align*}
&\begin{aligned}
q_{c} \eta_{c}(x)
&\geq
q_{c}\eta_{c}\left(\closedball(x, r)\right) - q_{c} L P_{X}\left(\closedball(x, r) \right)^{\alpha_{c}}
\geq
q_{c}\eta_{c} \left(\closedball(x, r)\right) - q_{\max} L p^{\alpha}
\end{aligned}
\end{align*}
for any \(c\).
Thus, if
\[
q_{f^{*}_{q}(x)} \eta_{f^{*}_{q}(x)}(x)
>
q_{c} \eta_{c}(x) + \Delta + 2 q_{\max} L p^{\alpha}
\]
for every \(c \neq f^{*}_{q}(x)\), then we see
\[
q_{f^{*}_{q}(x)} \eta_{f^{*}_{q}(x)}(B(x, r))
>
q_{c} \eta_{c}(B(x, r)) + \Delta
\]
for every \(c \neq f^{*}_{q}(x)\), which implies that \(x\) cannot be in \(\partial_{p, q, \Delta}\).
Thus, we have the desired inclusion.
\end{proof}

\begin{proof}[Proof of Corollary~\ref{corollary:SmoothTheorem}]
All that we have to do is apply Theorem~\ref{theorem:Theorem5Analogue} and Lemma~\ref{lemma:Lemma18Analogue} and then perform a few simple upper bounds.
From these two results, we have
\begin{align*}
&\begin{aligned}
\partial_{p, \Delta}
&\subseteq
\left\{
x \in \xspace \biggr|
q_{f^{*}_{q}(x)}\eta_{f^{*}_{q}(x)}(x) \leq  \max_{c \neq f^{*}_{q}(x)} q_{c}\eta_{c}(x) + \Delta + 2q_{\max} L p^{\alpha}
\right\} \\
&\subseteq
\left\{
x \in \xspace \biggr|
q_{f^{*}_{q}(x)} \eta_{f^{*}_{q}(x)}(x)
\leq
\max_{c \neq f^{*}_{q}(x)}
q_{c} \eta_{c}(x)
+
\sqrt{\frac{2q_{\max}^{2}}{k} \left(\log C + 2 \log \frac{2}{\delta}\right)}
+
2q_{\max} L p^{\alpha}
%\right. \\ &\qquad \left.
\right\},
\end{aligned}
\end{align*}
which completes the proof.
\end{proof}
%---------------------------------------------%
%---------------------------------------------%
\subsection{Proof of Corollary~\ref{corollary:MarginTheorem}}
In this section, we prove Corollary~\ref{corollary:MarginTheorem}.
For the first bound, we apply the margin condition to the result of Remark~\ref{remark:Smooth}.
This yields
\begin{align*}
&\begin{aligned}
\prob(f_{q,n, k}(X) \neq f^{*}_{q}(X))
&\leq
\delta + M \left(M'_{\alpha, L} q_{\max}^{\frac{2\alpha}{2\alpha + 1}}\left(2 \log C + 4 \log \frac{2}{\delta}\right)^{\frac{\alpha}{2 \alpha + 1}} n^{-\frac{\alpha}{2\alpha + 1}}\right)^{\beta} \\
&=
\delta
+
M_{\alpha, \beta, L}
q_{\max}^{\frac{2\alpha\beta}{2\alpha + 1}}
\left(2 \log C + 4 \log \frac{2}{\delta}\right)^{\frac{\alpha\beta}{2 \alpha + 1}} n^{-\frac{\alpha\beta}{2\alpha + 1}},
\end{aligned}
\end{align*}
which completes the proof.

%---------------------------------------------%
%---------------------------------------------%
For the risk, we have two lemmas, proved in Appendix~\ref{app:subsec:knn:ProofsOfLemmas}.
We start with a pointwise bound.

\begin{lemma}
Let \(\Delta(x) = \eta_{f^{*}_{q}(x)}(x) - \max_{c \neq f^{*}_{q}(x)} \eta_{c}(x)\).
Let \(P^{*} > 0\).
Then for \(x\) such that \(\Delta(x) > P^{*}\), we have
\begin{align*}
&\begin{aligned}
\excessrisk(f_{n, k}, x)
=
\expect \risk_{n, k}(x) - \risk^{*}(x)
&\leq
\exp\left(-\frac{k}{8}\right)
+
C \exp\left(
-\frac{k(\Delta(x) - P^{*})^{2}}{2}
\right).
\end{aligned}
\end{align*}
\label{lemma:Lemma20Analogue}
\end{lemma}

%---------------------------------------------%
%---------------------------------------------%
\begin{lemma}
Under the notation of Lemma~\ref{lemma:Lemma20Analogue}, set
\(
P := 2 L p^{\alpha},
\)
where \(p = 2k / n\).
Under the margin condition,
we have
\begin{align*}
&\begin{aligned}
\excessrisk(f_{n, k})
=
\expect_{X} \expect_{n} \left[\risk_{n, k}(X) - \risk^{*}(X)\right]
&\leq
\exp\left(-\frac{k}{8}\right)
+
C \cdot M
\max\left\{
P, \;
\left(
\frac{8(\beta + 1)}{k}
\right)^{\frac{1}{2}}
\right\}^{\beta}.
\end{aligned}
\end{align*}
\label{lemma:Lemma21Analogue}
\end{lemma}
%---------------------------------------------%
%---------------------------------------------%

To complete the proof of Corollary~\ref{corollary:MarginTheorem}, we simply apply Lemma~\ref{lemma:Lemma21Analogue} with
\(
k
=
C'_{\alpha, \beta, L} n^{2\alpha/(2\alpha + 1)}.
\)
This leads to the bound
\[
\excessrisk(f_{n, k})
\leq
M''_{\alpha, \beta, L} \cdot C n^{-\frac{\alpha \beta}{2\alpha + 1}},
\]
which concludes the proof.
\hfill
\qedsymbol

%---------------------------------------------%
%---------------------------------------------%
\subsection{Proof of Proposition~\ref{proposition:WeightedRisk}}

The goal of this section is to prove Proposition~\ref{proposition:WeightedRisk}.
We start by examining excess weighted risk \(\excessrisk_{q}\) and its relation to accuracy of the \(q\)-Bayes estimator \(f^{*}_{q}\).

\begin{lemma}
We have the bound
\[
\excessrisk_{q}(f)
=
\sum_{c = 1}^{C}
q_{c} p_{c}
\prob\left(f_{q, n, k}(X) \neq f^{*}_{q}(X)|
Y = c
\right).
\]
\label{lemma:WeightedRiskDecomposition}
\end{lemma}

\begin{proof}
The proof is straightforward;
we observe that
\begin{align*}
\excessrisk_{q}(f)
&=
\sum_{c = 1}^{C} q_{c} p_{c}
\left(
\prob\left(
f_{q, n, k}(X) \neq Y | Y = c
\right)
-
\prob\left(
f^{*}_{q}(X) \neq Y | Y = c
\right)
\right)\\
&=
\sum_{c = 1}^{C}
q_{c} p_{c}
\expect\left[
\ind\left\{
f_{q, n, k}(X) \neq f^{*}_{q}(X) = Y
\right\}
\middle|
Y = c
\right] \\
&=
\sum_{c = 1}^{C}
q_{c} p_{c}
\prob\left(f_{q, n, k}(X) \neq f^{*}_{q}(X) = Y|
Y = c
\right) \\
&\leq
\sum_{c = 1}^{C}
q_{c} p_{c}
\expect\left[
\ind\left\{
f_{q, n, k}(X) \neq f^{*}_{q}(X) = Y
\right\}
\middle|
Y = c
\right] \\
&=
\sum_{c = 1}^{C}
q_{c} p_{c}
\prob\left(f_{q, n, k}(X) \neq f^{*}_{q}(X) |
Y = c
\right).
\end{align*}
This completes the proof.
\end{proof}

%---------------------------------------------%
%---------------------------------------------%
Before we state the next lemma, recall that  \(\Delta_{q}(x) = q_{f^{*}_{q}(x)}(x) \eta_{f^{*}_{q}(x)}(x) - \max_{c \neq f^{*}_{q}(x)} q_{c} \eta_{c}(x)\).

\begin{lemma}
Let \(P^{*} > 0\).
Then for \(x\) such that \(\Delta_{q}(x) > P^{*}\), we have
\begin{align*}
&\begin{aligned}
\prob\left(f_{q, n, k}(x) \neq f^{*}_{q}(x) |
Y = c'
\right)
&\leq
\exp\left(-\frac{k}{8}\right)
+
C
\exp\left(- \frac{k(\Delta_{q}(x) - P^{*})^{2}}{2 q^{2}_{\max}}
\right).
\end{aligned}
\end{align*}
\label{lemma:Lemma20Analogue:Weighted}
\end{lemma}

\begin{proof}
By Lemma~\ref{lemma:Lemma7Analogue:Weighted}, we may apply the conditional measure \(\prob\left(\cdot| Y = c\right)\) to obtain
\begin{align*}
\prob\left(f_{q, n, k}(x) \neq f^{*}_{q}(x)|
Y = c'
\right)
&\leq
\prob\left(\eventA| Y = c'\right)
+
\prob\left(\eventB| Y = c'\right)
+
\prob\left(\eventC| Y = c'\right),
\end{align*}
where the events \(\eventA\), \(\eventB\), and \(\eventC\) are as in Lemma~\ref{lemma:Lemma7Analogue:Weighted}.
Thus, it suffices to bound these probabilities, and we do so in reverse order.

First, set
\(\Delta = \Delta_{q}(x) - P^{*}\).
Then since by assumption \(\Delta_{q}(x) > P^{*}\), we see that \(x\) is not in \(\partial_{p, \Delta - P}\) .
Thus, we have \(\ind(\eventC) = 0\) and therefore \(\prob\left(\eventC| Y = c'\right) = 0\).

Next, we consider the bound on the conditional probability of \(\eventB\).
Since \(\eventB\) is independent of the event \(\{Y = c'\}\), the bound of Lemma~\ref{lemma:Lemma8Analogue} holds for the conditional measure \(\prob\left(\cdot| Y = c'\right)\) instead of the usual measure, and so we have
\begin{align*}
\prob\left(\eventB | Y = c'\right)
&\leq
\exp\left(-\frac{k \gamma^{2}}{2}\right)
=
\exp\left(-\frac{k}{8}\right)
\end{align*}
when we set \(\gamma = 1/2\).

Finally, we consider the bound on the conditional probability of \(\eventA\).
Since the events \(\eventA_{a, b}\) are independent of the event \(\{Y = c'\}\), we may apply Lemma~\ref{lemma:UnionBound:Weighted} and Lemma~\ref{lemma:Lemma9Analogue:Weighted} to obtain
\begin{align*}
\prob\left(\eventA | Y = c'\right)
&\leq
\sum_{c \neq f^{*}_{q}(x)} \prob\left(\eventA_{f^{*}_{q}(x), c} | Y = c'\right)
 \\ &
\leq
\sum_{c \neq f^{*}_{q}(x)}
\exp\left(- \frac{k \Delta^{2}}{2 q_{\max}^{2}}
\right)
 \\ &
\leq
C
\exp\left(- \frac{k(\Delta_{q}(x) - P^{*})^{2}}{2 q^{2}_{\max}}
\right),
\end{align*}
and this completes the proof.
\end{proof}
%---------------------------------------------%
%---------------------------------------------%

\begin{lemma}
Under the notation of Lemma~\ref{lemma:Lemma20Analogue}, set
\[
P^{*} = P := 2 \max_{c = 1, \ldots, C}
L_{c} p^{\alpha},
\]
where \(p = 2k / n\).
Under the margin condition,
we have
\begin{align*}
&\begin{aligned}
\prob\left(f_{q, n, k}(X) \neq f^{*}_{q}(X) |
Y = c'
\right)
&\leq
\exp\left(-\frac{k}{8}\right)
+
M'' \cdot C
\max\left\{
P, \;
\left(
\frac{8(\beta + 1)}{k}
\right)^{\frac{1}{2}}
\right\}^{\beta}
\end{aligned}
\end{align*}
where \(M''\) is a constant.
\label{lemma:Lemma21Analogue:Weighted}
\end{lemma}
%---------------------------------------------%
%---------------------------------------------%

\begin{proof}
We start by decomposing the conditional error probability into two parts: when \(X\) is in the margin and when \(X\) is not in the margin.
We need to set the margin carefully, and so we define \(P_{j} = 2^{j} P\).
Here, \(j\) and later \(J\) are simply indices and are unrelated to the classes.
To this end, we write
\begin{align*}
\prob\left(f_{q, n, k}(X) \neq f^{*}_{q}(X) |
Y = c'
\right)
&=
\expect_{X}\left[
\ind\left\{
f_{q, n, k}(X) \neq f^{*}_{q}(X)
\right\}
\ind\left\{
X \in \partial _{p, \Delta - P}
\right\}
| Y = c'
\right]
\\ & \qquad  +
\expect_{X}\left[
\ind\left\{
f_{q, n, k}(X) \neq f^{*}_{q}(X)
\right\}
\ind\left\{
X \not \in \partial _{p, \Delta - P}
\right\}
| Y = c'
\right] \\
&=:
E_{1} + E_{2}.
\end{align*}
For the first term, if \(X\) is in \(\partial_{p, \Delta}\), then we require \(\Delta(X) \leq P_{J}\).
Thus, we have
\begin{align*}
E_{1}
&\leq
\expect\left[
\ind\left\{\Delta(X) \leq P_{c}\right\}
| Y = c'
\right] \\
&=
P_{X| Y = c'}\left(q_{f^{*}_{q}(X)} \eta_{f^{*}_{q}(X)}(X)
- \max_{c \neq f^{*}_{q}(X)} q_{c} \eta_{c}(X)
\leq P_{J}
\right) \\
&\leq
M P_{J}^{\beta},
\end{align*}
where the last inequality follows from the conditional margin assumption.
Thus, we simply need to bound \(E_{2}\).

By applying Lemma~\ref{lemma:Lemma20Analogue:Weighted}, we obtain
\begin{align*}
E_{2}
&\leq
\exp\left(-\frac{k}{8}\right)
+
C
\expect_{X|Y = c'}
\exp\left(
- \frac{k(\Delta(X) - P^{*})^{2}}{2 q_{\max}^{2}}
\right)
=:
\exp\left(-\frac{k}{8}\right)
+
C E_{3}.
\end{align*}
Thus, we need to bound \(E_{3}\).

We apply the margin condition for increasing values of \(\Delta(X)\), which gives
\begin{align*}
E_{3}
&=
\expect_{X | Y = c'}
\sum_{j \geq J}
\exp\left(- \frac{k}{2} (P_{j} - P)^{2}\right)
\ind\left\{
P_{j} < \Delta(X) \leq P_{j + 1}
\right\} \\
&=
\sum_{j \geq J}
\exp\left(- \frac{k}{2} (P_{j} - P)^{2}\right)
\prob_{X| Y = c'}\left(
P_{j} < \Delta(X) \leq P_{j + 1}
\right) \\
&\leq
\sum_{j \geq J}
\exp\left(- \frac{k}{2} (P_{j} - P)^{2}\right)
\prob_{X| Y = c'}\left(
\Delta(X) \leq P_{j + 1}
\right) \\
&\leq
M
\sum_{j \geq J}
P_{j + 1}^{\beta}
\exp\left(- \frac{k}{2} (P_{j} - P)^{2}\right).
\end{align*}
Now, we need to pick a good \(J\) so that the above series converges nicely.
We set
\[
J =
\max\left\{
1, \;
\left\lceil
\frac{1}{2} \log_{2} \frac{8(1 + \beta)}{k P^{2}}
\right\rceil
\right\}.
\]
With this choice of \(J\), we wish to show that our upper bound on each element of the series decreases by a factor of at least \(1/2\).
So, we have
\begin{align*}
Q
&:=
\frac{M P_{j + 1} \exp\left( -\frac{k}{2} (P_{j} - P)^{2}\right)}{M P_{j} \exp\left( -\frac{k}{2} (P_{j - 1} - P)^{2}\right)} \\
&=
2^{\beta} \exp\left(- \frac{k}{2}P^{2}
[(2^{j} - 1)^{2} - (2^{j - 1} - 1)^{2}]
\right) \\
&\leq
2^{\beta}
\exp\left(
-k P^{2} 2^{2j - 3}
\right),
\end{align*}
where the final inequality follows from Lemma~\ref{lemma:DyadicInequality}.
By our choice of \(J\), we have
\[
Q
\leq
2^{\beta}
\exp\left(-(1 + \beta)\right)
\leq
\frac{1}{2}.
\]
Returning to bounding \(E_{3}\), we have
\[
E_{3}
\leq
M P^{\beta}_{J} \exp\left(-\frac{k}{2}(P_{J - 1} - P)^{2}\right)
\sum_{j \geq J} 2^{-(j - J)} \\
\leq
M' P_{J}^{\beta},
\]
where \(M' = 2 M\) is a constant.
Putting everything together, we have
\begin{align*}
\prob\left(f_{q, n, k}(X) \neq f^{*}_{q}(X) |
Y = c
\right)
&\leq
\exp\left(-\frac{k}{8}\right)
+ M P_{J}^{\beta} + M' C P_{J}^{\beta}
\\
&\leq
\exp\left(-\frac{k}{8}\right)
+
M'' C P_{J}^{\beta} \\
&=
\exp\left(-\frac{k}{8}\right)
+
M''
C
\max\left\{
P, \;
\left(
\frac{8(\beta + 1)}{k}
\right)^{\frac{1}{2}}
\right\}^{\beta},
\end{align*}
where \(M'' = 2 M' = 4M\) is a constant and the final equality comes from plugging in our choice of \(J\).
This completes the proof.
\end{proof}

%---------------------------------------------%
%---------------------------------------------%

\begin{proof}[Proof of Proposition~\ref{proposition:WeightedRisk}]
By Lemma~\ref{lemma:WeightedRiskDecomposition} and Lemma~\ref{lemma:Lemma21Analogue:Weighted}, we have
\begin{align*}
\excessrisk_{q}(f)
&\leq
\left(\exp\left(-\frac{k}{8}\right)
+
M'' C
\max\left\{P, \;
\left(\frac{8(1 + \beta)}{k}\right)^{\frac{1}{2}}
\right\}^{\beta}
\right)
\left(
\sum_{c = 1}^{C}
q_{c} p_{c}
\right).
\end{align*}
If we set
\(
k = n^{2\alpha(2\alpha + 1)},
\)
then we obtain the bound
\[
\excessrisk_{q}(f)
\leq
M_{\alpha, \beta, L} C n^{-\frac{\alpha \beta}{2 \alpha + 1}}
\sum_{c = 1}^{C}
q_{c} p_{c}
\]
for some constant \(M_{\alpha, \beta, L}\).
This completes the proof.
\end{proof}
%---------------------------------------------%
%---------------------------------------------%
%---------------------------------------------%
%---------------------------------------------%
\subsection{Proof of Theorem~\ref{theorem:AccuracyLowerBound}}

For this proof, we embark on a three-part approach.
First, we use the sequential sampling scheme to control the \(k\)-nearest neighbors to a point \(x\).
Next, we fix some sub-optimal class \(c\) and reduce the multiclass classification to a binary classification where we get to choose between class \(c\) and the optimal class.
Third, we provide a lower bound on the probability of suboptimal classification for weighted binary classification.
Putting the proof together then involves unraveling our steps in reverse order.
The approach is largely based on the binary classification approach of \cite{chaudhuri2014rates}; however, we need to give a more extensive and careful conditioning argument.

%---------------------------------------------%
%---------------------------------------------%
\paragraph{Controlling the nearest neighbors.}

Now, we start the first step of the proof.
Let \(x\) be in \(\boundaryproximal_{q, n, k}\). 
We define the sampling scheme for \((X_1, Y_{1}), \ldots, (X_n, Y_{n})\) as follows:
\begin{enumerate}
	\item Pick  \(X_{(1)}\) in \(\xspace\) according to the marginal distribution of the \((k + 1)\)st nearest neighbor of \(x\).
	 Choose \(Y_{(1)}\) using the conditional measure \(\eta(X_{(1)})\)
	
	\item Pick \(k\) labels \(Y_{(2)}, \ldots, Y_{(k + 1)}\) from \(\eta(\closedball') \) where \(\closedball' = \closedball'(x, \rho(x, X_{(1)}))\).
	Sample the corresponding \(X_{(2)}, \ldots, X_{(k + 1)}\) from \(P_{X|Y_{(2)}}(\closedball'), \ldots, P_{X|Y_{(k + 1)}}(\closedball')\).

	\item Pick \(n - k - 1\) points \((X_{(k + 2)}, Y_{(k + 2)}), \ldots, (X_{(n)}, Y_{(n)})\) from the distribution \(P_{X, Y}\) restricted to \(\xspace\setminus \closedball'\).

	\item Randomly permute the \(n\) examples to obtain \((X_{1}, Y_{1}), \ldots, (X_{n}, Y_{n})\).
\end{enumerate}
By construction, we see that the resulting sample has the same distribution as simply sampling by the usual iid method.
However, the upshot is that this approach makes it clear that we can think of sampling as a relatively simple sequential process that in turn clarifies conditional independences.
More precisely, we ultimately need the sigma-fields \(\sfield_{1} = \sigma(\rho(x, X_{(1)})) = \sigma(\rho(x, X_{\sigma_{k + 1}(x)}))\) and
\(\sfield_{2} = \sigma(\rho(x, X_{(1)}), S)\), where \(S\) is a random variable to be defined shortly.

The \((k + 1)\)st nearest neighbor of \(x\) is given in step 1 of the sampling scheme, and we want to show that it lies in an annulus about \(x\).
Define
\[
\eventD_{1}
=
\left\{r_{k/n}(x) \leq \rho(x_0, X_{(k + 1)}(x_0)) \leq r_{(k + \sqrt{k} + 1)/n}(x) \right\}.
\]
We then have the following lemma.
\begin{lemma}[Lemma 11 of \citealt{chaudhuri2014rates}]
	There is a constant \(\gamma_1 > 0\) such that \(\gamma_1 \leq \prob(\eventD_{1})\).
\label{lemma:ChaudhuriLemma11}
\end{lemma}

%---------------------------------------------%
%---------------------------------------------%
\paragraph{Reduction to a binary problem.}

Our second step is to reduce to the binary case, and so to this end, we define
\begin{equation}
    S 
    = k\left(\hat{\eta}_{c}(x) + \hat{\eta}_{f^{*}_{q}(x)}\right)
    = \sum_{i = 1}^{n} \left( \ind\{Y_{i} = c\} + \ind\{Y_{i} = f^{*}_{q}(x)\}\right).
\label{eqn:DefinitionS}
\end{equation}
We then note that given \(X_{\sigma_{k + 1}(x)}\), the random variable \(S\) has a \(\binomial(k, p_{S}(\rho(x, X_{\sigma_{k + 1}(x)})))\) distribution where 
\[
p_{S}(r) = \eta_{c}(\closedball(x, r)) + \eta_{f^*_{q}(x)}(\closedball(x, r))).
\]
Ultimately, we want to treat the problem as one of choosing between \(c\) and \(f^*_{q}(x)\) using the \(S\) observations.
But before proceeding to that step, we state a simple lemma on the concentration of \(S\).

%---------------------------------------------%
%---------------------------------------------%
\begin{lemma}
Define the constants
\begin{align*}
    \underline{p}_{S} &= \inf_{r \in [r_{k/n}(x), r_{(k + \sqrt{k} + 1) / n}(x)]} p_{S}(r) \\ 
    \overline{p}_{S} &= \sup_{r \in [r_{k/n}(x), r_{(k + \sqrt{k} + 1) / n}(x)]} p_{S}(r).
\end{align*} 
Define the event 
\[
\eventD_{2} 
=
\left\{
k \underline{p}_{S} - \sqrt{k}
\leq 
S 
\leq k \overline{p}_{S} + \sqrt{k}
\right\}.
\]
Then, there is some constant \(\gamma_{2} > 0\) such that 
\(\gamma_{2} \leq \prob(\eventD_{2} | \sfield_{1})\) on the event \(\eventD_{1}\).
Moreover, we have
\[
\gamma_{2} \ind(\eventD_{1}) \leq \prob(\eventD_{2} | \sfield_{1}) \ind(\eventD_{1}).
\]
\label{lemma:BinomialSConcentration}
\end{lemma}
%---------------------------------------------%
%---------------------------------------------%

\begin{proof}
First, we quickly observe that the second inequality is an immediate consequence of the first.
Assuming that \(\ind(\eventD_{1}) = 1\), then we have \(\gamma_{2} \leq \prob(\eventD_{2} | \sfield_{1})\) by the first conclusion of the theorem.
If \(\ind(\eventD_{1}) = 0\), then the second inequality reduces to \(0 \leq 0\), which is certainly true.
Thus, we simply need to establish the first inequality.

The proof is a straightforward application of Hoeffding's inequality.
First, we note that
\[
\eventD_{2}'(r)
:=
\left\{
k p_{S}(r) - \sqrt{k}
\leq 
S 
\leq k p_{S}(r) + \sqrt{k}
\right\}
\subseteq 
\eventD_{2}
\]
for any \(r\) in \([r_{k/n}(x), r_{(k + \sqrt{k} + 1) / n}]\) on the event \(\eventD_{1}\).
Thus, it suffices to prove the result for \(\eventD_{2}'(r)\).
By the two-sided version of Hoeffding's inequality, we have 
\[
\gamma_{2} 
= 
1 - 2\exp(-2) 
\leq 
\prob(\eventD_{2}'(r) | \sfield{1})
\leq 
\prob(\eventD_{2} | \sfield_{1})
\]
on the event \(\eventD_{1}\) and this completes the proof of the lemma.
\end{proof}

%---------------------------------------------%
%---------------------------------------------%
\paragraph{Bounding the binary classification error.}

Finally, we arrive at the third step of the proof: analyzing the binary classification problem between \(c\) and \(f^{*}_{q}(x)\).
In particular, define the event
\[
\eventD_{3, c}
=
\left\{
q_{c} \hat{\eta}_{c}(\closedball') > q_{f^{*}_{q}(x)} \hat{\eta}_{f^{*}_{q}(x)}(\closedball')
\right\}
\]
where \(\closedball' = \closedball(x, \rho(x, X_{\sigma_{k + 1}(x)}))\).
This event is contained within the event \(\{f_{q, n, k}(x) \neq f^{*}_{q}(x)\}\).
Moreover, we can rewrite this as 
\begin{align*}
\eventD_{3, c}
&=
\left\{
q_{c} k \hat{\eta}_{c}(\closedball') 
> 
q_{f^{*}_{q}(x)} (S - k\hat{\eta}_{c}(\closedball'))
\right\}
=
\left\{
k \hat{\eta}_{c}(\closedball')
\geq 
\frac{q_{f^{*}_{q}(x)}}{q_{f^{*}_{q}(x)} + q_{c}} S
\right\}.
\end{align*}
We point out that in the case of uniform weighting, the threshold \(t = q_{f^{*}_{q}(x)}/(q_{f^{*}_{q}(x)} + q_{c})\) in the last expression reduces to \(1/2\); so indeed we have reduced the problem to a weighted binary classification problem (cf., Appendix~\ref{sec:AppConfusionMatrixQuantitiesBinary}).
Our task is now to lower bound the probability that \(k \hat{\eta}_{c}(\closedball')\) exceeds a threshold \(tS\). 
Since, under the appropriate conditioning, \(S' = k \hat{\eta}_{c}(\closedball')\) is a \(\binomial(S, p')\) random variable where
\(p' = \eta_{c}(\closedball') / (\eta_{c}(\closedball') + \eta_{f^*_{q}(x)}(\closedball'))\),
we naturally return to Slud's inequalities.

%---------------------------------------------%
%---------------------------------------------%
\begin{lemma}
Assume that \(p'\) satisfies \(0 < p' < 1/2\).
Suppose that \(t \leq 1 - p'\), that \(t - 1 / \sqrt{k} \leq p' \), and that \(k\) is sufficiently large to guarantee \(0 < t - 1 / \sqrt{k}\).
Then, there exists a constant \(\gamma_{3} > 0\) such that on the event \(\eventD_{2}\) we have the inequality
\(
\gamma_{3}
\leq 
\prob\left(\eventD_{3, c}
|
\sfield_{2}
\right),
\) 
and consequently,
\[
\gamma_{3} \ind(\eventD_{2})
\leq 
\prob\left(\eventD_{3, c}
|
\sfield_{2}
\right)
\ind(\eventD_{2}).
\]
\label{lemma:eventD3Bound}
\end{lemma}

Before we prove the lemma, we make a brief remark.
Note that if one is only interested in the unweighted setting, then we can drop the assumption that \(p' < 1/2\) because it is necessarily satisfied by the definition of \(f^{*}_{q}(x)\).

%---------------------------------------------%
%---------------------------------------------%

\begin{proof}
The second inequality is a direct consequence of the first: if \(\eventD_{3, c}\) occurs, then the former inequality provides the result; if \(\eventD_{3, c}\) does not occur, then both sides are zero.
Thus, it suffices to prove the second inequality.

First, we recall that \(S' = k \hat{\eta}_{c}(\closedball')\) has a \(\binomial(S, p')\) distribution, and both \(S\) and \(p'\) are measurable with respect to \(\sfield_{2}\).
Thus by expanding the measure to include some standard normal random variable \(Z\) independent of all other variables and invoking Lemma~\ref{lemma:SludsInequalities}, we have
\begin{align*}
\prob\left(
Z \geq \sqrt{\frac{S}{kp'}}
\middle| \sfield_{2} \right)
&\leq 
\prob\left(\eventD_{3, c} | \sfield_{2}\right).
\end{align*}
By our condition on \(p'\), we see that this is in turn lower bounded by
\begin{align*}
\prob\left(
Z \geq \sqrt{\frac{S}{k(t - 1/\sqrt{k})}}
\middle| \sfield_{2} \right)
&\leq 
\prob\left(\eventD_{3, c} | \sfield_{2}\right).
\end{align*}
So, the only remaining task is to apply the condition that \(\eventD_{2}\) occurs.
On \(\eventD_{2}\), the variable \(S\) is upper bounded by \(k \overline{p}_{S} + \sqrt{k}\), and so we have 
\begin{align*}
\prob\left(
Z \geq \sqrt{\frac{\overline{p}_{S} + 1 / \sqrt{k}}{t - 1/\sqrt{k}}}
\middle| \sfield_{2} \right)
&\leq 
\prob\left(\eventD_{3, c} | \sfield_{2}\right).
\end{align*}
Since the only random variable appearing in the probability on the left-hand size is the standard normal \(Z\), the left-hand side is lower bounded by some constant \(\gamma_{3}\),
and so we have 
\(\gamma_{3} \leq \prob\left(\eventD_{3, c} | \sfield_{2}\right)\),
as desired.
\end{proof}

%---------------------------------------------%
%---------------------------------------------%
\paragraph{Completing the proof.}

Now, the only thing left to do is to put the pieces together.
As a first step, we wish to lower bound the probability
\begin{align*}
P
=
\prob\left(\eventD_{1} \cap \eventD_{2} \cap \eventD_{3, c}\right)
&=
\expect\left[
\ind\left(\eventD_{1}\right) \ind\left(\eventD_{2}\right) \ind\left(\eventD_{3, c}\right) 
\right]
\end{align*}
for fixed \(x\), which we can randomize later.
To this end, our strategy is to use the tower property and our preceding lemmas.

Using the tower property of conditional expectations, we have
\begin{align*}
P
&=
\expect\left[
\expect\left[\ind\left(\eventD_{3, c}\right) | \sfield_{2}\right]
 \ind\left(\eventD_{2}\right)
\ind\left(\eventD_{1}\right) 
\right].
\end{align*}
By Lemma~\ref{lemma:eventD3Bound}, we obtain
\begin{align*}
\gamma_{3}
\expect\left[
\ind\left(\eventD_{2}\right)
\ind\left(\eventD_{1}\right) 
\right]
\leq 
P.
\end{align*}
Next, we again use the tower property to obtain
\[
\gamma_{3} \expect\left[ 
\expect\left[ \ind(\eventD_{2}) | \sfield_{1}\right]
\ind(\eventD_{1})
\right]
\leq 
P.
\]
Applying, Lemma~\ref{lemma:BinomialSConcentration}, we obtain
\[
\gamma_{2} \gamma_{3} \expect[\ind(\eventD_{1})] \leq P.
\]
Finally, Lemma~\ref{lemma:ChaudhuriLemma11} yields
\[
\gamma_{1}\gamma_{2} \gamma_{3} 
\leq 
P 
= 
\prob(\eventD_{1} \cap \eventD_{2} \cap \eventD_{3, c}).
\]
Likewise, if we define \(\eventD_{3} = \bigcup_{c \neq f^{*}_{q}(x)} \eventD_{3, c}\), then there exists a \(\gamma > 0\) such that 
\[
\gamma  
\leq 
\prob\left(\eventD_{1} \cap \eventD_{2} \cap \eventD_{3}\right)
\leq 
\prob_{n}(f_{q, n , k}(x) \neq f^{*}_{q}(x)).
\]
Replacing the fixed \(x\) by a random variable \(X\), taking expectations with respect to \(X\), and using Fubini's theorem then yields
\[
\gamma P_{X}(\boundaryproximal_{q, n, k})
\leq 
\expect_{n} \prob_{X}(f_{q, n , k}(X) \neq f^{*}_{q}(X)),
\]
which completes the proof.
\qedsymbol
%---------------------------------------------%
%---------------------------------------------%
%---------------------------------------------%
%---------------------------------------------%
\subsection{Proofs of Lemmas}
\label{app:subsec:knn:ProofsOfLemmas}

\begin{proof}[Proof of Lemma~\ref{lemma:Lemma20Analogue}]
Define the events \(\eventA\), \(\eventB\), and \(\eventC\) as in Lemma~\ref{lemma:Lemma7Analogue:Weighted}.
For the following, we set \(\Delta = \Delta(x) - P^{*}\), which implies that \(\ind(\eventC) = 0\).
By Lemma~\ref{lemma:Lemma7Analogue:Weighted}, we have
\begin{align}
\expect \risk_{n, k}(x) - \risk^{*}(x)
&=
\expect\left[
\left|\eta_{f^{*}_{q}(x)}(x) -\eta_{f_{n, k}(x)}(x)\right|
\ind\left\{f_{n, k}(x) \neq f^{*}(x)\right\}\right]
\nonumber \\
&\leq
\expect\left[
\ind(\eventA) + \ind(\eventB) + \ind(\eventC)
\right]
\nonumber \\
&\leq
\prob(\eventA) + \prob(\eventB).
\label{eqn:Lemma20AnalogueProof1}
\end{align}
So, it suffices to bound these two probabilities.

By Lemma~\ref{lemma:Lemma8Analogue}, with the choice of \(\gamma = 1/2\), we have
\begin{equation}
\prob(\eventB)
\leq
\exp\left(-\frac{k}{8}\right).
\label{eqn:Lemma20AnalogueProof2}
\end{equation}
Thus, it suffices to bound the probability of event \(\eventA\).

Using Lemma~\ref{lemma:UnionBound:Weighted} and Lemma~\ref{lemma:Lemma9Analogue:Weighted}, we have
\begin{align}
&\begin{aligned}
\prob(\eventA)
&\leq
\sum_{c \neq f^{*}_{q}(x)}
\prob(\eventA_{i, j})
%\\ &
\leq
\sum_{c \neq f^{*}_{q}(x)}
\exp\left(
-\frac{k (\Delta(x) - P^{*})^{2}}{2}
\right).
\label{eqn:Lemma20AnalogueProof3}
\end{aligned}
\end{align}
Combining equations~\eqref{eqn:Lemma20AnalogueProof1}, \eqref{eqn:Lemma20AnalogueProof2}, and \eqref{eqn:Lemma20AnalogueProof3} completes the proof.
\end{proof}

\begin{proof}[Proof of Lemma~\ref{lemma:Lemma21Analogue}]
We need to decompose the expectation into two parts based on when \(X\) is in \(\partial_{p, \Delta}\).
Define \(P_{j} = P \cdot 2^{j}\).
For the moment, let \(J > 0\) be an arbitrary integer.
Set \(\Delta = P_{J}\).
We have
\begin{align*}
&\begin{aligned}
\expect_{X}\left[
\expect R_{n, k}(X) - R^{*}(X)
\right]
&\leq
\expect_{X}  \left[
(\expect R_{n, k}(X) - R^{*}(X)) \ind\left\{X \in \partial_{p, \Delta}\right\}
\right]
\\ &\qquad +
\expect_{X}\left[
(\expect R_{n, k}(X) - R^{*}(X)) \ind\left\{X \not \in \partial_{p, \Delta}\right\}
\right] \\
&=:
E_{1} + E_{2}.
\end{aligned}
\end{align*}

The first term can be bounded by observing that \(X \in \partial_{p, \Delta}\) implies that \(\Delta(X) \leq P_{J}\).
Thus, we have
\begin{align*}
&\begin{aligned}
E_{1}
&=
\expect_{X}  \left[
\expect (\eta_{f^{*}_{q}(x)}(X) - \eta_{f_{n, k}(X)}(X)) \ind\left\{f_{n, k}(x) \neq f^{*}(x)\right\}
\ind\left\{\Delta(X) \leq P_{J}\right\}
\right] \\
&\leq
\expect_{X}
\ind\left\{\Delta(X) \leq P_{J}\right\} \\
&=
P_{X}\left\{
\eta_{f^{*}_{q}(x)}(X) - \max_{c \neq f^{*}_{q}(x)}\eta_{c}(X) \leq P_{J}
\right\} \\
&\leq
M P_{J}^{\beta},
\end{aligned}
\end{align*}
where the last inequality follows from the margin condition.
Thus, it remains to bound \(E_{2}\).

For \(E_{2}\), we apply Lemma~\ref{lemma:Lemma20Analogue} to obtain
\begin{align*}
&\begin{aligned}
E_{2}
&\leq
\exp\left(-\frac{k}{8}\right)
+
C \expect_{X}
\exp\left(-\frac{k (\Delta(X) - P)^{2}}{2}
\right)
% \\ &
=:
\exp\left(-\frac{k}{8}\right)
+
C E_{3}.
\end{aligned}
\end{align*}
Thus, it only remains to bound \(E_{3}\).

To do this, we want to apply the margin condition for successively increasing values of \(\Delta(X)\).
Then, we have
\begin{align*}
 &\begin{aligned}
E_{3}
&=
\expect_{X} \sum_{j \geq J}
\exp\left(-k\frac{(\Delta(X) - P)^{2}}{2}
\right)
\ind\left\{
P_{j} < \Delta(X) \leq P_{j + 1}
\right\} \\
&\leq
\sum_{j \geq J}
\exp\left(-k \frac{(P_{j} - P)^{2}}{2}\right)
\prob\left(\Delta(X) \leq P_{j + 1}
\right)
\\
&\leq
M \sum_{j \geq J}
P_{j + 1}^{\beta}
\exp\left(-k \frac{(P_{j} - P)^{2}}{2}\right).
 \end{aligned}
 \end{align*}

Now, our goal is to pick \(J\) so that this sum is finite.
To this end, we pick
\[
J
=
\max\left\{1, \;
\left\lceil \frac{1}{2} \log_{2} \frac{8(1 + \beta)}{k P^{2}}
\right\rceil
\right\}.
\]
With this choice of \(J\), we observe that our upper bound for each element of the series decreases by a factor of at least \(1/2\).
This can be seen by observing
\begin{align*}
&\begin{aligned}
Q
&:=
\frac{M P_{j + 1} \exp\left(- k \frac{(P_{j} - P)^{2}}{2}\right)}{M P_{j} \exp\left(- k \frac{(P_{j - 1} - P)^{2}}{2}\right)} \\
&=
2^{\beta}
\exp\left(- \frac{k}{2}P^{2}[(2^{j} - 1)^{2} - (2^{j - 1} - 1)^{2}]
\right) \\
&\leq
2^{\beta}
\exp\left(-k P^{2} 2^{2j - 3}\right).
\end{aligned}
\end{align*}
The last inequality follows from Lemma~\ref{lemma:DyadicInequality}.
From our choice of \(J\) above, we see that for \(j \geq J\) we have
\begin{align*}
&\begin{aligned}
Q
&\leq
2^{\beta} \exp(-(1 + \beta))
\leq
\frac{1}{2}.
\end{aligned}
\end{align*}

Returning to bounding \(E_{3}\), we now have
\begin{align*}
&\begin{aligned}
E_{3}
&\leq
M P_{J}^{\beta} \exp\left(-k\frac{(P_{J - 1} - P)^{2}}{2}
\right)
\sum_{j \geq J} 2^{-j}
% \\ &
\leq
M P_{J}^{\beta}.
\end{aligned}
\end{align*}
Putting everything together, we have
\begin{align*}
&\begin{aligned}
\expect_{X}\left[
\expect R_{n, k}(X) - R^{*}(X)
\right]
&\leq
M P_{J}^{\beta}
+
\exp\left(-\frac{k}{8}\right)
+
C M P_{J}^{\beta} \\
&=
\exp\left(-\frac{k}{8}\right)
+
C \cdot M P_{J}^{\beta} \\
&=
\exp\left(-\frac{k}{8}\right)
+
C \cdot M
\max\left\{
P, \;
\left(
\frac{8(\beta + 1)}{k}
\right)^{\frac{1}{2}}
\right\}^{\beta},
\end{aligned}
\end{align*}
where the maximum term comes from the definition of \(J\).
This completes the proof.
\end{proof}

%---------------------------------------------%
%---------------------------------------------%
%---------------------------------------------%
%---------------------------------------------%
\subsection{Auxiliary Lemmas}

\begin{lemma}
We have the inequality
\[
(2^{j} - 1)^{2} - (2^{j - 1} - 1)^{2}
\geq
2^{j - 1}.
\]
\label{lemma:DyadicInequality}
\end{lemma}

\begin{proof}
This is satisfied for the case of \(j = 1\) with equality; so we consider \(j > 1\).
We can bound the left hand side by
\begin{align*}
&\begin{aligned}
(2^{j} - 1)^{2} - (2^{j - 1} - 1)^{2}
&=
2^{2j} - 2^{2(j - 1)} - 2^{j + 1} + 2^{j} \\
&\geq
2^{2j} - 2^{2(j - 1)} - 2^{j + 1} \\
&\geq
4 \cdot 2^{2(j - 1)} - 2^{2(j - 1)} - 2 \cdot 2^{2(j - 1)} \\
&=
2^{2(j - 1)}.
\end{aligned}
\end{align*}
Note that in the second inequality we used \(j < 2(j - 1)\) for \(j > 1\).
This completes the proof.
\end{proof}

%---------------------------------------------%
%---------------------------------------------%
\section{Uniform Convergence of the Nearest Neighbor Regressor}
\label{app:UniformConvergenceProof}
%---------------------------------------------%
%---------------------------------------------%
\subsection{Upper Bounds}

In this appendix, we prove Theorem~\ref{thm:multiclass_unif_convergence}. Note that, in this section, we use $\|\cdot\|_\infty$ to denote the vector $\sup$-norm on $\R^C$ and $\|\cdot\|_{\X,\infty}$ to denote the function $\sup$-norm on $\X$.

We begin with a simple lemma concerning the behavior of multinomial random variables:

%---------------------------------------------%
%---------------------------------------------%
\begin{lemma}
Suppose $X \sim \text{Multinomial}(k, \mu)$ is drawn from a multinomial distribution with $k$ trials and mean $\mu \in \Delta^{C - 1}$. Then, for any $\delta \in (0,1)$, with probability at least $1 - \delta$, \[\|X/k - \mu\|_\infty \leq \frac{1}{\sqrt{k}} + \sqrt{\frac{\log \frac{1}{\delta}}{2k}}.\]
\label{lemma:multinomial}
\end{lemma}
%---------------------------------------------%
%---------------------------------------------%

\begin{proof}
By Jensen's inequality and the fact that $\sum_{c \in [C]} \mu_c = 1$,
\[\expect_X \left[ \|X/k - \mu\|_\infty \right]
   \leq \sqrt{\expect_X \left[ \|X/k - \mu\|_2^2 \right]}
   = \sqrt{\sum_{c \in [C]} \expect_X \left[ \left( X_c/k - \mu_c \right)^2 \right]}
   = \sqrt{\sum_{c \in [C]} \frac{\mu_c (1 - \mu_c)}{k}}
   \leq \frac{1}{\sqrt{k}}.\]

Note that $\|X/k - \mu\|_\infty$ is $1/k$-Lipschitz in each of the $k$ independent trials of $X$. Hence, by McDiarmid's inequality, with probability at least $1 - \delta$,
\[\|X/k - \mu\|_\infty
 \leq \expect_X \left[ \|X/k - \mu\|_\infty \right] + \sqrt{\frac{1}{2k}\log \frac{1}{\delta}}.\]
\end{proof}

We now proceed to prove Theorem~\ref{thm:multiclass_unif_convergence}.
For any $x \in \X$, let
\[\tilde \eta_{k,c}(x) := \frac{1}{k} \sum_{j = 1}^k \eta_c(X_{\sigma_j(x)})\] denote the mean of the true regression function over the $k$ nearest neighbors of $x$. By the triangle inequality,
\[\|\eta_c - \hat\eta_{c}\|_{\X,\infty}
  \leq \|\eta_c - \tilde \eta_{k,c}\|_{\X,\infty}
     + \|\tilde \eta_{k,c} - \hat \eta_{c}\|_{\X,\infty},\]
wherein $\|\eta_c - \tilde \eta_{k,c}\|_{\X,\infty}$ captures bias due to smoothing and $\|\tilde \eta_{k,c} - \hat \eta_{c}\|_{\X,\infty}$ captures variance due to label noise. We separately show that, with probability at least $1 - N \left( \left( \frac{2k}{p_* n} \right)^{1/d} \right) e^{-k/4}$,
\[\max_{c \in [C]} \left\| \eta_c - \tilde \eta_{k,c} \right\|_{\xspace, \infty}
  \leq 2^\alpha L\left( \frac{2k}{p_* n} \right)^{\alpha/d},\]
and that, with probability at least $1 - \delta$,
\[\max_{c \in [C]} \|\tilde \eta_{k,c} - \hat \eta_{c}\|_{\X,\infty}
  \leq \frac{1}{\sqrt{k}} + \sqrt{\frac{1}{2k} \log \frac{S(n)}{\delta}}.\]

%---------------------------------------------%
%---------------------------------------------%
\paragraph{Bounding the smoothing bias.}

Fix some $r > 0$ to be determined, and let $\{B_r(z_1),...,B_r(z_{N(r)})\}$ be a covering of $(\X, \rho)$ by $N(r)$ balls of radius $r$, with centers $z_1,...,z_{N(r)} \in \X$.

By the lower bound assumption on $P_X$, each $P_X(B_r(z_j)) \geq p_* r^d$. By a multiplicative Chernoff bound, with probability at least $1 - N(r) e^{-p_* n r^d/8}$, each $B_r(z_j)$ contains at least $p_* n r^d/2$ samples. In particular, if $r \geq \left( \frac{2k}{p_* n} \right)^{1/d}$, then each $B_k$ contains at least $k$ samples, and it follows that, for every $x \in \X$, $\rho(x, X_{\sigma_k(x)}) \leq 2r$.
Notably, this argument in no way depends on the classes $Y_1,...,Y_n$, and hence by the H\"older continuity of $\eta_c$,
\[
\left| \eta_c(x) - \tilde \eta_{k,c}(x) \right|
  = \left| \eta_c(x) - \frac{1}{k} \sum_{j = 1}^k \eta_c(X_{\sigma_j(x)}) \right|
  \leq \frac{1}{k} \sum_{j = 1}^k \left| \eta_c(x) - \eta_c(X_{\sigma_j(x)}) \right|
  \leq L(2r)^\alpha
\]
for every \(c\) in \([C]\) without a union bound. Finally, if $\frac{k}{n} \leq \frac{p_*}{2} (r^*)^d$, then we can let $r = \left( \frac{2k}{p_* n} \right)^{1/d}$.

%---------------------------------------------%
%---------------------------------------------%
\paragraph{Bounding variance due to label noise.}

Let $\Sigma := \{\sigma(x) \in [n]^k : x \in \X \}$ denote the set of possible $k$-nearest neighbor index sets. One can check from the definition of the shattering coefficient that $|\Sigma| \leq S(n)$.

For any $\sigma \in [n]^k$, let $Z_\sigma$ denote the $\R^C$-valued random variable whose $c^{th}$ coordinate is given by
\[Z_{\sigma, c} := \sum_{j = 1}^k 1\{Y_{\sigma_j(X_i)} = c\}.\]
Then, the conditional random variable $Z_\sigma|X_1,...,X_n$ has a multinomial distribution with $k$ trials and mean $\mu_\sigma := \frac{1}{k} \sum_{j = 1}^k \eta(X_i)$, by Lemma~\ref{lemma:multinomial},
\[\prob \left( \left\| \mu_\sigma - Z_\sigma/k \right\|_\infty > \frac{1}{k} + \epsilon \middle| X_1,...,X_n \right)
  \leq e^{-2k\epsilon^2}.\]
Moreover, for any $x \in \X$, $\mu_{\sigma(x)} = \tilde \eta_k(x)$ and $Z_{\sigma(x)}/k = \hat \eta(x)$. Hence, by a union bound over $\sigma$ in $\Sigma$,
\begin{align*}
    \prob \left( \sup_{x \in \X} \left\| \tilde \eta_k(x) - \hat \eta(x) \right\|_\infty > \frac{1}{k} + \epsilon \middle| X_1,...,X_n \right)
    & = \prob \left( \sup_{x \in \X} \left\| \mu_{\sigma(x)} - Z_{\sigma(x)}/k\right\|_\infty > \frac{1}{k} + \epsilon \middle| X_1,...,X_n \right) \\
    & \leq \prob \left( \sup_{\sigma \in \Sigma} \left\| \mu_\sigma - Z_\sigma/k\right\|_\infty > \frac{1}{k} + \epsilon \middle| X_1,...,X_n \right) \\
    & \leq \sum_{\sigma \in \Sigma} \prob \left( \left\| \mu_\sigma - Z_\sigma/k\right\|_\infty > \frac{1}{k} + \epsilon \middle| X_1,...,X_n \right) \\
    & \leq |\Sigma| e^{-2k\epsilon^2}
      \leq S(n) e^{-2k\epsilon^2}.
\end{align*}
Since the right-hand side is independent of $X_1,...,X_n$, the unconditional bound
\[\prob \left( \sup_{x \in \X} \left\| \tilde \eta_k(x) - \hat \eta(x) \right\|_\infty > \frac{1}{k} + \epsilon \right)
  \leq S(n) e^{-2k\epsilon^2}.\]
also holds. Setting $\epsilon := \sqrt{\frac{1}{2k} \log \frac{S(n)}{\delta}}$ gives the final result.
\hfill 
\qedsymbol
%---------------------------------------------%
%---------------------------------------------%
\subsection{Proof of Theorem~\ref{theorem:UniformErrorLowerBound}}

In this section, we prove our lower bound on the minimax uniform error of estimating a H\"older continuous regression function. We use a standard approach based on the following version of Fano's lemma:
\begin{lemma} (Fano's Lemma; Simplified Form of Theorem 2.5 of \citealt{tsybakov2009introduction})
    Fix a family $\P$ of distributions over a sample space $\X$ and fix a pseudo-metric $\rho : \P \times \P \to [0,\infty]$ over $\P$. Suppose there exist $P_0 \in \P$ and a set $T \subseteq \P$ such that
    \[\sup_{P \in T} D_{KL}(P,P_0)
      \leq \frac{\log |T|}{16},\]
    where $D_{KL} : \P \times \P \to [0,\infty]$ denotes Kullback-Leibler divergence.
    Then,
    \[\inf_{\hat P} \sup_{P \in \P} 
    \prob \left( \rho(P,\hat P)
      \geq \frac{1}{2} \inf_{P \in T} \rho(P,P_0) \right) \geq 1/8,\]
    where the first $\inf$ is taken over all estimators $\hat P$.
    \label{thm:tsybakov_fano}
\end{lemma}

Now, we proceed with the proof.
We now proceed to construct an appropriate $P_0 \in \P$ and $T \subseteq \P$. Let $g : [-1,1]^d \to [0,1]$ defined by
\[g(x) = \left\{ \begin{array}{cc}
    \exp \left( 1 - \frac{1}{1 - \|x\|_2^2} \right) & \text{ if } \|x\|_2 < 1 \\
    0 & \text{ else }
\end{array} \right.\]
denote the standard bump function supported on $[-1,1]^d$, scaled to have $\|g\|_{\X,\infty} = 1$. Since $g$ is infinitely differentiable and compactly supported, it has a finite $\alpha$-H\"older semi-norm:
\begin{equation}
\|g\|_{\Sigma^\alpha}
  := \sup_{\ell \in \mathbb{N}^d : \|\ell\|_1 \leq \alpha} \quad \sup_{x \neq y \in \X} \quad \frac{|g^\ell(x) - g^\ell(y)|}{\|x - y\|^{\alpha - \|\ell\|_1}}
  < \infty,
\label{eqn:HolderSeminorm}
\end{equation}
where $\ell$ is any $\lfloor \beta \rfloor$-order multi-index and $g^\ell$ is the corresponding mixed derivative of $g$.
Define $M := \left( \frac{32n}{\log n} \right)^{\frac{1}{2\alpha + d}} \geq 1$. For each $m \in [M]^d$, define $g_m : \X \to [0,1]$ by
\[g_m (x):= \left\{ \begin{array}{cc}
    g\left( Mx - \frac{2m - 1_d}{2} \right) & \text{ if } \|x\|_2 \leq 1 \\
    0 & \text{ else }
\end{array} \right..\]
Let $\eta_0 \equiv \frac{1}{4}$ denote the constant-$\frac{1}{4}$ function on $\X$. Finally, for each $m \in [M]^d$, define $\eta_m : \X \to [0,1]$ by
\[\eta_m := \eta_0 + \min \left\{ \frac{1}{2}, \frac{L}{\|g\|_{\Sigma^\alpha}} \right\} M^{-\alpha} g_m,\]
Note that, for any $m \in [M]^d$,
\[\|\eta_m\|_{\Sigma^\alpha}
  \leq L M^{-\alpha} \frac{\|g_m\|_{\Sigma^\alpha}}{\|g\|_{\Sigma^\alpha}}
  = L,\]
so that $\eta_m$ satisfies the H\"older smoothness condition.
For any particular $\eta$, let $P_\eta$ denote the joint distribution of $(X, Y)$. Note that $P_\eta(x, 1) = \eta(x) \geq 1/4$. Moreover, one can check that, for all $x \geq -2/3$, $-\log(1 + x) \leq x^2 - x$. Hence, for any $x \in \X$,
\begin{align*}
    P_{\eta_m}(x, 1) \log \frac{P_{\eta_m}(x, 1)}{P_\eta(x, 1)}
    & = \eta_m(x) \log \frac{\eta_m(x)}{\eta(x)} \\
    & = -\eta_m(x) \log \left( 1 + \frac{\eta(x) - \eta_m(x)}{\eta_m(x)} \right) \\
    & \leq \eta_m(x) \left( \left( \frac{\eta(x) - \eta_m(x)}{\eta_m(x)} \right)^2 - \frac{\eta(x) - \eta_m(x)}{\eta_m(x)} \right) \\
    & = \frac{\left( \eta(x) - \eta_m(x)\right)^2 }{\eta_m(x)} - \eta(x) + \eta_m(x) \\
    & \leq 4\left( \eta(x) - \eta_m(x)\right)^2 - \eta(x) + \eta_m(x),
\end{align*}
and, similarly, since $P_\eta(x,0) = 1 - \eta(x) \geq 1/4$,
\[P_{\eta_m}(x, 0) \log \frac{P_{\eta_m}(x, 0)}{P_\eta(x, 0)}
    \leq 4\left( \eta(x) - \eta_m(x)\right)^2 + \eta(x) - \eta_m(x).\]
Adding these two terms gives
\begin{align*}
D_{\text{KL}}\left( P_\eta^n, P_{\eta_m}^n \right)
 & = n \left( \int_\X P_\eta(x, 0) \log \frac{P_\eta(x, 0)}{P_{\eta_m}(x, 0)} \, dx + \int_\X P_\eta(x, 1) \log \frac{P_\eta(x, 1)}{P_{\eta_m}(x, 1)} \, x \right) \\
 & \leq 8n\int_\X \left( \eta(x) - \eta_m(x)\right)^2 \\
 & = 8n\|\eta - \eta_m\|_2^2 \\
 & \leq 2n M^{-2\alpha} \|g_m\|_2^2 \\
 & = 2n M^{-(2\alpha + d)} \|g\|_2^2 \\
 & = 2n \left( \left( \frac{32 n}{\log n} \right)^{\frac{1}{2\alpha + d}}\right)^{-(2\alpha + d)} \|g\|_2^2 \\
 & = \frac{1}{16} \log n \|g\|_2^2
   \leq \frac{d}{16} \log M
   = \frac{\log |[M]^d|}{16}.
\end{align*}
Fano's lemma therefore implies the lower bound
\[\inf_{\hat \eta} \sup_{\eta \in \Sigma^\alpha(L)} 
\prob_{\{(X_i,Y_i)\}_{i = 1}^n \sim P_\eta^n} \left( \left\|\eta - \hat \eta \right\|_{\X,\infty} \geq \frac{1}{2} \min \left\{ \frac{1}{2}, \frac{L}{\|g\|_{\Sigma^\alpha}} \right\} \left( \frac{\log n}{32n} \right)^{\frac{\alpha}{2\alpha + d}}\right)
  \geq 1/8,\]
  which completes the proof.
\hfill 
\qedsymbol

%---------------------------------------------%
%---------------------------------------------%
\section{General Classification Metrics in Binary Classification}
\label{sec:AppConfusionMatrixQuantitiesBinary}

While the main focus of our paper is the multiclass classification setting, in this appendix, we consider general classification metrics in binary classification.
We consider the binary case because because it is an important case, it simplifies the notation, and the results are stronger.

There are also two other important properties that have been  observed previously in this setting.
First, class-weighting of the regression function is equivalent to setting a threshold for classification, which we explain in greater detail in the sequel.
Second, for a certain family of general classification metrics to be maximized, including precision and recall but not F1 score, choosing the correct threshold maximizes the general classification metric.

Thus, the organization of this section is as follows.
In Section~\ref{app:subsec:binaryCM:Setup}, we discuss the notational simplifications for binary classification and the equivalence of weighting and thresholding.
In Section~\ref{app:subsec:binaryCM:Results}, we discuss the main results on the empirical convergence of confusion matrix quantities to the optimal confusion matrix quantities uniformly over the choice of threshold.
In Section~\ref{app:subsec:BinaryClassification:PrecisionRecallF1}, we provide corollaries for the uniform convergence of precision, recall, and F1 score, although our method is far more general and may be applied to other general classification metrics.
Finally, in Section~\ref{app:subsec:binaryCM:proofs}, we prove the results of this appendix.

%---------------------------------------------%
%---------------------------------------------%
\subsection{Binary Classification Setup}
\label{app:subsec:binaryCM:Setup}

In this section, we discuss modifications for the binary case, examine the equivalence of weighting and thresholding, and define our confusion matrix quantities.
For the binary case, we first let \(\yspace = \{0, 1\}\), as is traditional in nonparametric classification.
Here, the regression function \(\eta: \xspace \to [0, 1]\) is defined to be
\[
\eta(x)
=
\prob\left(Y = 1 | X = x\right).
\]

%---------------------------------------------%
%---------------------------------------------%
\subsubsection{Thresholding and Weighting in Plug-In Classification}

Next, we consider the equivalence of weighting and thresholding in plug-in classification.
Usually in plug-in classification, given a point \(X\) to classify, we pick the class with the highest estimated probability.
Thus, we pick \(1\) if \(\hat{\eta}(X) \geq 1 - \hat{\eta}(X)\) or \(0\) otherwise.
By rearranging, we obtain the new decision rule where we select class \(1\) if \(\hat{\eta}(X) \geq 1/2\) and \(0\) otherwise.

The modification for thresholding is immediate: instead of selecting class \(1\) for \(\hat{\eta}(X) \geq 1/2\), we can select \(1\) for \(\hat{\eta}(X) \geq t\), where \(t\) is a threshold in \([0, 1]\).
The modification for weighting requires slightly more work.
Given non-negative weights \(q_{0}\) and \(q_{1}\), we revise the original decision rule; so we predict class \(1\) when
\[
q_{1} \eta(X) \geq q_{0}(1 - \eta(X)).
\]
Rearranging, this is equivalent to 
\[
\eta(X) \geq \frac{q_{0}}{q_{0} + q_{1}}.
\]
Thus, picking a set of non-negative weights is equivalent to choosing a different threshold in binary classification.
For our binary classification results, we consider thresholding because we find it is more intuitive, but we note that weighting extends more cleanly to the multiclass case considered in the rest of the paper.

%---------------------------------------------%
%---------------------------------------------%
\subsubsection{Confusion Matrix Quantities}

In binary classification, the confusion matrix is generally only computed with respect to the class \(Y = 1\).
Thus, the empirical confusion matrix quantities are
\begin{align}
&\begin{aligned}
\widehat{\truenegative}(t) 
= \frac{1}{n} \sum_{i = 1}^{n} \ind\left\{Y_{i} \neq 1\right\} \ind\left\{\hat{\eta}(X_{i}) < t\right\}  
&\qquad 
\widehat{\falsenegative}(t) 
= \frac{1}{n} \sum_{i = 1}^{n} \ind\left\{Y_{i} = 1\right\} \ind\left\{\hat{\eta}(X_{i}) < t\right\}\\
\widehat{\falsepositive}(t) 
= \frac{1}{n} \sum_{i = 1}^{n} \ind\left\{Y_{i} \neq 1\right\} \ind\left\{\hat{\eta}(X_{i}) \geq t\right\}
&\qquad 
\widehat{\truepositive}(t) 
= \frac{1}{n} \sum_{i = 1}^{n} \ind\left\{Y_{i} = 1\right\} \ind\left\{\hat{\eta}(X_{i}) \geq t\right\}.
\end{aligned}
\end{align}
The Bayes confusion matrix is a result of evaluating the confusion matrix quantities on the true distribution, which gives
\begin{align}
&\begin{aligned}
\truenegative(t) = \int_{\xspace} (1 - \eta(X)) \ind\left\{\eta(X) < t\right\} dP_{X}  
&\qquad 
\falsenegative(t) = \int_{\xspace} \eta(X) \ind\left\{\eta(X) < t\right\} dP_{X}\\
\falsepositive(t) = \int_{\xspace} (1 - \eta(X)) \ind\left\{\eta(X) \geq t\right\} dP_{X}
&\qquad 
\truepositive(t) = \int_{\xspace} \eta(X) \ind\left\{\eta(X) \geq t\right\} dP_{X}.
\end{aligned}
\end{align}

We would like to maximize \(\truenegative\) and \(\truepositive\) and minimize \(\falsepositive\) and \(\falsenegative\).
Or, alternatively, we would like to maximize or minimize some function of these quantities.
These quantities are not computable because we do not know the underlying data-generating distribution, but we may estimate these quantities through their empirical counterparts.
Thus, the goal is to show that the empirical confusion matrix converges to the Bayes confusion matrix.
Further, since we generally pick the thresholds after examining the empirical confusion matrix, we would like the convergence to be uniform over the choice of threshold.

%---------------------------------------------%
%---------------------------------------------%
\subsection{Uniform Convergence of the Confusion Matrix}
\label{app:subsec:binaryCM:Results}

In this section, we present our uniform convergence result for the empirical confusion matrix to the Bayes confusion matrix.
We start with a ``uniform convergence'' assumption for the regression estimator.

\begin{assumption}
We assume the regression function estimate converges uniformly.
More precisely, for every \(\epsilon > 0\), and for all \(n\) sufficiently large, we have with probability at least \(1 - \delta / 2\)
\[
\|\eta - \hat{\eta}\|_{\xspace, \infty} 
\leq 
\epsilon.
\]
\label{assumption:UniformEstimateBinary}
\end{assumption}
%---------------------------------------------%
%---------------------------------------------%
Now, we state the main result for this appendix.
%---------------------------------------------%
%---------------------------------------------%
\begin{proposition}
Suppose that Assumption~\ref{assumption:UniformEstimateBinary} holds.
Let \(T = \{t_{i}\}_{i = 1}^{N}\), where \(0 \leq t_{1} < \ldots < t_{N} \leq 1\).
\begin{align*}
\tne(t', t'') 
&= 
\fpe(t', t'')
=
\prob\left(
Y \neq 1, \; t' - \epsilon \leq \eta(X) \leq t'' + \epsilon
\right) \\
\fne(t', t'')
&=
\tpe(t', t'')  =
\prob\left(
Y = 1, \; t' - \epsilon \leq \eta(X) \leq t'' + \epsilon
\right) 
%\\ 
%\fpe(t', t'')
%&=
%\prob\left(
%Y \neq 1, \; t' - \epsilon \leq \eta(X) \leq t'' + \epsilon
%\right) \\ 
%\tpe(t', t'') 
%&=
%\prob\left(
%Y \neq 1, \; t' - \epsilon \leq \eta(X) \leq t'' + \epsilon
%\right)
\end{align*}
Then with probability at least \(1 - \delta\), we have the bounds
\begin{align*}
|\hat{\truenegative}(t) - \truenegative(t)|
&\leq 
3 \inf_{t', t'' \in T: t' \leq t \leq t''}
\tne(t', t'')
+
3 \sqrt{\frac{\log \frac{48N}{\delta}}{2n}} 
=:
E_{\truenegative}(t)
\\ 
|\hat{\falsenegative}(t) - \falsenegative(t)|
&\leq 
3 \inf_{t', t'' \in T: t' \leq t \leq t''}
\fne(t', t'')
+
3 \sqrt{\frac{\log \frac{48N}{\delta}}{2n}}
=:
E_{\falsenegative}(t)\\
|\hat{\falsepositive}(t) - \falsepositive(t)|
&\leq 
3 \inf_{t', t'' \in T: t' \leq t \leq t''}
\fpe(t', t'')
+
3 \sqrt{\frac{\log \frac{48N}{\delta}}{2n}} 
=:
E_{\falsepositive}(t)
\\ 
|\hat{\truepositive}(t) - \truepositive(t)|
&\leq 
3 \inf_{t', t'' \in T: t' \leq t \leq t''}
\tpe(t', t'')
+
3 \sqrt{\frac{\log \frac{48N}{\delta}}{2n}}
=:
E_{\truepositive}(t)
\end{align*}
for all \(t\) in \([0, 1]\).
\label{prop:ConfusionMatrixBoundsBinary}
\end{proposition}

In the bounds, the first term should decrease with \(N\) for sufficiently nice distributions, and the second term increases in \(N\).
Additionally, since \(T\) is only used in the analysis, we may optimize the choice of cover, although the optimal cover depends on the unknown data distribution.
We additionally note that this is a stronger result than the multiclass result of Theorem~\ref{theorem:MulticlassConfusionMatrixBound} because the set \(T\) may include \(0\) and \(1\).
In contrast, Theorem~\ref{theorem:MulticlassConfusionMatrixBound} become vacuous when for weights that approach \(0\) in a coordinate.
Our analysis shows that this arises since in binary classification, the probability of class \(\eta(X)\) specifies the conditional probability of both \(Y = 0\) and \(Y = 1\).

%---------------------------------------------%
%---------------------------------------------%
\subsection{Precision, Recall, and F1 Score}
\label{app:subsec:BinaryClassification:PrecisionRecallF1}

In this section, we consider three general classification metrics commonly considered in imbalanced classification: precision, recall, and F1 score.
Similar bounds can be obtained for other general classification metrics that are sufficiently nice functions of the confusion matrix.

We start with precision.
Precision and empirical precision are defined by
\begin{align*}
\precision(t) 
&=
\frac{\truepositive(t)}{\truepositive(t) + \falsepositive(t)}
&
\hat{\precision}(t)
&=
\frac{\hat{\truepositive}(t)}{\hat{\truepositive}(t) + \hat{\falsepositive}(t)}.
\end{align*}
Now, we can consider a uniform convergence result.
\begin{corollary}
Define an error term \(E_{\precision}(t)\) by
\[
E_{\precision}(t)
=
3
\frac{E_{\truepositive}(t) + E_{\falsepositive}(t) }{\truepositive(t) + \falsepositive(t) - E_{\truepositive}(t) - E_{\falsepositive}(t)}.
\]
Suppose that Assumption~\ref{assumption:UniformEstimateBinary} holds.
Then, with probability at least \(1 - \delta\),
we have
\[
|\hat{\precision}(t) - \precision(t)|
\leq 
E_{\precision}(t)
\]
for every \(t\) in \([0, 1]\) for which \(E_{\precision} \geq 0\).
\label{corollary:BinaryClassificationPrecision}
\end{corollary}

We make the assumption that \(E_{\precision} \geq 0\) to simplify the algebra and presentation;
a more carefully stated bound could likely relax or modify this assumption.

%---------------------------------------------%
%---------------------------------------------%
Next, we turn to recall. 
Recall and empirical recall are defined as
\begin{align*}
\recall(t)
&=
\frac{\truepositive(t)}{\truepositive(t) + \falsenegative(t)} 
&
\hat{\recall}(t)
&=
\frac{\hat{\truepositive}(t)}{\hat{\truepositive}(t) + \hat{\falsenegative}(t)}.
\end{align*}
Similarly, we can prove a uniform convergence result.

\begin{corollary}
Define the error term \(E_{\recall}(t)\) by
\[
E_{\recall}(t)
=
3
\frac{E_{\truepositive}(t) + E_{\falsenegative}(t) }{\truepositive(t) + \falsenegative(t) - E_{\truepositive}(t) - E_{\falsenegative}(t)}.
\]
Suppose that Assumption~\ref{assumption:UniformEstimateBinary} holds.
Then with probability at least \(1 - \delta\), we have
\[
|\hat{\recall}(t) - \recall(t)|
\leq 
E_{\recall}(t)
\]
for all \(t\) for which \(E_{\recall}(t) \geq 0\).
\label{corollary:BinaryClassificationRecall}
\end{corollary}

%---------------------------------------------%
%---------------------------------------------%
Finally, we consider F1 score;
the F1 and empirical F1 scores are defined by
\begin{align*}
\fone(t)
&=
2 \frac{\precision(t) \cdot \recall(t)}{\precision(t) + \recall(t)}
&
\hat{\fone}(t) 
&=
2 \frac{\hat{\precision}(t) \cdot \hat{\recall}(t)}{\hat{\precision(t)} + \hat{\recall}(t)}.
\end{align*}
We now state a similar convergence result.

\begin{corollary}
Define the error bound
\[
E_{\fone}(t)
= 
9
\frac{E_{\precision}(t) + E_{\recall}(t)}{\precision(t) + \recall(t) - E_{\precision}(t) - E_{\recall}(t)}.
\]
Suppose that Assumption~\ref{assumption:UniformEstimateBinary} holds.
Then with probability at least \(1 - \delta\), we have
\[
|\hat{\fone}(t) - \fone(t)|
\leq 
E_{\fone}(t)
\]
for all \(t\) such that \(E_{\fone}(t) \geq 0\).
\label{corollary:BinaryClassificationF1}
\end{corollary}

%---------------------------------------------%
%---------------------------------------------%
% I commented the MCC section out for the moment because I'm not sure I actually want to do the algebra for this.
\begin{comment}
\subsubsection{Matthew's Correlation Coefficient}

Finally, we consider Matthew's correlation coefficient.
Matthew's correlation coefficient and its empirical counterpart are defined by
\begin{align*}
\mcc(t)
&=
\frac{\truepositive(t) \truenegative(t) - \falsepositive(t) \falsenegative(t)}{\sqrt{(\truepositive(t) + \falsepositive(t)) (\truepositive(t) + \falsenegative(t)) (\truenegative(t) + \falsepositive(t))(\truenegative(t) + \falsenegative(t))}} \\
\hat{\mcc}(t)
&=
\frac{\hat{\truepositive}(t) \hat{\truenegative}(t) - \hat{\falsepositive}(t) \hat{\falsenegative}(t)}{\sqrt{(\hat{\truepositive}(t) + \hat{\falsepositive}(t)) (\hat{\truepositive}(t) + \hat{\falsenegative}(t)) (\hat{\truenegative}(t) + \hat{\falsepositive}(t))(\hat{\truenegative}(t) + \hat{\falsenegative}(t))}}.
\end{align*}

We have the following convergence result.

\begin{corollary}

\label{corollary:BinaryClassificationMatthews}
\end{corollary}
\end{comment}

%---------------------------------------------%
%---------------------------------------------%
\subsection{Proofs}
\label{app:subsec:binaryCM:proofs}

In this section, we prove Proposition~\ref{prop:ConfusionMatrixBoundsBinary} and the ensuing corollaries.
For the former, we rely on decomposing the difference between the confusion matrix and its empirical counterpart via a discretization and an intermediate confusion matrix.
We define the intermediate confusion matrix to be
\begin{align}
&\begin{aligned}
\widetilde{\truenegative}(t) 
= \frac{1}{n} \sum_{i = 1}^{n} \ind\left\{Y_{i} \neq 1\right\} \ind\left\{\eta(X_{i}) < t\right\}  
&\qquad 
\widetilde{\falsenegative}(t) 
= \frac{1}{n} \sum_{i = 1}^{n} \ind\left\{Y_{i} = 1\right\} \ind\left\{\eta(X_{i}) < t\right\}\\
\widetilde{\falsepositive}(t) 
= \frac{1}{n} \sum_{i = 1}^{n} \ind\left\{Y_{i} \neq 1\right\} \ind\left\{\eta(X_{i}) \geq t\right\}
&\qquad 
\widetilde{\truepositive}(t) 
= \frac{1}{n} \sum_{i = 1}^{n} \ind\left\{Y_{i} = 1\right\} \ind\left\{\eta(X_{i}) \geq t\right\}.
\end{aligned}
\end{align}
These terms are just what the true regression function would obtain on the observed data under plug-in estimation.

For the remainder of the proof of Proposition~\ref{prop:ConfusionMatrixBoundsBinary}, we only consider the true negatives entry of the confusion matrix; the proofs for the other three entries are analogous.
So, we consider our decomposition lemma.
\begin{lemma}
Let \(t\) and \(t'\) be any elements of \([0, 1]\).
We have the inequality
\begin{align*}
|\hat{\truenegative}(t) - \truenegative(t)|
&\leq
\underbrace{|\widehat{\truenegative}(t') - \widetilde{\truenegative}(t')|}_{A(t')}
+ 
\underbrace{|\widetilde{\truenegative}(t') - \truenegative(t')|}_{B(t')}
\\ & \qquad +
\underbrace{|\truenegative(t) - \truenegative(t')|}_{C(t, t')} 
+
\underbrace{|\widehat{\truenegative}(t) - \widehat{\truenegative}(t')|}_{D(t, t')}.
\end{align*}
\label{lemma:ConfusionMatrixDecomposition}
\end{lemma}

The benefits of this decomposition are as follows.
First, \(t'\) should be an element of a finite cover of \(t'\).
Then, on the right hand side, \(A\) and \(B\) may be bounded by Hoeffding's inequality.
Importantly, since we are choosing a finite cover, we can use a union bound to guarantee this is small for all \(t'\) in the cover.
Next, \(C\) and \(D\) on the right hand side are approximation terms.
These rely on a distributional assumption that a fine enough partition of the interval \([0, 1]\) leads to small approximation errors, which is reasonable in most circumstances.

Before continuing on to the estimation and approximation errors, we define an event that simplifies the exposition.
Let \(A_{\uniform}\) be the event
\[
A_{\uniform}
=
\left\{
\|\hat{\eta} - \eta\| \leq \epsilon
\right\}.
\]
We shall use this frequently.

%---------------------------------------------%
%---------------------------------------------%
\subsubsection{Estimation Errors}

In this section, we bound \(A\) and \(B\) of Lemma~\ref{lemma:ConfusionMatrixDecomposition}.
We state the following bound for true negatives.
\begin{lemma}
Define \(p(t, \epsilon)\) by
\begin{align*}
p(t, \epsilon)
&= 
\prob\left(Y \neq 1, \;
t - \epsilon \leq \eta(X) \leq t + \epsilon
\right).
\end{align*}
Define the event 
\[
A_{t}
=
\left\{
\frac{1}{n}S
\leq 
p(t, \epsilon) + \sqrt{\frac{\log \frac{2}{\delta}}{2n}}
\right\},
\]
where \(S\) is a binomial random variable that we specify in the proof.
Suppose that \(A_{t}\) and \(A_{\uniform}\) hold.
Then, we have the bound
\[
A(t)
\leq 
p(t, \epsilon) + \sqrt{\frac{\log\frac{2}{\delta}}{2n}}.
\]
Moreover, we have \(\prob(A_{t}) \geq 1 - \delta\).
\label{lemma:BinomialBoundA}
\end{lemma}

%---------------------------------------------%
%---------------------------------------------%
\begin{proof}
First, we note that
\[
A(t)
= 
\max\left\{
\widehat{\truenegative}(t) - \widetilde{\truenegative}(t), \;
-\left(\widehat{\truenegative}(t) - \widetilde{\truenegative}(t)\right)
\right\},
\]
and so it suffices to analyze each of these terms separately.
Under \(A_{\uniform}\), we have
\begin{align*}
\widehat{\truenegative}(t) - \widetilde{\truenegative}(t) 
&=
\frac{1}{n} \sum_{i = 1}^{n} \ind\left\{Y_{i} \neq 1\right\} 
\left(\ind\left\{\hat{\eta}(X_{i}) < t \right\} - \ind\left\{\eta (X_{i}) < t\right\}\right) \\
&\leq 
\frac{1}{n} \sum_{i = 1}^{n} \ind\left\{Y_{i} \neq 1\right\} 
\left(\ind\left\{\eta(X_{i}) < t + \epsilon \right\} - \ind\left\{\eta(X_{i}) < t\right\}\right)
\\
&\leq 
\frac{1}{n} \sum_{i = 1}^{n} \ind\left\{Y_{i} \neq 1\right\} 
\ind\left\{t - \epsilon \leq \eta(X_{i}) < t + \epsilon \right\} \\
&=
\frac{1}{n} S.
\end{align*}
The other bound is nearly identical.
We have
\begin{align*}
-\left(\widehat{\truenegative}(t) - \widetilde{\truenegative}(t) \right)
&=
\frac{1}{n} \sum_{i = 1}^{n} \ind\left\{Y_{i} \neq 1\right\} 
\left(\ind\left\{\eta (X_{i}) < t\right\} - \ind\left\{\hat{\eta}(X_{i}) < t \right\}\right) \\
&\leq 
\frac{1}{n} \sum_{i = 1}^{n} \ind\left\{Y_{i} \neq 1\right\} 
\left(\ind\left\{\eta(X_{i}) < t \right\} - \ind\left\{\eta_{c} (X_{i}) < t - \epsilon\right\}\right)
\\
&\leq 
\frac{1}{n} \sum_{i = 1}^{n} \ind\left\{Y_{i} \neq 1\right\} 
\ind\left\{t - \epsilon \leq \eta(X_{i}) < t + \epsilon\right\} \\
&=
\frac{1}{n} S.
\end{align*}

Under \(A_{c, t}\), we have the desired bound on \(A(t)\).
Additionally, the sum \(S\) is a \(\binomial(n, p(t, \epsilon))\) random variable; so applying Hoeffding's inequality shows that
\(\prob\left(A_{c, t}\right) \geq 1 - \delta\).
\end{proof}
%---------------------------------------------%
%---------------------------------------------%
Now, we turn to \(B\) of Lemma~\ref{lemma:ConfusionMatrixDecomposition}.

\begin{lemma}
Let \(B_{t}\) be the event
\[
B_{t}
=
\left\{
B(t)
\leq 
\sqrt{\frac{\log \frac{2}{\delta}}{2n}}
\right\}.
\]
Then, we have \(\prob\left(B_{t}\right) \geq 1 - \delta\).
\label{lemma:BinomialBoundB}
\end{lemma}

\begin{proof}
Recall that \(B(t) = |\widetilde{\truenegative}(t) - \truenegative(t)|\).
We simply observe that
\(n \times \widetilde{\truenegative}(t)\) is a binomial random variable with parameters \(n\) and 
\[
p(t) 
=
\prob\left(Y \neq 1, \; \eta(X) < t\right) 
=
\truenegative(t).
\]
Applying Hoeffding's inequality completes the proof.
\end{proof}

%---------------------------------------------%
%---------------------------------------------%
\subsubsection{Approximation Errors}
In this section, we bound \(C\) and \(D\) of Lemma~\ref{lemma:ConfusionMatrixDecomposition}.

\begin{lemma}
Let \(t'\) and \(t''\) be cover elements such that \(t' \leq t \leq t''\). 
Then, we have
\[
\max\left\{C(t, t'), \; C(t, t'') \right\}
\leq 
\prob\left(Y \neq 1, \;t' \leq \eta(X) < t''\right).
\]
\label{lemma:ApproximationCBound}
\end{lemma}

\begin{proof}
The proof consists of straightforward manipulations.
First, we note that
\begin{align*}
C(t, t')
&\leq 
\truenegative(t) - \truenegative(t') 
\leq 
\truenegative(t'') - \truenegative(t')
\end{align*}
and
\begin{align*}
C(t, t'')
&\leq 
\truenegative(t'') - \truenegative(t) 
\leq 
\truenegative(t'') - \truenegative(t'). 
\end{align*}
Thus, it suffices to analyze \(\truenegative(t'') - \truenegative(t')\).
We have
\begin{align*}
\truenegative(t'') - \truenegative(t)
&=
\expect
\ind\left\{Y \neq 1, \; \eta < t''\right\}
-
\expect
\ind\left\{Y \neq 1, \; \eta < t\right\} \\
&=
\expect
\ind\left\{Y \neq 1, \; t \leq \eta < t''\right\} \\
&=
\prob\left(
Y \neq 1, \; t \leq \eta < t''
\right) \\
&\leq 
\prob\left(
Y \neq 1, \; t' \leq \eta < t''
\right),
\end{align*}
which completes the proof.
\end{proof}
%---------------------------------------------%
%---------------------------------------------%
The next lemma concerns the approximation error for the empirical quantities.
The only trick here is that we need to use Assumption~\ref{assumption:UniformEstimate} to convert the empirical regression function back to the true regression function.

\begin{lemma}
Let \(t'\) and \(t''\) be cover elements such that \(t' \leq t \leq t''\).
Define
\[
p(t', t'', \epsilon)
=
\prob\left(Y \neq 1, \; t' - \epsilon \leq \eta(X) \leq t'' + \epsilon\right)
\]
and the random variable
\[
S 
= 
\sum_{i = 1}^{n} 
\ind\left\{Y_i \neq 1, \;
t' - \epsilon \leq \eta(X_{i}) < t'' + \epsilon\right\}
\]
Define the event
\[
D_{t', t''}
=
\left\{
\frac{1}{n}S
\leq 
p(t', t'', \epsilon)
+ 
\sqrt{\frac{\log \frac{1}{\delta}}{2n}}
\right\}.
\]
If event \(A_{\uniform}\) and \(D_{t', t''}\) hold.
then, we have
\[
\widehat{\truenegative}(t'') - \widehat{\truenegative}(t)
\leq 
p(t', t'', \epsilon)
+
\sqrt{\frac{\log \frac{1}{\delta}}{2n}}.
\]
Additionally, we have \(\prob\left(D_{t', t''}\right) \geq 1 - \delta\).
\label{lemma:ApproximationDBound}
\end{lemma}

\begin{proof}
We start by performing simple manipulations and using event \(A_{\uniform}\) to obtain
\begin{align*}
\widehat{\truenegative}(t'') - \widehat{\truenegative}(t)
&\leq
\widehat{\truenegative}(t'') - \widehat{\truenegative}(t') \\
&=
\frac{1}{n} \sum_{i = 1}^{n} \ind\left\{Y_i \neq 1\right\} 
\left(\ind\left\{\hat{\eta}(X_{i}) < t''\right\}
-
\ind\left\{\hat{\eta}(X_{i}) < t'\right\}
\right) \\
&=
\frac{1}{n} \sum_{i = 1}^{n} \ind\left\{Y_i \neq 1\right\} 
\ind\left\{t' \leq \hat{\eta}(X_{i}) < t''\right\} \\
&\leq 
\frac{1}{n} \sum_{i = 1}^{n} 
\ind\left\{Y_i \neq 1, \;
t' - \epsilon \leq \eta(X_{i}) < t'' + \epsilon\right\} \\
&=
\frac{1}{n} S.
\end{align*}
We observe that \(S\) is a binomial random variable with \(n\) trials and success probability 
\(p(t', t'', \epsilon)\).
Thus, Hoeffding's inequality completes the proof.
\end{proof}
%---------------------------------------------%
%---------------------------------------------%
\subsection{Proof of Proposition~\ref{prop:ConfusionMatrixBoundsBinary}}

We apply Lemma~\ref{lemma:ConfusionMatrixDecomposition} to obtain
\begin{align*}
|\hat{\truenegative}(t) - \truenegative(t)|
&\leq 
A(t') + B(t') + \inf_{t' \in T} \left\{C(t, t') + D(t, t')\right\}.
\end{align*}

Thus, we turn to using our lemmas.
By assumption, we have \(\prob\left(A_{\uniform}\right) \geq 1 - \delta / 2\); so we need to set our error probabilities for the events \(A_{t'}\), \(B_{t'}\), and \(D_{t', t''}\) accordingly.
We pick \(\delta'' = \delta / (24 N)\) because we need to bound the error probabilities for the prior three events for each \(t'\) in \(T\), since \(t''\) in \(D_{t', t''}\) is fixed given \(t'\), and each of the four quantities of the confusion matrix.

Applying
 Lemma~\ref{lemma:BinomialBoundA} for \(A_{t'}\) for each \(t'\) in \(T\),
Lemma~\ref{lemma:BinomialBoundB} for \(B_{t'}\) for each \(t'\) in \(T\),
Lemma~\ref{lemma:ApproximationCBound}, and Lemma~\ref{lemma:ApproximationDBound} for \(D_{t', t''}\) each pair of consecutive elements \((t', t'')\) in \(T\) with error probability \(\delta''\) in every case, 
we have
\begin{align*}
|\hat{\truenegative}(t) - \truenegative(t)|
&\leq 
\inf_{t', t'' \in T: t' \leq t \leq t''}\left\{
p(t', \epsilon) 
+
p(t', t'', \epsilon)
+
\prob\left(Y \neq 1, \; t' \leq \eta(X) \leq t''\right)
\right\}
\\ &\qquad +
3 \sqrt{\frac{\log \frac{16 N}{\delta}}{2n}}.
\end{align*}
Note that each term of the infimum is less than or equal to 
\[
p(t', t'', \epsilon)
=
\prob\left(Y \neq 1, \; t' - \epsilon \leq \eta(X) \leq t'' + \epsilon\right),
\]
and we have
\[
|\hat{\truenegative}(t) - \truenegative(t)|
\leq 
3 p(t', t'', \epsilon)
+
3 \sqrt{\frac{\log \frac{48 N}{\delta}}{2n}}.
\]
This result holds with probability at least \(1 - \delta\), and moreover the analogous statements hold for all other confusion matrix quantities and all \(t\) in \([0, 1]\). 
This completes the proof.
\hfill
\qedsymbol

%---------------------------------------------%
%---------------------------------------------%
\subsection{Proofs of Corollaries}

In this section, we prove our corollaries for precision, recall, and F1 score. %, and Matthew's correlation coefficient.
%---------------------------------------------%
%---------------------------------------------%
% Precision
%---------------------------------------------%
\begin{lemma}
Let \(A_{\precision}\) be the event 
\[
A_{\precision}
= 
\left\{|\hat{\truepositive}(t) - \truepositive(t)| \leq E_{\truepositive}(t), \; 
|\hat{\falsepositive}(t) - \falsepositive(t)| \leq E_{\falsepositive}(t) 
\text{ for all } t \in [0, 1]
\right\}.
\]
Then, on the event \(A_{\precision}\), we have
\[
|\hat{\precision}(t) - \precision(t)|
\leq 
E_{\precision}(t) .
\]
\label{lemma:BinaryClassificationPrecision}
\end{lemma}
%---------------------------------------------%
%---------------------------------------------%
\begin{proof}
We provide upper and lower bounds under the event \(A_{\precision}\).
For the upper bound, we have
\begin{align*}
\hat{\precision}(t)
&\leq 
\frac{\truepositive(t) + E_{\truepositive}(t)}{\truepositive(t) + \falsepositive(t) - E_{\truepositive}(t) - E_{\falsepositive}(t)} \\
&=
\left(\precision(t) + 
\frac{E_{\truepositive}(t)}{\truepositive(t) + \falsepositive(t)}
\right)
\left(
1 
+ 
\frac{E_{\truepositive}(t) + E_{\falsepositive}(t)}{\truepositive(t) + \falsepositive(t) - E_{\truepositive}(t) - E_{\falsepositive}(t)}
\right) \\
&\leq 
\precision(t) + 2 B_{\precision}(t) + B_{\precision}(t)^{2},
\end{align*}
where
\[
B_{\precision}(t)
= 
\frac{E_{\truepositive}(t) + E_{\falsepositive}(t)}{\truepositive(t) + \falsepositive(t) - E_{\truepositive}(t) - E_{\falsepositive}(t)}.
\]
First, note that \(B_{\precision} \geq 0\) by assumption.
Next if \(B_{\precision}(t) \leq 1\), then we can obtain the bound
\[
\hat{\precision}(t)
\leq 
\precision(t)
+
3 B_{\precision}(t).
\]
If \(B_{\precision}(t) \geq 1\), then since empirical precision is bounded by \(1\), the previous inequality is also true, which proves the upper bound.

For the lower bound, we have
\begin{align*}
\hat{\precision}(t)
&\geq 
\frac{\truepositive(t) - E_{\truepositive}(t)}{\truepositive(t) + \falsepositive(t) + E_{\truepositive}(t) + E_{\falsepositive}(t)} \\
&=
\frac{\truepositive(t) - E_{\truepositive}(t)}{\truepositive(t) + \falsepositive(t)}
\cdot 
\frac{\truepositive(t) + \falsepositive(t)}{\truepositive(t) + \falsepositive(t) + E_{\truepositive}(t) + E_{\falsepositive}(t)} \\
&=
\left( \precision(t) 
- 
\frac{E_{\truepositive}(t)}{\truepositive(t) + \falsepositive(t)} \right)
\left(1 
- 
\frac{E_{\truepositive}(t) + E_{\falsepositive}(t)}{\truepositive(t) + \falsepositive(t) + E_{\truepositive}(t) + E_{\falsepositive}(t)}
\right) 
\\
&\geq
\precision(t)
- 
2B_{\precision}(t).
\end{align*}
This proves the lower bound and completes the proof.
\end{proof}

\begin{proof}[Proof of Corollary~\ref{corollary:BinaryClassificationPrecision}]
By Proposition~\ref{prop:ConfusionMatrixBoundsBinary}, the event \(A_{\recall}\) occurs with probability at least \(1 - \delta\).
The inequality of Lemma~\ref{lemma:BinaryClassificationPrecision} completes the proof.
\end{proof}

%---------------------------------------------%
%---------------------------------------------%
% Recall
%---------------------------------------------%
%---------------------------------------------%
Now, we turn to recall. 
\begin{lemma}
Let \(A_{\recall}\) be the event
\[
A_{\recall}
=
\left\{
|\hat{\truepositive}(t) - \truepositive(t)| \leq E_{\truepositive}(t), \;
|\hat{\falsenegative}(t) - \falsenegative(t)| \leq 
E_{\falsenegative}(t) 
\text{ for all } t \in [0, 1]
\right\}.
\]
Define the error term \(E_{\recall}(t)\) by
\[
E_{\recall}(t)
=
3
\frac{E_{\truepositive}(t) + E_{\falsenegative}(t) }{\truepositive(t) + \falsenegative(t) - E_{\truepositive}(t) - E_{\falsenegative}(t)}.
\]
Then, on the event \(A_{\recall}\), we have
\[
|\hat{\recall}(t) - \recall(t)|
\leq 
E_{\recall}(t).
\]
\label{lemma:BinaryClassificationRecall}
\end{lemma}
%---------------------------------------------%
%---------------------------------------------%
Since the proof is entirely analogous to that of Lemma~\ref{lemma:BinaryClassificationPrecision}, we omit it and skip straight to the proof of Corollary~\ref{corollary:BinaryClassificationRecall}.

\begin{proof}[Proof of Corollary~\ref{corollary:BinaryClassificationRecall}]
By Proposition~\ref{prop:ConfusionMatrixBoundsBinary}, the event \(A_{\recall}\) occurs with probability at least \(1 - \delta\).
The inequality of Lemma~\ref{lemma:BinaryClassificationRecall} completes the proof.
\end{proof}

%---------------------------------------------%
%---------------------------------------------%
% F1 Score
%---------------------------------------------%
Next, we prove our corollary for F1 score.
The proof only differs modestly from the proofs for precision and recall.
As before, we start with an approximation lemma.
\begin{lemma}
Define the event 
\[
A_{\fone}
=
\left\{
|\hat{\precision}(t) - \precision(t)| \leq E_{\precision}(t), \;
|\hat{\recall}(t) - \recall(t)| \leq E_{\recall}(t) 
\text{ for all } t \in [0, 1]
\right\}.
\]
Then, on the event \(A_{\fone}\), we have
\[
|\hat{\fone}(t) - \fone(t)|
\leq 
E_{\fone}(t).
\]
\label{lemma:BinaryClassificationF1}
\end{lemma}

\begin{proof}
We upper bound and lower bound the empirical F1 score under \(A_{\fone}\).
Starting with the former, we have
\begin{align*}
\hat{\fone}(t)
&\leq 
2 
\frac{(\precision(t) + E_{\precision}(t)) (\recall(t) + E_{\recall}(t))}{\precision(t) + \recall(t) - E_{\precision}(t) - E_{\recall}(t)} \\ 
&=
\left(
\fone(t) + 4 \frac{E_{\precision}(t) + E_{\recall}(t)}{\precision(t) + \recall(t)} 
\right)
% \\ & \qquad \times 
\left(
1 
+
\frac{E_{\precision}(t) + E_{\recall}(t)}{\precision(t) + \recall(t) - E_{\precision}(t) - E_{\recall}(t)}
\right) \\ 
&\leq 
\fone(t) + 5 B_{\fone}(t) + 4B_{\fone}(t)^{2}
\end{align*}
where 
\[
B_{\fone}(t)
=
\frac{E_{\precision}(t) + E_{\recall}(t)}{\precision(t) + \recall(t) - E_{\precision}(t) - E_{\recall}(t)}.
\]
If we have \(B_{\fone}(t) \leq 1\), then we can write 
\[
\hat{\fone}(t)
\leq 
\fone(t) + 9 B_{\fone}(t)
=
\fone(t) + E_{\fone}(t).
\]
If \(B_{\fone}(t) \geq 1\), then the previous upper bound still holds since \(\fone(t)\) is bounded by \(1\), which completes the upper bound

For the lower bound, we have
\begin{align*}
\hat{\fone}(t)
&\geq 
2 
\frac{(\precision(t) - E_{\precision}(t))(\recall(t) - E_{\recall}(t))}{\precision(t) + \recall(t) + E_{\precision}(t) + E_{\recall}(t)} \\
&=
\left(
\fone(t)
-
2 \frac{E_{\precision}(t) + E_{\recall}(t)}{\precision(t) + \recall(t)}
\right)
\left(
1 
-
\frac{E_{\precision}(t) + E_{\recall}(t)}{\precision(t) + \recall(t) + E_{precision}(t) + E_{\recall}(t)}
\right) \\ 
&\geq 
\fone(t)
- 
3 B_{\fone}(t).
\end{align*}
This completes the proof of the lemma.
\end{proof}

\begin{proof}[Proof of Corollary~\ref{corollary:BinaryClassificationF1}]
By Proposition~\ref{prop:ConfusionMatrixBoundsBinary}, event \(A_{\fone}\) occurs with probability at least \(1 - \delta\).
Thus, applying Lemma~\ref{lemma:BinaryClassificationF1} completes the proof.
\end{proof}

%---------------------------------------------%
% Matthew's Correlation Coefficient
%---------------------------------------------%
%Finally, we prove our results for Matthew's correlation coefficient.
%\begin{lemma}

%\label{lemma:BinaryClassificationMatthews}
%\end{lemma}

%\begin{proof}

%\end{proof}

%\begin{proof}[Proof of Corollary~\ref{corollary:BinaryClassificationMatthews}]

%\end{proof}

%---------------------------------------------%
%---------------------------------------------%
\section{Proofs for Multiclass Confusion Matrix Quantities}

The main goal of this section is to prove Theorem~\ref{theorem:MulticlassConfusionMatrixBound}.
Since all confusion matrix entries behave similarly, we only prove the results for true negatives; the results for the other entries are proved nearly identically.

Our goal is to use a decomposition of the difference between the empirical true negatives and the actual true negatives, along with our discretization of the space of weights.
We start with the decomposition; for brevity, we define
\begin{align}
& \begin{aligned}
M(q, c, x)
&= \max_{j \neq c} q_{j} \eta_{j}(x) \\
\hat{M}(q, c, x)
&= \max_{j \neq c} q_{j} \hat{\eta}_{j}(x).
\label{eqn:MandMhat}
\end{aligned}
\end{align}
Next, we define the intermediate confusion matrix quantities
\begin{align}
&\begin{aligned}
\widetilde{\truenegative}_{c}(q) 
&= \frac{1}{n} \sum_{i = 1}^{n} \ind\left\{Y_{i} \neq c\right\} 
\ind\left\{q_{c}\eta_{c}(X_{i}) < M(q, c, x)\right\}  \\
\widetilde{\falsenegative}_{c}(q) 
&= \frac{1}{n} \sum_{i = 1}^{n} \ind\left\{Y_{i} = c\right\} 
\ind\left\{q_{c}\eta_{c}(X_{i}) < M(q, c, x)\right\}  \\
\widetilde{\falsepositive}_{c}(q) 
&= \frac{1}{n} \sum_{i = 1}^{n} \ind\left\{Y_{i} \neq c\right\} 
\ind\left\{q_{c}\eta_{c}(X_{i}) \geq M(q, c, x)\right\}  \\
\widetilde{\truepositive}_{c}(q) 
&= \frac{1}{n} \sum_{i = 1}^{n} \ind\left\{Y_{i} = c\right\} 
\ind\left\{q_{c}\eta_{c}(X_{i}) \geq M(q, c, x)\right\}  .
\label{eqn:EmpiricalCMOfBayes}
\end{aligned}
\end{align}

In what follows, we frequently use the event
\begin{equation}
A_{\uniform}
=
\left\{
\|\hat{\eta}_c - \eta_c \| \leq \epsilon
\right\}.
\label{eqn:MulticlassProofs:UniformEvent}
\end{equation}

%---------------------------------------------%
%---------------------------------------------%
\subsection{A Decomposition Lemma}

Now, we can state our decomposition lemma.
\begin{lemma}
Let \(q\) and \(q'\) be any elements of \(Q\).
We have the inequality
\begin{align*}
|\hat{\truenegative}_{c}(q) - \truenegative_{c}(q)|
&\leq
\underbrace{|\widehat{\truenegative}_{c}(q') - \widetilde{\truenegative}_{c}(q')|}_{A(q')}
+ 
\underbrace{|\widetilde{\truenegative}_{c}(q') - \truenegative_{c}(q')|}_{B(q')}
\\ & \qquad +
\underbrace{|\truenegative_{c}(q) - \truenegative_{c}(q')|}_{C(q, q')} 
+
\underbrace{|\widehat{\truenegative}_{c}(q) - \widehat{\truenegative}_{c}(q')|}_{D(q, q')}.
\end{align*}
\label{lemma:ConfusionMatrixDecompositionMulticlass}
\end{lemma}

The proof is a straightforward application of the triangle inequality.

%---------------------------------------------%
%---------------------------------------------%
\subsection{Estimation Errors}

In this section, we consider bounds on \(A(q)\) and \(B(q)\).
The results are similar to those in the case of binary classification.

\begin{lemma}
Define the probability
\[
p(q, c)
=
\prob\left(
Y \neq c, \;
t(q, c, X) - \epsilon r(q, c) 
\leq 
\eta_{c}(X)
<
t(q, c, X) + \epsilon r(q, c)
\right).
\]
Define the event
\[
A_{q, c} 
=
\left\{
\frac{1}{n} S 
\leq 
p(q, c) + \sqrt{\frac{\log \frac{1}{\delta}}{2n}}
\right\}
\]
where \(S\) is a binomial random variable that we specify in the proof.
If \(A_{\uniform}\) and \(A_{q, c}\) hold, 
then, we have
\[
A(q) %\hat{\truenegative}_{c}(q) - \Tilde{\truenegative}_{c}(q)
\leq 
p(q, c)
+
\sqrt{\frac{\log \frac{1}{\delta}}{2n}}.
\]
Moreover, we have \(\prob\left(A_{q, c}\right) \geq 1 - \delta\).
\label{lemma:MulticlassBoundTermA}
\end{lemma}

%---------------------------------------------%
%---------------------------------------------%
\begin{proof}
We note that
\[
A(q) 
=
\max\left\{
\hat{\truenegative}_{c}(q) - \Tilde{\truenegative}_{c}(q),
\; -\left(\hat{\truenegative}_{c}(q) - \Tilde{\truenegative}_{c}(q)\right)
\right\},
\]
and so it suffices to bound \(\hat{\truenegative}_{c}(q) - \Tilde{\truenegative}_{c}(q)\) and \(-\left(\hat{\truenegative}_{c}(q) - \Tilde{\truenegative}_{c}(q)\right)\) by \(S\) separately.
We start with the former and obtain
\begin{align*}
\hat{\truenegative}_{c}(q) - \Tilde{\truenegative}_{c}(q)
&= 
\frac{1}{n} \sum_{i = 1}^{n} \ind\left\{Y_{i} \neq c\right\}
\left(
\ind\left\{
q_{c} \hat{\eta}_{c}(X) < \hat{M}(q, c, X_{i})
\right\}
-
\ind\left\{
q_{c} \hat{\eta}_{c}(X) < M(q, c, X_{i})
\right\}
\right) \\ 
&\leq
\frac{1}{n} \sum_{i = 1}^{n} \ind\left\{Y_{i} \neq c\right\}
\ind\left\{
M(q, c, X_{i}) \leq q_{c} \hat{\eta}_{c}(X) < \hat{M}(q, c, X_{i})
\right\}. 
\end{align*}
Under event \(A_{\uniform}\), we have
\begin{align*}
\hat{\truenegative}_{c}(q) - \Tilde{\truenegative}_{c}(q)
&\leq 
\frac{1}{n} \sum_{i = 1}^{n} \ind\left\{Y_{i} \neq c\right\}
\ind\left\{
M(q, c, X_{i}) - \epsilon q_{c} \leq q_{c} \eta_{c}(X) < M(q, c, X_{i}) +  \epsilon (q_{\max} + q_{c})
\right\} \\
&\leq 
\frac{1}{n} \sum_{i = 1}^{n} \ind\left\{Y_{i} \neq c\right\}
\\ &\qquad \times 
\ind\left\{
t(q, c, X_{i}) - \epsilon \left(\frac{q_{\max}}{q_{c}} + 1\right) \epsilon \leq  \eta_{c}(X) < t(q, c, X_{i}) + \epsilon \left(\frac{q_{\max}}{q_{c}} + 1\right)
\right\} \\
&=:
\frac{1}{n} S.
\end{align*}

We now consider the second term in the maximum defining \(A(q)\).
The analysis is essentially unchanged; we have
\begin{align*}
-\left(\hat{\truenegative}_{c}(q) - \Tilde{\truenegative}_{c}(q)\right)
&=
\frac{1}{n} \sum_{i = 1}^{n} \ind\left\{Y_{i} \neq c\right\} 
\\ &\qquad \times
\left(
\ind\left\{
q_{c} \hat{\eta}_{c}(X) < M(q, c, X_{i})
\right\}
-
\ind\left\{
q_{c} \hat{\eta}_{c}(X) < \hat{M}(q, c, X_{i})
\right\}
\right) \\ 
&\leq 
\frac{1}{n} \sum_{i = 1}^{n} \ind\left\{Y_{i} \neq c\right\}
\ind\left\{
\hat{M}(q, c, X_{i}) \leq q_{c} \hat{\eta}_{c}(X) < M(q, c, X_{i})
\right\}.
\end{align*}
Under event \(A_{\uniform}\), we obtain
\begin{align*}
-\left(\hat{\truenegative}_{c}(q) - \Tilde{\truenegative}_{c}(q)\right)
&\leq 
\frac{1}{n} \sum_{i = 1}^{n} \ind\left\{Y_{i} \neq c\right\}
\\ &\qquad \times
\ind\left\{
M(q, c, X_{i}) - (q_{c} + q_{\max}) \epsilon 
\leq 
q_{c} \eta_{c}(X) 
< 
M(q, c, X_{i}) + q_{c} \epsilon
\right\} \\
&\leq 
\frac{1}{n} \sum_{i = 1}^{n} \ind\left\{Y_{i} \neq c\right\}
\\ &\qquad \times
\ind\left\{
t(q, c, X_{i}) - \epsilon r(q, c) 
\leq 
\eta_{c}(X) 
< 
t(q, c, X_{i}) + \epsilon r(q, c)
\right\} \\
&= 
\frac{1}{n} S.
\end{align*}

Thus, we see that \(A(q) \leq S / n\); so it suffices to bound \(S\).
Under \(A_{q}\), we have the desired bound
\begin{equation}
A(q)
\leq 
\frac{1}{n}S
\leq 
p(q, c) + \sqrt{\frac{\log\frac{1}{\delta}}{2n}}.
\label{eqn:MulticlassABinomial}
\end{equation}
Additionally, since \(S\) is a Binomial random variable with \(n\) trials and parameter \(p(c, q)\), applying Hoeffding's inequality proves \(\prob\left(A_{q}\right) \geq 1 - \delta\).
\end{proof}

%---------------------------------------------%
%---------------------------------------------%
Next, we bound \(B(q)\) of Lemma~\ref{lemma:ConfusionMatrixDecompositionMulticlass}.
\begin{lemma}
Let \(B_{q}\) be the event
\[
B_{q}
=
\left\{
B(q)
\leq 
\sqrt{\frac{\log\frac{2}{\delta}}{2n}}
\right\}.
\]
Then, we have \(\prob\left(B_{q}\right) \geq 1 - \delta\).
\label{lemma:MulticlassBoundBTerm}
\end{lemma}

\begin{proof}
We see that \(n \times \Tilde{\truenegative}_{c}(q)\) is a binomial random variable with \(n\) trials and parameter 
\[
p
=
\prob\left(Y \neq c, \; \eta_{c} \leq M(q, c, X) \right).
\]
Applying Hoeffding's inequality completes the proof.
\end{proof}

%---------------------------------------------%
%---------------------------------------------%
\subsection{Approximation Error}
Recall that we define \(t'(x)\) and \(t''(x)\) by
\begin{align*}
t'(x) 
&= 
\frac{1}{q'_{c}} M(q', c, x) 
&
t''(x)
&=
\frac{1}{q''_{c}} M(q'', c, x).
\end{align*}
Now, we bound term \(C(q, q')\).

\begin{lemma}
Suppose that \((q', q'')\) class \(c\)-covers \(q\).
Then, we have
\begin{align*}
\max\left\{|\truenegative_{c}(q) - \truenegative_{c}(q')|, \;
|\truenegative_{c}(q'') - \truenegative_{c}(q)|
\right\} 
&\leq 
\truenegative_{c}(q'') - \truenegative_{c}(q') .
\end{align*}
Furthermore, we have
\begin{align*}
\truenegative_{c}(q'') - \truenegative_{c}(q')
&=
\prob\left(
Y \neq c, t'(X) \leq \eta_{c}(X) \leq t''(X)
\right).
\end{align*}
\label{lemma:MulticlassBoundTermC}
\end{lemma}

%---------------------------------------------%
%---------------------------------------------%
\begin{proof}
Since \((q', q'')\) is a class \(c\)-cover of \(q\), we have 
\[
\frac{1}{q'_{c}} M(q', c, x) \leq \frac{1}{q_{c}}M(q, c, x) \leq \frac{1}{q''_{c}}M(q'', c, x)
\]
for all \(x\) in \(\xspace\).
Consequently, we obtain
\begin{align*}
|\truenegative_{c}(q) - \truenegative_{c}(q')|
&=
\left|
\expect \left[\ind\left\{Y \neq c, \; q_{c}\eta_{c}(X) < M(q, c, X)\right\}
-
\ind\left\{Y \neq c, \; q'_{c} \eta_{c}(X) < M(q', c, X)\right\}
\right]
\right| \\
&=
\expect \ind\left\{Y \neq c, \; \frac{1}{q'_{c}} M(q', c, X) \leq \eta_{c}(X) < \frac{1}{q_{c}}M(q, c, X)\right\} \\
&\leq 
\expect \ind\left\{Y \neq c, \; \frac{1}{q'_{c}} M(q', c, X) \leq \eta_{c}(X) < \frac{1}{q''_{c}}M(q'', c, X)\right\} \\
&= 
\prob\left(Y \neq c, \; \frac{1}{q'_{c}} M(q', c, X) \leq \eta_{c}(X) < \frac{1}{q''_{c}}M(q'', c, X)\right) \\
&=
\truenegative_{c}(q'') - \truenegative_{c}(q').
\end{align*}
A similar proof holds for \(|\truenegative_{c}(q'') - \truenegative_{c}(q)|\), which proves the lemma.
\end{proof}

%---------------------------------------------%
%---------------------------------------------%

\begin{lemma}
Define the probability 
\[
p(c, q', q'')
=
\prob\left(
Y \neq c, \; 
t'(X) - \epsilon\left(1 + \frac{q'_{\max}}{q'_{c}}\right) 
\leq 
\eta_{c}(X) 
\leq t''(X) + \epsilon\left(1 + \frac{q''_{\max}}{q''_{c}}\right)
\right),
\]
and the event
\[
D_{q', q''}
=
\left\{
\frac{1}{n} S
\leq 
p(c, q', q'') + \sqrt{\frac{\log \frac{1}{\delta}}{2n}}
\right\},
\]
where \(S\) is a binomial random variable that we define in the proof of the lemma.
We have the following results.
\begin{itemize}
    \item 
Suppose that \((q', q'')\) class \(c\)-covers \(q\).
Then, we have
\begin{align*}
\max\left\{|\hat{\truenegative}_{c}(q) - \hat{\truenegative}_{c}(q')|, \; 
|\hat{\truenegative}_{c}(q'') - \hat{\truenegative}_{c}(q)|
\right\} 
&\leq 
\hat{\truenegative}_{c}(q'') - \hat{\truenegative}_{c}(q') .
\end{align*}
    
    \item 
If \(A_{\uniform}\) and \(D_{q', q''}\) hold, then we have
\begin{align*}
\hat{\truenegative}_{c}(q'') - \hat{\truenegative}_{c}(q')
&\leq 
p(c, q', q'')
+
\sqrt{\frac{\log \frac{1}{\delta}}{2n}}.
\end{align*}
    
    \item 
Finally, we have \(\prob\left(D_{q', q''}\right) \geq 1 - \delta\).
\end{itemize}
\label{lemma:MulticlassBoundTermD}
\end{lemma}
%---------------------------------------------%
%---------------------------------------------%
\begin{proof}
First, we note that since \((q', q'')\) class \(c\)-covers \(q\), we have
\begin{align*}
|\hat{\truenegative}_{c}(q) - \hat{\truenegative}_{c}(q')|
&=
\frac{1}{n} \sum_{i = 1}^{n} 
\ind\left\{
Y_{i} \neq c
\right\}
\left(
\ind\left\{
q_{c} \hat{\eta}_{c}(X) < \hat{M}(q, c, X)
\right\}
-
\ind\left\{
q'_{c} \hat{\eta}_{c}(X) < \hat{M}(q', c, X)
\right\}
\right) \\
&\leq 
\frac{1}{n} \sum_{i = 1}^{n} 
\ind\left\{
Y_{i} \neq c
\right\}
\left(
\ind\left\{
q''_{c} \hat{\eta}_{c}(X) < \hat{M}(q'', c, X)
\right\}
-
\ind\left\{
q'_{c} \hat{\eta}_{c}(X) < \hat{M}(q', c, X)
\right\}
\right) \\ 
&\leq 
\frac{1}{n} \sum_{i = 1}^{n} 
\ind\left\{
Y_{i} \neq c
, \;
\hat{t}'(X) \leq \hat{\eta}_{c}(X) < \hat{t}''(X)
\right\} \\
&=
\hat{\truenegative}_{c}(q'') - \hat{\truenegative}_{c}(q').
\end{align*}
A similar analysis shows that \(|\hat{\truenegative}_{c}(q'') - \hat{\truenegative}_{c}(q)|\) is bounded by
\(\hat{\truenegative}_{c}(q'') - \hat{\truenegative}_{c}(q').\)
This proves the first claim of the lemma.

Under \(A_{\uniform}\), we have
\begin{align*}
\hat{\truenegative}_{c}(q'') &- \hat{\truenegative}_{c}(q') \\ &
\leq 
\frac{1}{n} \sum_{i = 1}^{n} \ind\left\{Y_{i} \neq c, \;
t'(X) - \epsilon\left(1 + \frac{q'_{\max}}{q'_{c}}\right)
\leq 
\eta_{c}(X) 
\leq 
t''(X) + \epsilon\left(1 + \frac{q''_{\max}}{q''_{c}}\right)
\right\} \\
&=:
\frac{1}{n} S.
\end{align*}
Thus, under \(D_{q', q''}\), we have
\[
\hat{\truenegative}_{c}(q'') - \hat{\truenegative}_{c}(q')
\leq 
p(c, q', q'') + \sqrt{\frac{\log \frac{1}{\delta}}{2n}}.
\]
This proves the second claim of the lemma.

Finally, since \(S\) is a binomial random variable with \(n\) trials and parameter \(p(c, q', q'')\), Hoeffding's inequality gives
\(\prob\left(D_{q', q''}\right) \geq 1 - \delta\), which completes the proof.
\end{proof}

%---------------------------------------------%
%---------------------------------------------%
\subsection{Proof of Theorem~\ref{theorem:MulticlassConfusionMatrixBound}}
We apply Lemma~\ref{lemma:ConfusionMatrixDecompositionMulticlass} to obtain
\[
|\hat{\truenegative}_{c}(q) - \truenegative_{c}(q)|
\leq 
A(q') + B(q') + C(q, q') + D(q, q').
\]
Subsequently, we examine each of these terms individually.
We apply Lemma~\ref{lemma:MulticlassBoundTermA}, Lemma~\ref{lemma:MulticlassBoundBTerm}, Lemma~\ref{lemma:MulticlassBoundTermC}, and
Lemma~\ref{lemma:MulticlassBoundTermD}.
For the events \(A_{q'}\), \(B_{q'}\), and \(D_{q', q''}\), we need to specify the error probabilities.

For \(A_{q'}\), we specify an error probability of \(\delta'' = \delta / (24 N C)\),
since we need to use the lemma for each \(q'\) in \(Q_{\cover}\), each \(c\) in \([C]\), and each entry of the confusion matrix.
Thus by Assumption~\ref{assumption:UniformEstimate} and Lemma~\ref{lemma:MulticlassBoundTermA},  we have
\[
A(q')
\leq 
p(q', c) + \sqrt{\frac{\log \frac{24NC}{\delta}}{2n}}
\leq 
p(q, q', c) + \sqrt{\frac{\log \frac{12 N C^{2}}{\delta}}{2n}}
\]
for all \(c\) and \(q'\) with probability at least \(1 - 2 \delta / 3.\)

For \(B_{q'}\), we specify an error probability of \(\delta'' = \delta / (24 N C)\), as well, since again we need to apply this lemma and its analogue for each \(q'\) in \(Q_{\cover}\), each \(c\) in \([C]\), and each entry of the confusion matrix.
Thus, by Assumption~\ref{assumption:UniformEstimate} and Lemma~\ref{lemma:MulticlassBoundBTerm}, we obtain
\[
B(q') 
\leq 
\sqrt{\frac{\log \frac{48 N C}{\delta}}{2n}}
\leq 
\sqrt{\frac{\log \frac{24 N C^{2}}{\delta}}{2n}}
\]
for all \(c\) and \(q'\) with probability at least \(1 - 2 \delta / 3\).

The term \(C(q, q')\) is the easiest since it is an almost-sure bound.
By Lemma~\ref{lemma:MulticlassBoundTermC}, we can bound \(C(q, q')\) by
\[
C(q, q')
\leq 
\prob\left(Y \neq c, \; 
t'(c, X) \leq \eta_{c}(X) \leq t''(c, X)
\right) 
\leq 
p(q', q', c).
\]

Finally, we turn to the event \(D_{q', q''}\).
Here we use the error probability \(\delta''' = \delta / (24 N C^{2})\) because we have to apply Lemma~\ref{lemma:MulticlassBoundTermD} for each element \(q'\) of \(Q_{\cover}\), each \(c\) in \([C]\), a different \(q''\) in \(Q_{\cover}\) for each \(q'\) and \(c\), and each of the confusion matrix quantities.
By Lemma~\ref{lemma:MulticlassBoundTermD}, we have the following bound for \(D(q, q')\):
\[
D(q, q')
\leq 
p(q', q'', c) + \sqrt{\frac{\log \frac{24 N C^{2}}{\delta}}{2n}}
\]
with probability at least \(1 - 2 \delta / 3\)
Putting everything together, we have 
\begin{align*}
|\hat{\truenegative}_{c}(q) - \truenegative_{c}(q)|
\leq 
3 p(q', q'', c) + 3\sqrt{\frac{\log \frac{24 N C^{2}}{\delta}}{2n}},
\end{align*}
with probability \(1 - \delta\) because we only need to count \(\prob\left(A_{\uniform}^{c}\right) \leq \delta / 2\) once in the final union bound.
Furthermore, this result holds simultaneously for all other confusion matrix quantities, which completes the proof.
\hfill
\qedsymbol
%---------------------------------------------%
%---------------------------------------------%
%\subsection{Proof of Corollary~\ref{corollary:knnMulticlassConfusionMatrixBound}}

%\qedsymbol

%---------------------------------------------%
%---------------------------------------------%
\subsection{Lower Bound Proof}

\begin{proof}[Proof of Proposition~\ref{prop:LowerBoundAlpha}]
We need to find two distributions \(P_{+}\) and \(P_{-}\) that are sufficiently close.
To this end, we use the canonical ``bump'' function
\[
g(x)
=
\begin{cases}
\exp\left(-\frac{1}{1 - \|x\|_{2}^{2}}\right) & \|x\|_{2} \leq 1 \\
0 & \text{otherwise.}
\end{cases}
\]
We then define the recentered version 
\[
h(x) 
= g\left(
Mx - \frac{1}{2}\ind 
\right),
\]
where in this case we use \(\ind\) to denote the all ones vector \((1, \ldots, 1)\) and the constant \(M\) is \(M = n^{1/(2\alpha + d)}\).
Now, we define the regression functions.
First, on \(x'\), we set \(\eta_{+}(x') = \eta_{-}(x') = 1\).
Next, we define the positive and negative regression functions on \([0, 1]^d\) by
\begin{align*}
\eta_{+}(x)
&=
\frac{1}{2} + \gamma 
+ \min\left\{\frac{1}{4}, \, \frac{L}{\|g\|_{\Sigma^{\alpha}}}\right\} M^{-\alpha} h(x).  \\
\eta_{-}(x)
&=
\frac{1}{2} - \gamma
- \min\left\{\frac{1}{4}, \, \frac{L}{\|g\|_{\Sigma^{\alpha}}}\right\} M^{-\alpha} h(x),
\end{align*}
where we recall that \(\|\cdot\|_{\Sigma^{\alpha}}\) is the \(\alpha\)-H\"{o}lder semi-norm as defined in equation~\eqref{eqn:HolderSeminorm}.
The purpose of \(\gamma\) is to ensure the proper values of \(\truenegative(1/2)\).
Otherwise, \(\gamma > 0\) is a small constant that satisfies
\begin{align}
\begin{aligned}
\|M^{-\alpha} h + 2\gamma\|_{L_{2}}^{2} 
&\leq 2 \|M^{-\alpha} h\|_{L_{2}}^{2} .
%& \text{ and } &&
%\gamma &\leq 
%\frac{M^{-\alpha}}{2}
%\sup_{x \in X}
%h(x).
\label{eqn:gammaProperties}
\end{aligned}
\end{align}
Finally, we define the distribution \(P_{+}\) such that \(X\) is uniformly distributed on \([0, 1]^{d}\) with probability \(p\), \(X\) takes the value \(x'\) with probability \(1 - p\), and the distribution of \(Y\) given \(X\) is determined by the regression function \(\eta_{+}\).
Similarly, \(P_{-}\) is defined such that \(X\) is uniformly distributed on \([0, 1]^{d}\) with probability \(p\), \(X\) takes the value \(x'\) with probability \(1 - p\), and the distribution of \(Y\) given \(X\) is given by the regression function \(\eta_{-}\).

Now, we have two things to prove: (1) that both \(P_{+}\) and \(P_{-}\) are actually in \(\Sigma^{\alpha}(L)\) for some \(L\) and (2) that any estimator must err with the specified probability one of these distributions.
For (1), we follow the same argument as in the proof of Theorem~\ref{theorem:UniformErrorLowerBound}.
Specifically, we see that
\begin{align*}
\|\eta_{+}(x)\|_{\Sigma^{\alpha}}
&\leq 
\frac{L}{\|g\|_{\Sigma^\alpha}} M^{-\alpha} \|h\|_{\Sigma^\alpha}
=
\frac{L}{\|g\|_{\Sigma^\alpha}} M^{-\alpha} M^{\alpha} \|g\|_{\Sigma^\alpha}
=
L.
\end{align*}
Thus, \(\eta_{+}\) is an element of \(\Sigma^{\alpha}(L)\), and one can show identically that \(\eta_{-}\) is an element of \(\Sigma^{\alpha}(L)\).

So, we continue to step (2): establishing the bound on the error probability.
Consider the event \(E = \left\{\hat{\truenegative}\left(1/2\right) > p / 4\right\}\), and let \(E^{c}\) denote its complement.
Then, we have the inequality
\begin{align*}
A
&:=
\sup_{(X, Y) \sim P \in \mathfrak{P}(\alpha, L)} 
\prob\left(
\left|
\hat{\truenegative}\left(\frac{1}{2}\right) - \truenegative\left(\frac{1}{2}\right)
\right|
> 
\frac{p}{4}
\right) \\
&\geq 
\max\left\{
P_{+}^{n}(E), \; P_{-}^{n}(E^{c})
\right\} \\
&\geq 
\frac{1}{2}
\left(P_{+}^{n}(E) + P_{-}^{n}(E^{c})\right),
\end{align*}
where the first inequality comes from considering a smaller set and specifying particular events, and the second inequality follows from the elementary fact that the mean is smaller than the maximum.
Using Lemma~\ref{lemma:PinskerInequality2} and the fact that any event \(F\) satisfies the inequality
\[
P(F) + Q(F^{c}) \geq \int \min\left\{dP, dQ\right\},
\]
we then have
\begin{align}
A
&\geq 
\frac{1}{4} \exp\left(-\kldiv(P_{+}^{n}, \; P_{-}^{n})\right) 
=
\frac{1}{4}\exp\left(-n \kldiv(P_{+}, \; P_{-})\right).
\label{eqn:BoundTheKLAlpha}
\end{align}
Thus, it suffices to bound the Kullback-Leibler divergence between \(P_{+}\) and \(P_{-}\).
To do this, we make use of the elementary fact that for \(x \geq -2/3\), we have \(-\log(1 + x) \leq x^{2} - x\).
Then, we have
\begin{align*}
P_{+}(x, 1) \log \frac{P_{+}(x, 1)}{P_{0}(x, 1)} 
&\leq 
p \eta_{+}(x) \log \frac{\eta_{+}(x)}{\eta_{-}(x)} \\
&=
-p \eta_{+}(x) \log \left(1 + \frac{\eta_{-}(x) - \eta_{+}(x)}{\eta_{+}(x)} \right) \\
&\leq 
p \eta_{+}(x) \left(
\left(
\frac{\eta_{-}(x) - \eta_{+}(x)}{\eta_{+}(x)}
\right)^{2}
-
\frac{\eta_{-}(x) - \eta_{+}(x)}{\eta_{+}(x)}
\right) \\
&=
p\left[\frac{(\eta_{-}(x) - \eta_{+}(x))^{2}}{\eta_{+}(x)}
-
\eta_{-}(x) + \eta_{+}(x)\right] \\
&\leq 
p\left[4 (\eta_{-}(x) - \eta_{+}(x))^{2}
- \eta_{-}(x) + \eta_{+}(x)\right], 
\end{align*}
where in the final inequality we also use the fact that \(n\) is sufficiently large and \(\gamma\) is sufficiently small so that \(\eta_{+}(x) \geq 1/4\).
Similarly, we may obtain
\[
P_{+}(x, 0) \log \frac{P_{+}(x, 0)}{P_{-}(x, 0)}
\leq 
p\left[4 (\eta_{-}(x) - \eta_{+}(x))^{2}
+ \eta_{-}(x) - \eta_{+}(x)\right].
\]
Putting the previous two equations together, we have
\begin{align*}
\kldiv(P_{+},\; P_{-})
&=
\int_{x \in [0, 1]^{d}}
P_{+}(x, 1) \log \frac{P_{+}(x, 1)}{P_{-}(x, 1)} dx
+
\int_{x \in [0, 1]^{d}}
P_{+}(0, 1) \log \frac{P_{+}(x, 0)}{P_{-}(x, 0)} dx \\
&\leq 
8p \int_{x \in [0, 1]^{d}}
(\eta_{+}(x) - \eta_{-}(x))^{2} dx \\
&=
8p  \int_{x \in [0, 1]^{d}}
(M^{-\alpha}h(x) + 2 \gamma)^{2} dx \\
&\leq 
8p M^{-2\alpha} \|h\|_{L_{2}}^{2},
\end{align*}
where the final inequality comes from equation~\eqref{eqn:gammaProperties}.
Now, by a change of variables, we have
\begin{align*}
\kldiv(P_{+},\; P_{-})
&\leq 
8\|g\|_{L_{2}}^{2} p M^{-(2\alpha + d)}.
\end{align*}
Substituting into the bound~\eqref{eqn:BoundTheKLAlpha}, we have the lower bound
\begin{equation}
A
\geq 
\frac{1}{4} \exp\left(-8\|g\|_{L_{2}}^{2} pn M^{-(2\alpha + d)} \right).
\label{eqn:ProbABound}
\end{equation}

Now, we need to determine the maximum error that is implied by this bound.
Denote the maximum error by \(\epsilon\). 
We have
\begin{align*}
\epsilon
&=
\min_{\eta \in \{\eta_{+}, \eta_{-}\}} \sup_{x \in [0, 1]^{d}} \left|\eta(x) - \frac{1}{2}\right| %\\
=
\gamma + \frac{M^{-\alpha}}{2} 
\sup_{x \in [0,1]^{d}}
h(x) %\\
>
\frac{M^{-\alpha}}{2}
\sup_{x \in [0,1]^{d}}
h(x) %\\
\geq  
B M^{-\alpha}
%B n^{-\frac{\alpha}{2\alpha + d}},
\end{align*}
for some constant \(B\).
Thus, writing \(\epsilon^{(2\alpha + d) / \alpha}\) in place of \(M^{-(2\alpha + d)}\) in equation~\eqref{eqn:ProbABound} and setting \(\delta\) to be equal to the right hand side of equation~\eqref{eqn:ProbABound}, we obtain
\[
\delta 
\geq
\frac{1}{4}
\exp\left(-8\|g\|_{L_{2}}^{2}B'' pn \epsilon^{\frac{2\alpha + d}{\alpha}} \right).
\]
Solving for \(\epsilon\), we obtain
\[
\epsilon
\geq
\left(\frac{\log \frac{1}{4 \delta}}{8 \|g\|_{L_{2}} B p n}\right)^{\frac{\alpha}{2\alpha + d}}
=
B'
\left(\frac{\log \frac{1}{4 \delta}}{p n}\right)^{\frac{\alpha}{2\alpha + d}}
\]
for constants \(B\) and \(B'\).
Thus, in summary, we have shown that with probability at least \(\delta\), we make an error of size at least \(p/4\) in estimating \(\truenegative(1/2)\); \(p\) has the form
\(p =
\prob\left(t - \epsilon \leq \eta(X) \leq t + \epsilon
\right)\);
and \(\epsilon\) may be of size at least \(\Omega(n^{-\alpha/(2\alpha + d)})\).
This establishes the proposition.
\end{proof}

%---------------------------------------------%
%---------------------------------------------%
%---------------------------------------------%
%---------------------------------------------%
\section{General Classification Metric Proofs}
\label{app:GeneralClassificationMetricProofs}

In this appendix, we prove the corollaries for precision, recall, and F1 score.

Define the event 
\begin{align*}
A_{\confusionmatrix}
&=
\bigg\{
|\hat{\truenegative}_{c}(q) - \truenegative_{c}(q)| 
\leq E_{\truenegative, c}(q), 
\;
|\hat{\falsenegative}_{c}(q) - \falsenegative{c}(q)| 
\leq E_{\falsenegative, c}(q), 
\\ &\qquad 
\; 
|\hat{\falsepositive}_{c}(q) - \falsepositive{c}(q)| 
\leq E_{\falsepositive, c}(q),
\;
|\hat{\truepositive}_{c}(q) - \truepositive{c}(q)| 
\leq E_{\truepositive, c}(q)
\text{ for all } c \in [C] 
\bigg\}
\end{align*}

\begin{lemma}
Under the event \(A_{\confusionmatrix}\), 
we have
\begin{align*}
|\hat{\precision}_{c}(q) - \precision_{c}(q)|
&\leq 
E_{\precision, c}(q)   \\
|\hat{\recall}_{c}(q) - \recall_{c}(q)|
&\leq 
E_{\recall, c}(q)
\end{align*}
for all \(q\) for which \(E_{\precision, c}(q) > 0\) and
for all \(q\) for which \(E_{\recall, c}(q) > 0\)
\label{lemma:MulticlassGeneral:PrecisionRecall}
\end{lemma}

\begin{proof}
We only prove the result for precision, since the proof for recall only requires substituting false negatives in the place of false positives.

We provide upper and lower bounds on \(\hat{\precision}_{c}(q)\).
For the upper bound, we have
\begin{align*}
\hat{\precision}_{c}(q)
&\leq 
\frac{\truepositive_{c}(q) + E_{\truepositive, c}(q)}{\truepositive_{c}(q) + \falsepositive(t) - E_{\truepositive, c}(q) - E_{\falsepositive, c}(q)} \\
&=
\left(\precision_{c}(q) + 
\frac{E_{\truepositive, c}(q)}{\truepositive_{c}(q) + \falsepositive_{c}(q)}
\right)
\left(
1 
+ 
\frac{E_{\truepositive, c}(q) + E_{\falsepositive, c}(q)}{\truepositive_{c}(q) + \falsepositive_{c}(q) - E_{\truepositive, c}(q) - E_{\falsepositive, c}(q)}
\right) \\
&\leq 
\precision_{c}(q) + 2 B_{\precision, c}(q) + B_{\precision, c}(q)^{2},
\end{align*}
where
\[
B_{\precision, c}(q)
= 
\frac{E_{\truepositive, c}(q) + E_{\falsepositive, c}(q)}{\truepositive_{c}(q) + \falsepositive_{c}(q) - E_{\truepositive, c}(q) - E_{\falsepositive, c}(q)}.
\]
First, note that \(B_{\precision, c}(q) \geq 0\) by assumption.
Next if \(B_{\precision, c}(q) \leq 1\), then we can obtain the bound
\[
\hat{\precision}_{c}(q)
\leq 
\precision_{c}(q)
+
3 B_{\precision, c}(q).
\]
If \(B_{\precision, c}(q) \geq 1\), then since empirical precision is bounded by \(1\), the previous inequality is also true, which proves the upper bound.

For the lower bound, we have
\begin{align*}
\hat{\precision}_{c}(q)
&\geq 
\frac{\truepositive_{c}(q) - E_{\truepositive, c}(q)}{\truepositive_{c}(q) + \falsepositive_{c}(q) + E_{\truepositive, c}(q) + E_{\falsepositive, c}(q)} \\
&=
\frac{\truepositive_{c}(q) - E_{\truepositive, c}(q)}{\truepositive_{c}(q) + \falsepositive_{c}(q)}
\cdot 
\frac{\truepositive_{c}(q) + \falsepositive_{c}(q)}{\truepositive_{c}(q) + \falsepositive_{c}(q) + E_{\truepositive, c}(q) + E_{\falsepositive, c}(q)} \\
&=
\left( \precision_{c}(q) 
- 
\frac{E_{\truepositive, c}(q)}{\truepositive_{c}(q) + \falsepositive_{c}(q)} \right)
\left(1 
- 
\frac{E_{\truepositive, c}(q) + E_{\falsepositive, c}(q)}{\truepositive_{c}(q) + \falsepositive_{c}(q) + E_{\truepositive, c}(q) + E_{\falsepositive, c}(q)}
\right) 
\\
&\geq
\precision_{c}(q)
- 
2B_{\precision, c}(q).
\end{align*}
This proves the lower bound and completes the proof.
\end{proof}

\begin{proof}[Proof of Corollary~\ref{corollary:Examples:PrecisionRecall}]
Theorem~\ref{theorem:MulticlassConfusionMatrixBound} shows that \(A_{\confusionmatrix}\) occurs with probability at least \(1 - \delta\), and so Lemma~\ref{lemma:MulticlassGeneral:PrecisionRecall} completes the proof.
\end{proof}
%---------------------------------------------%
%---------------------------------------------%

Define the event 
\[
A_{\fone}
=
\left\{
|\hat{\precision}_{c}(q) - \precision_{c}(q)| \leq 
E_{\precision, c}(q), \;
|\hat{\recall}_{c}(q) - \recall_{c}(q)|
\leq 
E_{\recall, c}(q)
\text{ for all } c \in [C]
\right\}.
\]
\begin{lemma}
Under \(A_{\fone}\), we have
we have
\[
|\hat{\fone}_{c}(q) - \fone_{c}(q)|
\leq 
E_{\fone, c}(q),
\]
for each \(c\) in \([C]\). 
Additionally, we have
\[
|\hat{\fone}(q) - \fone(q)|
\leq 
\frac{1}{C} \sum_{j = 1}^{C} 
E_{\fone, c}(q).
\]
\label{lemma:Examples:F1Score}
\end{lemma}
%---------------------------------------------%
%---------------------------------------------%
\begin{proof}
We start by upper and lower bounding the empirical F1 score for class \(c\) under \(A_{\fone}\).
Starting with the upper bound, we have
\begin{align*}
\hat{\fone}_{c}(q)
&\leq 
2 
\frac{(\precision_{c}(q) + E_{\precision, c}(q)) (\recall_{c}(q) + E_{\recall, c}(q))}{\precision_{c}(q) + \recall_{c}(q) - E_{\precision, c}(q) - E_{\recall, c}(q)} \\ 
&=
\left(
\fone_{c}(q) + 4 \frac{E_{\precision, c}(q) + E_{\recall, c}(q)}{\precision_{c}(q) + \recall_{c}(q)} 
\right)
% \\ & \qquad \times 
\left(
1 
+
\frac{E_{\precision, c}(q) + E_{\recall, c}(q)}{\precision_{c}(q) + \recall_{c}(q) - E_{\precision, c}(q) - E_{\recall, c}(q)}
\right) \\ 
&\leq 
\fone_{c}(q) + 5 B_{\fone, c}(q) + 4B_{\fone, c}(q)^{2}
\end{align*}
where 
\[
B_{\fone, c}(q)
=
\frac{E_{\precision, c}(q) + E_{\recall, c}(q)}{\precision_{c}(q) + \recall_{c}(q) - E_{\precision, c}(q) - E_{\recall, c}(q)}.
\]
If we have \(B_{\fone, c}(q) \leq 1\), then we can write 
\[
\hat{\fone}_{c}(q)
\leq 
\fone_{c}(q) + 9 B_{\fone, c}(q)
=
\fone_{c}(q) + E_{\fone, c}(q).
\]
If \(B_{\fone, c}(q) \geq 1\), then the previous upper bound still holds since \(\fone_{c}(q)\) is bounded by \(1\), which completes the upper bound

For the lower bound, we have
\begin{align*}
\hat{\fone}_{c}(q)
&\geq 
2 
\frac{(\precision_{c}(q) - E_{\precision, c}(q))(\recall_{c}(q) - E_{\recall, c}(q))}{\precision_{c}(q) + \recall_{c}(q) + E_{\precision, c}(q) + E_{\recall, c}(q)} \\
&=
\left(
\fone_{c}(q)
-
2 \frac{E_{\precision, c}(q) + E_{\recall, c}(q)}{\precision_{c}(q) + \recall_{c}(q)}
\right)
\left(
1 
-
\frac{E_{\precision, c}(q) + E_{\recall, c}(q)}{\precision_{c}(q) + \recall_{c}(q) + E_{precision, c}(q) + E_{\recall, c}(q)}
\right) \\ 
&\geq 
\fone_{c}(q)
- 
3 B_{\fone, c}(q).
\end{align*}
This completes the first part of the lemma.
For the second, we have
\begin{align*}
|\hat{\fone}(q) - \fone(q)|
\leq 
\frac{1}{C} \sum_{j = 1}^{C}
|\hat{\fone}_{c}(q) - \fone_{c}(q)|,
\end{align*}
which completes the proof.
\end{proof}

\begin{proof}[Proof of Corollary~\ref{corollary:Examples:F1Score}]
By Corollary~\ref{corollary:Examples:PrecisionRecall}, we see that \(A_{\fone}\) has probability at least \(1 - \delta\). 
Thus, Lemma~\ref{lemma:Examples:F1Score} completes the proof.
\end{proof}

%---------------------------------------------%
%---------------------------------------------%
\section{Additional Optimization Details}
\label{sec:AdditionalOptimization}

In this section, we provide additional details for optimization results of section \ref{sec:NumericalResults}.
First, we provide the explicit algorithm for grid search in Algorithm~\ref{algorithm:GridSearch}.
%---------------------------------------------%
%---------------------------------------------%
\begin{algorithm}[t]
	\caption{Grid search weights}
	\SetKwInOut{Input}{Input}
	\SetKwInOut{Output}{Output}
	\Input{grid of weights \(W\), empirical F1 function \(\hat{\fone}\).}
	\(q_{\text{grid}}\gets \)None\\
	\ForEach{\(q\) in \(W\)}{
		\If{\(q_{\text{grid}}\) is None or \(\hat{\fone}(q) > \hat{\fone}(q_{\text{grid}})\)}{\(q_{\text{grid}}\gets q\)}
	}
	
	\Output{The weights \(q_{\text{grid}}\).}
	\caption{Grid search algorithm.}
	\label{algorithm:GridSearch}
\end{algorithm}
%---------------------------------------------%
%---------------------------------------------%

Second, we provide additional comparisons between the greedy algorithm and grid search in Figure~\ref{figure:GreedyGridDiffs}.
In particular, we examine the difference between the population F1 score for grid search and the greedy algorithm in Figure~\ref{figure:ExpDetails:F1AbsDiff}.
This shows that the difference is quite small.
Second, we examine the \(\ell_{2}\) distance between the learned weights in Figure~\ref{figure:ExpDetails:WeightL2Dist} as a function of \(n\).
Although there is no clear pattern, the distance between the weights is small.

\begin{figure}[t]
	\begin{subfigure}[t]{0.48\textwidth}
		\includegraphics[width=\textwidth]{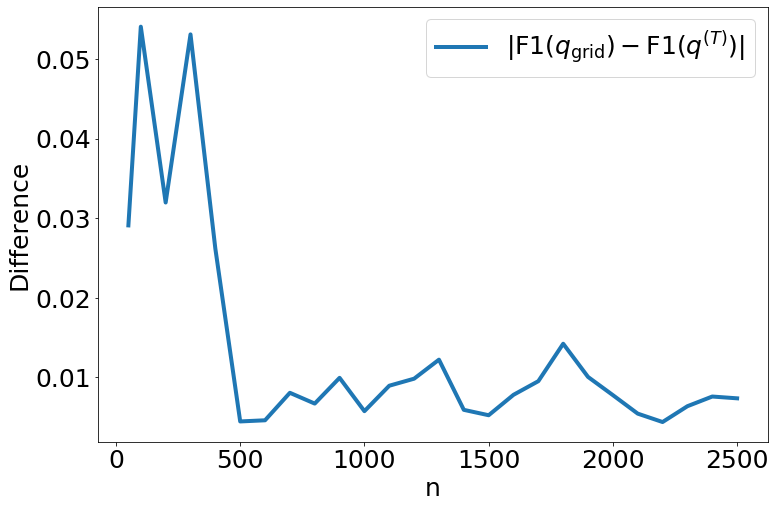}
		\caption{Absolute value of difference between \(\fone(q^{(T)})\) and \(\fone(q_{\text{grids}})\) for different training sizes \(n\). This difference decreases as \(n\) increases.}
		\label{figure:ExpDetails:F1AbsDiff}
	\end{subfigure}\hfil
	\begin{subfigure}[t]{0.48\textwidth}
		\includegraphics[width=\textwidth]{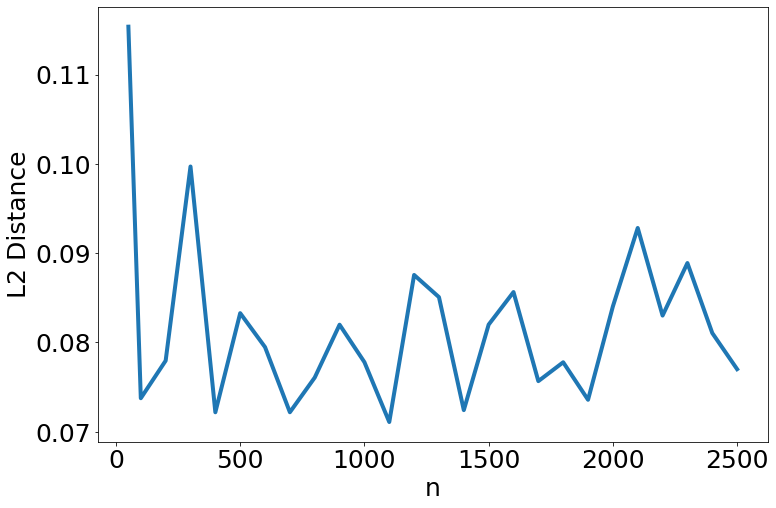}
		\caption{The \(\ell_{2}\) distance between \(q^{(T)}\) and \(q_{\text{grids}}\) for different training sizes \(n\). This difference decreases as \(n\) increases as well.}
		\label{figure:ExpDetails:WeightL2Dist}
	\end{subfigure}
	\caption{Comparing the greedy algorithm and grid search algorithms as a function of sample size \(n\).}
	\label{figure:GreedyGridDiffs}
\end{figure}

%---------------------------------------------%
%---------------------------------------------%
\section{Standard Tools}
\label{sec:app:StandardTools}

\subsection{Lemmas}

\begin{lemma}[Hoeffding's inequality]
Let \(X_{1}, \ldots, X_{n}\) be a sequence of random variables such that each \(X_{i}\) takes values in the interval \([a_{i}, b_{i}]\).
Define the sum
\(S_{n} = \sum_{i = 1}^{n} X_{i}\).
Let \(t > 0\).
Then, we have the inequality
\begin{align*}
& \begin{aligned}
\prob\left(
S_{n} - \expect S_{n} > t
\right)
&\leq 
\exp\left(
-\frac{2t^{2}}{\sum_{t = 1}^{n} (b_{i} - a_{i})^{2}}
\right).
\end{aligned}
\end{align*}
\label{lemma:Hoeffding}
\end{lemma}

\begin{lemma}[multiplicative Chernoff bound]
Let \(X_{1}, \ldots, X_{n}\) be a sequence of random variables such that each \(X_{i}\) takes values in \(\{0, 1\}\).
Define the sum
\(S_{n} = \sum_{i = 1}^{n} X_{i}\), and for brevity, let \(\mu = \expect S_{n}\).
Let \(\delta\) be in \((0, 1)\).
Then, we have the inequality
\begin{align*}
& \begin{aligned}
\prob\left(
S_{n} < (1 - \delta) \mu
\right)
&\leq 
\exp\left(
-\frac{\delta^{2} \mu }{2}
\right).
\end{aligned}
\end{align*}
\label{lemma:ChernoffMultiplicative}
\end{lemma}

%---------------------------------------------%
%---------------------------------------------%

\begin{lemma}[Lemma~2.6 of \citealt{tsybakov2009introduction}]
We have the inequality
\[
\frac{1}{2} \exp\left(- \kldiv(P, Q)\right)
\leq 
\int \min\left\{dP, dQ\right\}.
\]
\label{lemma:PinskerInequality2}
\end{lemma}

%---------------------------------------------%
%---------------------------------------------%
\begin{lemma}[\citealt{slud1977distribution}]
	Let \(0 < p < 1 / 2\) and define \(B\) to be a \(\binomial(n, p)\) random variable.
	Let \(\ell\) be a nonnegative integer,
	and let \(Z\) be a standard normal random variable.
	
	\begin{itemize}
	    \item[(a)]
	    If \(\ell \leq np\),
	    then we have \(\prob(B \geq \ell) \geq \prob\left(Z \geq (l - np) / \sqrt{np}\right)\).
	    
	    \item[(b)]
	    If \(np \leq \ell \leq n(1 - p)\),
	    then we have \(\prob(B \geq \ell) \geq \prob\left(Z \geq (l - np) / \sqrt{np(1 - p)}\right)\).
	\end{itemize}
	\label{lemma:SludsInequalities}
\end{lemma}
%---------------------------------------------%
%---------------------------------------------%
%---------------------------------------------%
%---------------------------------------------%
\end{document}